\documentclass{article}
\pdfoutput=1
\usepackage[dvipsnames,table]{xcolor}
\usepackage[stretch=5, shrink=33]{microtype}
\usepackage{graphicx}
\usepackage{booktabs}
\usepackage{xfrac}
\usepackage{hyperref}

\usepackage[accepted]{icml2024}

\usepackage[scaled=0.92]{helvet}

\usepackage{amsmath}
\usepackage{amssymb}
\usepackage{mathtools}
\usepackage{amsthm}

\usepackage[utf8]{inputenc}
\usepackage[T1]{fontenc}
\usepackage{natbib}
\usepackage{subcaption}
\usepackage{booktabs}
\usepackage{multirow}
\usepackage{url}
\usepackage{tikz}
\usetikzlibrary{arrows}
\usetikzlibrary{shadows,arrows.meta}
\usepackage{paralist}
\usepackage[stable]{footmisc}
\usepackage[edges]{forest}
\usepackage{fontawesome5}
\usepackage{ifthen}

\usepackage[capitalize,noabbrev]{cleveref}

\usepackage{bm}
\usepackage{amsmath}
\usepackage{amssymb}
\usepackage{amsthm}
\usepackage{amsfonts}
\usepackage{thmtools}		
\usepackage{mleftright}
\usepackage{stmaryrd}
\usepackage{nicefrac}

\let\originalleft\left
\let\originalright\right
\renewcommand{\left}{\mathopen{}\mathclose\bgroup\originalleft}
\renewcommand{\right}{\aftergroup\egroup\originalright}

\theoremstyle{plain}
\newtheorem{theorem}{Theorem}

\newtheorem{proposition}[theorem]{Proposition}

\newtheorem{lemma}[theorem]{Lemma}
\newtheorem{corollary}[theorem]{Corollary}
\theoremstyle{definition}
\newtheorem{definition}[theorem]{Definition}

\usepackage{thm-restate}
\usepackage{siunitx}

\usepackage{pifont}

\usepackage{enumitem}
\setlist[enumerate]{itemsep=0.2ex, topsep=0.5\topsep}
\setlist[description]{itemsep=0.2ex, topsep=0.5\topsep}
\setlist[itemize]{itemsep=0.2ex, topsep=0.5\topsep}

\usepackage{ellipsis}
\usepackage[auth-lg]{authblk}

\makeatletter
\def\thmt@refnamewithcomma #1#2#3,#4,#5\@nil{%
\@xa\def\csname\thmt@envname #1utorefname\endcsname{#3}%
\ifcsname #2refname\endcsname
\csname #2refname\expandafter\endcsname\expandafter{\thmt@envname}{#3}{#4}%
\fi
}
\makeatother
\usepackage{changepage}

\newcommand{\Nb}{\mathbb{N}}

\newcommand{\Rb}{\mathbb{R}}

\newcommand{\kwl}[1]{$#1$\textrm{-}\textsf{WL}\xspace}
\newcommand{\klwl}[1]{$#1$\textrm{-}\textsf{LWL}\xspace}

\newcommand{\dkwl}[1]{$\delta$\textrm{-}$#1$\textrm{-}\textsf{WL}\xspace}
\newcommand{\dklwl}[1]{$\delta$\textrm{-}$#1$\textrm{-}\textsf{LWL}\xspace}

\newcommand{\kgt}[1]{$#1$\textrm{-}\texttt{GT}\xspace}
\newcommand{\kfgt}[1]{ET\xspace}

\newcommand{\wlone}{$1$\textrm{-}\textsf{WL}}

\newcommand{\deltakwl}{$\delta$\textrm{-}$k$\textrm{-}\textsf{WL}}

\newcommand{\localkwl}{$\delta$\textrm{-}$k$\textrm{-}\textsf{LWL}}

\newcommand{\hb}{\mathbf{h}}

\newcommand{\REL}{\mathsf{RELABEL}}

\newcommand{\new}[1]{\emph{#1}}

\renewcommand{\vec}[1]{\mathbf{#1}}
\newcommand{\oms}{\{\!\!\{}
\newcommand{\cms}{\}\!\!\}}

\definecolor{dark2green}{rgb}{0.1, 0.65, 0.3}
\definecolor{dark2orange}{rgb}{0.9, 0.4, 0.}
\definecolor{dark2purple}{rgb}{0.4, 0.4, 0.8}

\newcommand{\embed}[1]{\text{#1}}

\colorlet{GreenCellColor}{PineGreen!40}
\colorlet{FigureColorTransformer}{PineGreen!40}

\newcommand{\first}[1]{\textbf{\textcolor{PineGreen!80}{#1}}}
\newcommand{\second}[1]{\textbf{\textcolor{Dandelion}{#1}}}
\newcommand{\third}[1]{\textbf{\textcolor{RoyalBlue}{#1}}}

\usepackage[textsize=tiny]{todonotes}

\icmltitlerunning{Aligning Transformers with Weisfeiler--Leman}

\begin{document}

\twocolumn[
\icmltitle{Aligning Transformers with Weisfeiler--Leman}

\icmlsetsymbol{equal}{*}

\begin{icmlauthorlist}
\icmlauthor{Luis Müller}{rwth}
\icmlauthor{Christopher Morris}{rwth}
\end{icmlauthorlist}

\icmlaffiliation{rwth}{Department of Computer Science, RWTH Aachen University, Germany}

\icmlcorrespondingauthor{Luis Müller}{luis.mueller@cs.rwth-aachen.de}

\icmlkeywords{Machine Learning, ICML, Transformers, Weisfeiler-Leman, Graph}

\vskip 0.3in
]

\printAffiliationsAndNotice{}  %

\begin{abstract}
Graph neural network architectures aligned with the $k$-dimensional Weisfeiler--Leman (\kwl{k}) hierarchy offer theoretically well-understood expressive power. However, these architectures often fail to deliver state-of-the-art predictive performance on real-world graphs, limiting their practical utility. While recent works aligning graph transformer architectures with the \kwl{k} hierarchy have shown promising empirical results, employing transformers for higher orders of $k$ remains challenging due to a prohibitive runtime and memory complexity of self-attention as well as impractical architectural assumptions, such as an infeasible number of attention heads. Here, we advance the alignment of transformers with the \kwl{k} hierarchy, showing stronger expressivity results for each $k$, making them more feasible in practice. In addition, we develop a theoretical framework that allows the study of established positional encodings such as Laplacian PEs and SPE. We evaluate our transformers on the large-scale PCQM4Mv2 dataset, showing competitive predictive performance with the state-of-the-art and demonstrating strong downstream performance when fine-tuning them on small-scale molecular datasets.
\end{abstract}

\section{Introduction}\label{sec:intro}
Message-passage graph neural networks (GNNs) are the de-facto standard in graph learning~\citep{Gil+2017,Kip+2017,Sca+2009,Xu+2018b}. However, due to their purely local mode of aggregating information, they suffer from limited expressivity in distinguishing non-isomorphic graphs in terms of the $1$-dimensional Weisfeiler--Leman algorithm (\wlone)~\citep{Mor+2019, Xu+2018b}. Hence, recent works~\citep{Azi+2020,Mar+2019,Morris2020b,Mor+2022,Mor+2022b} introduced \textit{higher-order} GNNs, aligned with the $k$-dimensional Weisfeiler--Leman (\kwl{k}) hierarchy for graph isomorphism testing~\cite{Cai+1992}, resulting in more expressivity with an increase in the order $k > 1$. The \kwl{k} hierarchy draws from a rich history in graph theory \citep{Bab1979,Cai+1992,Gro2017,Wei+1968}, offering a deep theoretical understanding of \kwl{k}-aligned GNNs. In contrast, graph transformers (GTs)~\citep{Glickman+2023,He+2022,ma2023GraphInductiveBiases,Mue+2023,rampavsek2022recipe,Ying2021} recently demonstrated state-of-the-art empirical performance. However, they draw their expressive power mostly from positional or structural encodings (PEs in the following), making it challenging to understand these models in terms of an expressivity hierarchy such as the \kwl{k}. In addition, their empirical success relies on modifying the attention mechanism \citep{Glickman+2023,ma2023GraphInductiveBiases,Ying2021} or using additional message-passing GNNs \citep{He+2022,rampavsek2022recipe}.

To still benefit from the theoretical power offered by the \kwl{k}, previous works~\citep{Kim+2021, Kim+2022} aligned transformers with $k$-IGNs, showing that transformer layers can approximate invariant linear layers \citep{Mar+2019c}. Crucially, \citet{Kim+2022} designed \new{pure transformers}, requiring no architectural modifications of the standard transformer and instead draw their expressive power from an appropriate \new{tokenization}, i.e., the encoding of a graph as a set of input tokens. Pure transformers provide several benefits over graph transformers that use message-passing or modified attention, such as being directly applicable to innovations for transformer architectures most notably Performers \citep{Choromanski+2021} and Flash Attention \citep{Dao+2022} to reduce the runtime or memory demands of transformers, as well as the Perceiver \citep{Jaegle+2021} enabling multi-modal learning. Unfortunately, the framework of \citet{Kim+2022} does not allow for a feasible transformer with provable expressivity strictly greater than \wlone{} due to an $O(n^6)$ runtime complexity and the requirement of $203$ attention heads resulting from their alignment with IGNs. This poses the question of whether a more feasible hierarchy of pure transformers exists.

\begin{figure*}[t]
	\begin{center}
		\resizebox{.95\textwidth}{!}{
			
			\tikzset{
				treenode/.style = {shape=rectangle, rounded corners,
					draw, align=center,
					minimum width=50pt,
				}
			}
				\begin{tikzpicture}[scale=1.4,font=\footnotesize,>=stealth', thick,sibling distance=15mm, level distance=30pt,minimum size=18pt, sibling distance=50pt]
						
					\node(aaa) at (-5.0,0) [treenode,fill=FigureColorTransformer, minimum width=60pt] {\wlone{}};
					\node(a) at (-3.0,0) [treenode,fill=FigureColorTransformer, minimum width=60pt] {\kwl{(2,1)}};
					\node(aa) at (-3.0,1) [treenode,fill=Dandelion!50, minimum width=60pt] {\kgt{(2, 1)}};
					
					\node(aaaa) at (-5.0,1) [treenode,fill=Dandelion!50, minimum width=60pt] {\kgt{1}};

                    \node(b) at (-1.0,0) [treenode,fill=FigureColorTransformer, minimum width=60pt] {\kwl{(3,1)}};
					\node(bb) at (-1.0,1) [treenode,fill=Dandelion!50, minimum width=60pt] {\kgt{(3,1)}};
					
					\node(c) at (1.0,0) [treenode,fill=FigureColorTransformer, minimum width=60pt] {\kwl{3}};
					\node(cc) at (1.0,1) [treenode,fill=Dandelion!50, minimum width=60pt] {\kgt{3}};
					
					\node(d) at (4.0,0) [treenode,fill=FigureColorTransformer, minimum width=60pt] {\kwl{k}};
					\node(dd) at (4.0,1) [treenode,fill=Dandelion!50, minimum width=60pt] {\kgt{k}};

					\node(e) at (2.5,0) [] {$\cdots$};
							
					\draw[->] (a) -- (b) node[midway,label={[shift={(0.0,-.39)}]\small $\sqsupset$}]{};
					\draw[->] (b) -- (c) node[midway,label={[shift={(0.0,-.39)}]\small $\sqsupset$}]{};
					\draw[->] (c) -- (e) node[midway,label={[shift={(0.0,-.39)}]\small $\sqsupset$}]{};
					\draw[->] (e) -- (d) node[midway,label={[shift={(0.0,-.39)}]\small $\sqsupset$}]{};
					
					\draw[->] (a) -- (aa) node[midway,label={[shift={(-.25,-0.65)}]\small $\sqsubseteq$}]{};
					\draw[->] (b) -- (bb) node[midway,label={[shift={(-.25,-0.65)}]\small $\sqsubseteq$}]{};
					\draw[->] (c) -- (cc) node[midway,label={[shift={(-.25,-0.65)}]\small $\sqsubset$}]{};
					\draw[->] (d) -- (dd) node[midway,label={[shift={(-.25,-0.65)}]\small $\sqsubset$}]{};
					
					\draw[<-] (a) -- (aaa) node[midway,label={[shift={(0.0,-.39)}]\small $\sqsupset$}]{};
					
					\draw[->] (aaa) -- (aaaa) node[midway,label={[shift={(-.25,-0.65)}]\small $\sqsubseteq$}]{};
			
					\draw[->, thick] (-5.7,1.5) -- (4.7,1.5) node [midway,fill=white] {\rotatebox{0}{\textbf{\footnotesize Approximate more functions}}};
				\end{tikzpicture}
    }
	\end{center} 
	\caption{Overview of our theoretical results, aligning transformers with the established \kwl{k} hierarchy. Forward arrows point to more powerful algorithms or neural architectures. $A \sqsubset B$ ($A \sqsubseteq B$, $A \equiv B$)---algorithm $A$ is strictly more powerful than (as least as powerful as, equally powerful as) $B$. The relations between the boxes in the lower row stem from~\citet{Cai+1992} and \citet{Mor+2022b}.\label{fig:overview}}
\end{figure*}
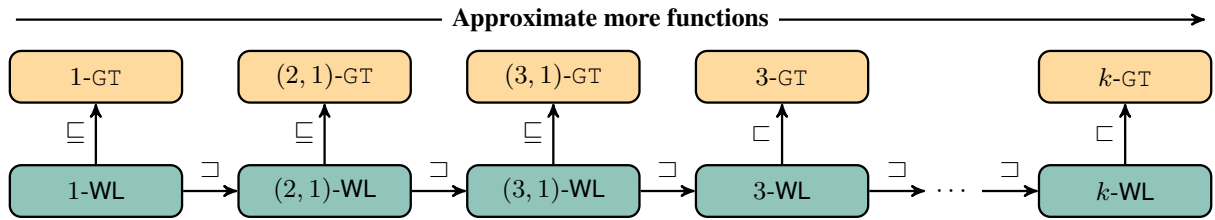

\paragraph{Present work} Here, we offer theoretical improvements for pure transformers. 
Our guiding question is
\begin{center}
\begin{adjustwidth}{20pt}{20pt}
\begin{center}
\textit{Can we design a hierarchy of increasingly expressive pure transformers that are also feasible in practice?}
\end{center}
\end{adjustwidth}
\end{center}
In this work, we will make significant progress to answering this question in the affirmative.
First, in \Cref{sec:theory}, we improve the expressivity of pure transformers for every order $k > 0$, while keeping the same runtime complexity as \citet{Kim+2022} by aligning transformers closely with the Weisfeiler--Leman hierarchy. Secondly, we demonstrate the benefits of this close alignment by showing that our results directly yield transformers aligned with the \kwl{(k,s)} hierarchy, recently proposed to increase the scalability of higher-order variants of the Weisfeiler--Leman algorithm~\citep{Mor+2022b}. These transformers then form a hierarchy of increasingly expressive pure transformers. We show in \Cref{sec:implementation}, that our transformers can be naturally implemented with established node-level PEs such as the Laplacian PEs in \citet{Kre+2021}.
Lastly, we show in \Cref{sec:experiments} that our transformers are feasible in practice and have more expressivity than the \wlone{}. In particular, we obtain very close to state-of-the-art results on the large-scale \textsc{PCQM4Mv2} dataset \citep{Hu2021} and further show that our transformers are highly competitive with GNNs when fine-tuned on small-scale molecular benchmarks \citep{hu2020ogb}. See \Cref{fig:overview} for an overview of our theoretical results and \Cref{tab:comparison_to_Kim2022} for a comparison of our pure transformers with the transformers in \citet{Kim+2021} and \citet{Kim+2022}.

\paragraph{Related work} Many graph learning architectures with higher-order Weisfeiler--Leman expressive power exist, most notably $\delta$-$k$-GNNs \citep{Morris2020b}, SpeqNets \citep{Mor+2022b}, $k$-IGNs \citep{Mar+2019c, Mar+2019b}, PPGN \citep{Azi+2020,Mar+2019}, and the more recent PPGN++ \citep{Puny+2023}. Moreover, \citet{yaronlipman2020global} devise a low-rank attention module possessing the same power as the folklore \kwl{2}. \citet{Bod+2021b} propose CIN with an expressive power of at least \kwl{3}. However, while theoretically intriguing, higher-order GNNs often fail to deliver state-of-the-art performance on real-world problems~\citep{Azi+2020,Morris2020b, Mor+2022b}, making theoretically grounded architectures less relevant in practice.

Graph transformers with higher-order expressive power are Graphormer-GD \citep{Zha+2023} as well as the higher-order graph transformers in \citet{Kim+2021} and \citet{Kim+2022}.
However, Graphormer-GD is less expressive than the \kwl{3}~\citep{Zha+2023}.  Further, \citet{Kim+2021} and \citet{Kim+2022} align transformers with $k$-IGNs, and, thus, obtain the theoretical expressive power of the corresponding \kwl{k}. However, they do not empirically evaluate their transformers for $k > 2$. For $k = 2$, \citet{Kim+2022} propose \new{TokenGT}, a pure transformer with a $n + m$ tokens per graph, where $n$ is the number of nodes and $m$ is the number of edges. This transformer has \kwl{2} expressivity, which, however, is the same as \wlone{} expressivity~\citep{Mor+2022}.
From a theoretical perspective, the transformers in \citet{Kim+2022} still require impractical assumptions such as $\textit{bell}(2k)$ attention
heads, where $\textit{bell}(2k)$ is the $2k$-th bell number \citep{Mar+2019},\footnote{For example,  $\textit{bell}(2 \cdot 3) = 203$, $\textit{bell}(2 \cdot 4) = 4\,140$, $\textit{bell}(2 \cdot 5) = 115\,975$. In comparison, GPT-3, a large transformer with 175B parameters, has only 96 attention heads \citep{brownGPT3+2020}.} resulting in 203 attention heads for a transformer with a provable expressive power strictly stronger than the \wlone{}. In addition, \citet{Kim+2022} introduces special encodings called node- and type identifiers that are theoretically necessary but, as argued in Appendix A.5 in \citep{Kim+2022}, not ideal in practice. For an overview of the Weisfeiler--Leman hierarchy in graph learning; see \citet{Mor+2022}.

\section{Background}
We consider \new{node-labeled graphs} $G \coloneqq (V(G), E(G), \ell)$ with $n$ nodes and without self-loops and without isolated nodes, where $V(G)$ is the set of \new{nodes}, $E(G)$ is the set of \new{edges}, and $\ell \colon V(G) \rightarrow \Nb$ assigns an initial \new{color} or \new{label} to each node. For convenience of notation, we always assume an arbitrary but fixed ordering over the nodes such that each node corresponds to a number in $[n]$.
Further $\vec{A}(G) \in \{0, 1\}^{n \times n}$ denotes the \new{adjacency matrix} where $\vec{A}(G)_{ij} = 1$ if, and only, if nodes $i$ and $j$ share an edge. We also construct a \new{node feature matrix} $\vec{F} \in \mathbb{R}^{n \times d}$ that is \new{consistent} with $\ell$, i.e., for nodes $i$ and $j$ in $V(G)$, $\vec{F}_i = \vec{F}_j$ if, and only, if $\ell(i) = \ell(j)$. Note that, for a finite subset of $\mathbb{N}$, we can always construct $\vec{F}$, e.g., with a one-hot encoding of the initial colors.
We call pairs of nodes $(i, j) \in V(G)^2$ \new{node pairs} or $2$-tuples. For a matrix $\vec{X} \in \mathbb{R}^{n \times d}$, whose $i$-th row represents the embedding of node $v \in V(G)$ in a graph $G$, we also write $\vec{X}_v \in \mathbb{R}^d$ or $\vec{X}(v) \in \mathbb{R}^d$ to denote the row corresponding to node $v$.
Further, we define a \new{learnable embedding} as a function, mapping a subset $S \subset \mathbb{N}$ to a $d$-dimensional Euclidean space, parameterized by a neural network, e.g., a neural network applied to a one-hot encoding of the elements in $S$. Moreover, $||\cdot||_\text{F}$ denotes the Frobenius norm.
Finally, we refer to some transformer architecture's concrete set of parameters, including its tokenization, as a \new{parameterization}. See \Cref{app:notation} for a complete description of our notation. Further, see~\cref{kwl_intro} for a formal description of the \kwl{k} and its variants.

\section{Expressive power of transformers on graphs}\label{sec:theory} 
Here, we consider the (standard) \new{transformer} \citep{Vaswani2017}, a stack of alternating blocks of \new{multi-head attention} and fully-connected \new{feed-forward networks}; see \Cref{sec:background_transformer} for a definition.

To align the above transformer with the \kwl{k}, we will provide appropriate input tokens $\vec{X}^{(0, k)}$ to the standard transformer for each $k \geq 1$ and subsequently show that then the $t$-th layer of the transformer can simulate the $t$-th iteration of some $k$-order Weisfeiler--Leman algorithm.
To gradually introduce our theoretical framework, we begin with tokenization for $k = 1$, obtaining a transformer with the expressive power of at least the \wlone{} and afterward generalize the tokenization to higher orders $k$.

\subsection{Transformers with $1$-WL expressive power}\label{sec:1wl_pure}
A commonly used baseline in prior work \citep{Mue+2023, rampavsek2022recipe} is to employ a standard transformer with Laplacian PEs \citep{Kre+2021}. Here, we start by characterizing such a transformer in terms of \wlone{} expressivity and introduce our theoretical framework.
Concretely, we propose a tokenization for a standard transformer to be at least as expressive as the \wlone{}. We first formalize our tokenization and then derive how \wlone{} expressivity follows. Let $G = (V(G), E(G), \ell)$ be a graph with $n$ nodes and feature matrix $\vec{F} \in \mathbb{R}^{n \times d}$, consistent with $\ell$. Then, we initialize $n$ token embeddings $\vec{X}^{(0, 1)} \in \mathbb{R}^{n \times d}$ as
\begin{equation}\label{eq:1_wl_token_embeddings}
    \vec{X}^{(0, 1)} \coloneqq \vec{F} + \vec{P},%
\end{equation}
where we call $\vec{P} \in \mathbb{R}^{n \times d}$ \new{structural embeddings}, encoding structural information for each token. For node $v$, we define
\begin{equation}\label{eq:1_wl_structural_embeddings}
    \vec{P}(v) \coloneqq \textsf{FFN}(\embed{deg}(v) + \embed{PE}(v)),
\end{equation}
where $\embed{deg} \colon V(G) \rightarrow \mathbb{R}^{d}$ is a learnable embedding of the node degree, $\embed{PE}\colon V(G) \rightarrow \mathbb{R}^{d}$ is a node-level PE such as the Laplacian PE \citep{Kre+2021} or SPE \citep{Huang+2023}, and $\textsf{FFN} \colon \mathbb{R}^d \rightarrow \mathbb{R}^d$ is a multi-layer perceptron. For the PE, we require that it enables us to distinguish whether two nodes share an edge in $G$, which we formally define as follows.
\begin{definition}[Adjacency-identifying]\label{adjacency_identifying}
Let $G = (V(G), E(G), \ell)$ be a graph with $n$ nodes. Let $\vec{X}^{(0, 1)} \in \mathbb{R}^{n \times d}$ denote the initial token embeddings according to \Cref{eq:1_wl_token_embeddings} with structural embeddings $\vec{P} \in \mathbb{R}^{n \times d}$. Let further
\begin{equation*}
   \Tilde{\mathbf{P}} \coloneqq \frac{1}{\sqrt{d_k}} \mathbf{P}\mathbf{W}^Q\big(\mathbf{P}\mathbf{W}^K\big)^T,
\end{equation*}
where $\vec{W}^Q,\vec{W}^K \in \mathbb{R}^{d \times d}$. 
Then $\vec{P}$ is \new{adjacency-identifying} if there exists $\vec{W}^Q,\vec{W}^K$ such that for any row $i$ and column $j$, 
\begin{equation*}
    \Tilde{\mathbf{P}}_{ij} = \max_{k} \Tilde{\mathbf{P}}_{ik},
\end{equation*}
if, and only, if $\vec{A}(G)_{ij} = 1$, where each row can have between one and $(n -1)$ maxima (for connected graphs). Further, an approximation $\vec{Q} \in \mathbb{R}^{n \times d}$ to $\vec{P}$ is \textit{sufficiently} adjacency-identifying if
\begin{equation*}
    \big|\big| \vec{P} - \vec{Q} \big|\big|_\text{F} < \epsilon, 
\end{equation*}
for any $\epsilon > 0$.
\end{definition}
As the name suggests, the property of \textit{(sufficiently) adjacency-identifying} allows us to identify the underlying adjacency matrix during the attention computation. In~\Cref{sec:implementation}, we introduce a slight generalization of the commonly used Laplacian PEs \citep{Kre+2021,rampavsek2022recipe}, which we refer to as LPE, and show that structural embeddings are sufficiently adjacency-identifying when using LPE or SPE \citep{Huang+2023} as node-level PEs. 
We call a standard transformer using the above tokenization with sufficiently adjacency-identifying node-level PEs the \kgt{1}.

Let us now understand why the above token embeddings with degree encodings and adjacency-identifying PEs are sufficient to obtain \wlone{} expressivity. To this end, we revisit a \wlone{} expressive GNN from \citet{Gro+2021}, which updates node representations $\vec{F}^{(t)}$ at layer $t$ as %
\begin{equation}\label{eq:grohe_gnn}
    \vec{F}^{(t)} \coloneqq \mathsf{FFN} \big(  \vec{F}^{(t-1)} + 2 \vec{A}(G) \vec{F}^{(t-1)} \big),
\end{equation}

where $\mathsf{FFN}$ is again a feed-forward neural network and $\vec{F}^{(t)}_i$ contains the representation of node $i \in V(G)$ at layer $t$. Contrast the above to the update of a \kgt{1} layer with a single head, given as
\begin{equation*}
\resizebox{\columnwidth}{!}{
    $\vec{X}^{(t, 1)} \coloneqq \mathsf{FFN} \big(  \vec{X}^{(t-1, 1)} + \mathsf{softmax}\big(\Tilde{\vec{X}}^{(t-1, 1)}\big) \vec{X}^{(t-1, 1)}\vec{W}^V \big),$}
\end{equation*}
where
\begin{equation*}
   \Tilde{\vec{X}}^{(t-1, 1)} \coloneqq \frac{1}{\sqrt{d_k}} \mathbf{X}^{ (t-1, 1)}\mathbf{W}^Q\big(\mathbf{X}^{(t-1, 1)}\mathbf{W}^K\big)^T
\end{equation*}
denotes the unnormalized attention matrix at layer $t$.
If we now set $\vec{W}^V = 2\vec{I}$, where $\vec{I}$ is the identity matrix, we obtain
\begin{equation*}
    \vec{X}^{(t,1)} \coloneqq \mathsf{FFN} \big(  \vec{X}^{(t-1,1)} + 2\cdot \mathsf{softmax}\big(\Tilde{\vec{X}}^{(t-1, 1)}\big) \vec{X}^{(t-1,1)} \big).
\end{equation*}
At this point, if we could reconstruct the adjacency matrix with $\mathsf{softmax}(\Tilde{\vec{X}}^{(t-1, 1)})$, we could simulate the \wlone{} expressive GNN of~\cref{eq:grohe_gnn} with a single attention head. However, the attention matrix is \new{right-stochastic}, meaning its rows sum to $1$. Unfortunately, this is not expressive enough to reconstruct the adjacency matrix, which generally is not right-stochastic. Thus, we aim to reconstruct the \new{row-normalized adjacency matrix} $\Tilde{\vec{A}}(G) \coloneqq\vec{D}^{-1}\vec{A}(G)$, where $\vec{D} \in \mathbb{R}^{n \times n}$ is the diagonal degree matrix, such that $\vec{D}_{ii}$ is the degree of node $i$. Indeed, $\Tilde{\vec{A}}(G)$ is right-stochastic. Further, element-wise multiplication of row $i$ of $\Tilde{\vec{A}}(G)$ with the degree of $i$ recovers $\vec{A}(G)$. As we show, adjacency-identifying PEs are, in fact, sufficient for $\mathsf{softmax}(\Tilde{\vec{X}}^{(t-1, 1)})$ to approximate $\Tilde{\vec{A}}(G)$ arbitrarily close. We show that then, we can use the degree embeddings $\embed{deg}(i)$ to de-normalize the $i$-th row of $\Tilde{\vec{A}}(G)$ and obtain
\begin{equation*}
    \vec{X}^{(t,1)} = \mathsf{FFN} \big(  \vec{X}^{(t-1, 1)} + 2 \vec{A}(G) \vec{X}^{(t-1, 1)} \big).
\end{equation*}
Showing that a transformer can simulate the GNN in \citet{Gro+2021} implies the connection to the \wlone{}. We formally prove the above in the following theorem, showing that the \kgt{1} can simulate the \wlone{}; see \Cref{app:proof_1gt} for proof details.
\begin{theorem}\label{theorem:1_wl}
Let  $G = (V(G), E(G), \ell)$ be a labeled graph with $n$ nodes and $\vec{F} \in \mathbb{R}^{n \times d}$ be a node feature matrix consistent with $\ell$. Further, let $C^1_t \colon V(G) \rightarrow \mathbb{N}$ denote the coloring function of the \wlone{} at iteration $t$.
Then, for all iterations $t \geq 0$, there exists a parametrization of the \kgt{1} such that
\begin{equation*}
    C^1_{t}(v) = C^1_{t}(w) \Longleftrightarrow \vec{X}^{(t, 1)}(v) = \vec{X}^{(t, 1)}(w),
\end{equation*}
for all nodes $v, w \in V(G)$.
\end{theorem}
Having set the stage for theoretically aligned transformers, we now generalize the above to higher orders $k > 1$.

\subsection{Transformers with $k$-WL expressive power}\label{sec:kwl_pure}
Here, we propose tokenization for a standard transformer, operating on $n^k$ tokens, to be strictly more expressive as the \kwl{k}, for an arbitrary but fixed $k > 1$. Subsequently, to make the architecture more practical, we reduce the number of tokens while remaining strictly more expressive than the \wlone{}.

Again, we consider a labeled graph $G = (V(G), E(G), \ell)$ with $n$ nodes and feature matrix $\vec{F} \in \mathbb{R}^{n \times d}$, consistent with $\ell$. Intuitively, to surpass the limitations of the \wlone, the \kwl{k} colors ordered subgraphs instead of a single node. More precisely, the \kwl{k} colors the tuples from $V(G)^k$ for $k \geq 2$ instead of the nodes. Of central importance to our tokenization is consistency with the initial coloring of the \kwl{k} and recovering the adjacency information between $k$-tuples imposed by the \kwl{k}, both of which we will describe hereafter; see \Cref{app:preliminaries} for a formal definition of the \kwl{k}.

The initial color of a $k$-tuple $\vec{v} \coloneqq (v_1, \dots, v_k) \in V(G)^k$ under the \kwl{k} depends on its \new{atomic type} and the labels $\ell(v_1), \dots, \ell(v_k)$. Intuitively, the atomic type describes the structural dependencies within elements in a tuple. We can represent the atomic type of $\vec{v}$ by a $k \times k$ matrix $\vec{K}$ over $\{ 1,2,3\}$. That is, the entry $\vec{K}_{ij}$ is 1 if $(v_i,v_j) \in E(G)$, 2 if $v_i = v_j$, and 3 otherwise; see \Cref{app:notation} for a formal definition of the atomic type. Hence, to encode the atomic type as a real vector, we can learn an embedding from the set of $k \times k$ matrices to $\mathbb{R}^e$, where $e > 0$ is the embedding dimension of the atomic type. For a tuple $\vec{v}$, we denote this embedding with $\embed{atp}(\vec{v})$. Apart from the initial colors for tuples, the \kwl{k} also imposes a notion of adjacency between tuples. Concretely, we define
\begin{equation*}
	\phi_j(\vec{v},w)\coloneqq (v_1, \dots, v_{j-1}, w, v_{j+1}, \dots, v_k),
\end{equation*}
i.e., $\phi_j(\vec{v},w)$ replaces the $j$-th component of the tuple $\vec{v}$ with the vertex $w$. We say that two tuples are \new{adjacent} or \new{$j$-neighbors} if they are different in the $j$-th component (or equal, in the case of self-loops). 

Now, to construct token embeddings consistent with the initial colors under the \kwl{k}, we initialize $n^k$ token embeddings $\vec{X}^{(0, k)} \in \mathbb{R}^{n^k \times d}$, one for each $k$-tuple. For order $k$ and embedding dimension $d$ of the tuple-level tokens, we first compute node-level tokens $\vec{X}^{(0,1)}$ in \Cref{eq:1_wl_token_embeddings} and, in particular, the structural embeddings $\vec{P}(v)$ for each node $v$, with an embedding dimension of $d$ and then concatenate node-level embeddings along the embedding dimension to construct tuple-level embeddings which are then projected down to fit the embedding dimension $d$. Specifically, we define the token embedding of a $k$-tuple $\vec{v} = (v_1, \dots, v_k)$ as
\begin{equation}\label{eq:kwl_token_embeddings}
    \vec{X}^{(0, k)}(\vec{v}) \coloneqq \big[
        \vec{X}^{(0,1)}(v_i)
    \big]_{i=1}^k \mathbf{W} + \embed{atp}(\vec{v}),
\end{equation}
where $\mathbf{W} \in \mathbb{R}^{d \cdot k \times d}$ is a projection matrix.
Intuitively, the above construction ensures that the token embeddings respect the initial node colors and the atomic type.
Analogously to our notion of adjacency-identifying, we use the structural embeddings $\vec{P}(v_i)$ %
in each node embedding $\vec{X}^{(0,1)}(v_i)$ to identify the $j$-neighborhood adjacency between tuples $\vec{v}$ and $\vec{w}$ in the $j$-th attention head. To reconstruct the $j$-neighborhood adjacency, we show that it is sufficient to identify the nodes in each tuple-level token. Hence, we define the following requirements for the structural embeddings $\vec{P}$.

\begin{definition}[Node-identifying]\label{def:node_identifying}
Let $G = (V(G), E(G), \ell)$ be a graph with $n$ nodes. Let $\vec{X}^{(0, k)} \in \mathbb{R}^{n^k \times d}$ denote the initial token embeddings according to \Cref{eq:kwl_token_embeddings} with structural embeddings $\vec{P} \in \mathbb{R}^{n \times d}$.
Let further
\begin{equation*}
   \Tilde{\mathbf{P}} \coloneqq \frac{1}{\sqrt{d_k}} \mathbf{P}\mathbf{W}^Q\big(\mathbf{P}\mathbf{W}^K\big)^T,
\end{equation*}
where $\vec{W}^Q,\vec{W}^K \in \mathbb{R}^{d \times d}$. 
Then $\vec{P}$ is \new{node-identifying} if there exists $\vec{W}^Q,\vec{W}^K$ such that for any row $i$ and column $j$, 
\begin{equation*}
    \Tilde{\mathbf{P}}_{ij} = \max_{k} \Tilde{\mathbf{P}}_{ik},
\end{equation*}
if, and only, if $i = j$. Further, an approximation $\vec{Q} \in \mathbb{R}^{n \times d}$ of $\vec{P}$ is sufficiently node-identifying if
\begin{equation*}
    \big|\big| \vec{P} - \vec{Q} \big|\big|_\text{F} < \epsilon, 
\end{equation*}
for any $\epsilon > 0$.
\end{definition}
As we will show, the above requirement allows the attention to distinguish whether two tuples share the same nodes, intuitively, by counting the number of node-to-node matches between two tuples, which is sufficient to determine whether the tuples are $j$-neighbors. For structural embeddings $\vec{P}$ to be node-identifying, it suffices, for example, that $\vec{P}$ has an orthogonal sub-matrix. 

In~\Cref{sec:implementation}, we show that structural embeddings are also sufficiently node-identifying when using LPE or SPE \citep{Huang+2023} as node-level PEs.
It should also be mentioned that our definition of node-identifying structural embeddings generalizes the node identifiers in \citet{Kim+2022}, i.e., their node identifiers are node-identifying. Still, structural embeddings exist that are (sufficiently) node-identifying and do not qualify as node identifiers in \citet{Kim+2022}.

We call a standard transformer using the above tokenization with sufficiently node- and adjacency-identifying structural embeddings the \kgt{k}. We then show the connection of the \kgt{k} to the \kwl{k}, resulting in the following theorem; see \Cref{app:proof_kgt} for proof details.
\begin{theorem}\label{theorem:k_wl}
Let $G=(V(G),E(G),\ell)$ be a labeled graph with $n$ nodes and $k \geq 2$ and $\vec{F} \in \mathbb{R}^{n \times d}$ be a node feature matrix consistent with $\ell$. Let $C^k_t$ denote the coloring function of the \kwl{k} at iteration $t$.
Then for all iterations $t \geq 0$, there exists a parametrization of the \kgt{k} such that
\begin{equation*}
    C^{k}_t(\vec{v}) = C^{k}_t(\vec{w}) \Longleftrightarrow \vec{X}^{(t, k)}(\vec{v}) = \vec{X}^{(t, k)}(\vec{w})
\end{equation*}
for all $k$-tuples $\vec{v}$ and $\vec{w} \in V(G)^k$.
\end{theorem}
Note that similar to \Cref{theorem:1_wl}, the above theorem gives a lower bound on the expressivity of the \kgt{k}. We now show that the \kgt{k} is strictly more powerful than the \kwl{k} by showing that the \kgt{k} can also simulate the $\delta$-\kwl{k}, a \kwl{k} variant that, for each $j$-neighbor $\phi_j(\vec{v}, w)$, additionally considers whether $(v_j, w) \in E(G)$ \citep{Morris2020b}. That is, the $\delta$-\kwl{k} distinguishes for each $j$-neighbor whether the replaced node and the replacing node share an edge in $G$. \citet{Morris2020b} showed that the $\delta$-\kwl{k} is strictly more expressive than the \kwl{k}. We use this to show that the \kgt{k} is strictly more expressive than the \kwl{k}, implying the following result; see again \Cref{app:proof_kgt} for proof details.
\begin{theorem}\label{theorem:delta_k_wl}
For $k>1$, the \kgt{k} is strictly more expressive than the \kwl{k}.
\end{theorem}
Note that the \kgt{k} uses $n^k$ tokens. However, for larger $k$ and large graphs, the number of tokens might present an obstacle in practice. Luckily, our close alignment of the \kgt{k} with the \kwl{k} hierarchy allows us to directly benefit from a recent result in \citet{Mor+2022b}, reducing the number of tokens while maintaining an expressivity guarantee. Specifically, \citet{Mor+2022b} define the set of $(k, s)$-tuples $V(G)^k_s$ as
\begin{equation*}
    V(G)^k_s \coloneqq \{ \vec{v} \in V(G)^k \mid \#\text{comp}(G[\vec{v}]) \leq s \},
\end{equation*}
where $\#\text{comp}(H)$ denotes the number of connected components in subgraph $H$, $G[\vec{v}]$ denotes the ordered subgraph of $G$, induced by the nodes in $\vec{v}$ and $s \leq k$ is a hyper-parameter. \citet{Mor+2022b} then define the \klwl{(k,s)} as a \kwl{k} variant that only considers the tuples in $V(G)^k_s$ and additional only use subset of adjacent tuples; see~\cite{Mor+2022b} and~\Cref{app:preliminaries} for a detailed description of the \kwl{(k,s)}. Fortunately, the runtime of the \klwl{(k,s)} depends only on $k$, $s$, and the sparsity of the graph, resulting in a more efficient and scalable \kwl{k} variant. We can now directly translate this modification to our token embeddings. Specifically, we call the \kgt{k} using only the token embeddings $\vec{X}^{ (0, k)}(\vec{v})$ where $\vec{v} \in V(G)^k_s$, the \kgt{(k,s)}.

\begin{table*}
    \centering
\caption{Comparison of our theoretical results with \citet{Kim+2021} and \citet{Kim+2022}. Highlighted are aspects in which our results improve over \citet{Kim+2022}. Here, $n$ denotes the number of nodes, and $m$ denotes the number of edges. We denote with $\sqsubset A$ that the transformer is strictly more expressive than algorithm $A$. Note that for pure transformers, the squared complexity can be relaxed to linear complexity when applying linear attention approximation such as the Performer \citep{Choromanski+2021}.
    }
    \resizebox{\textwidth}{!}{
    \begin{tabular}{lcccc}
    \toprule
       \textbf{Transformer} & Provable expressivity & Runtime complexity & \# Heads & Pure transformer \\ \midrule  
    Sparse Transformer \citep{Kim+2021} & $\sqsubset$ \, \citet{Gil+2017} & $\mathcal{O}(m)$ & 1 & \ding{55} \\
    TokenGT \citep{Kim+2022} & \wlone{} & $\mathcal{O}((n + m)^2)$ & 15 & \ding{51} \\
    \kgt{1} (ours) & \wlone{} & \cellcolor{GreenCellColor} $\mathcal{O}(n^2)$ & \cellcolor{GreenCellColor} 1 & \ding{51} \\
    \kgt{(2,1)} (ours) & \cellcolor{GreenCellColor} $\sqsubset$ \wlone{} & $\mathcal{O}((n + m)^2)$ & \cellcolor{GreenCellColor} 2 & \ding{51} \\ \midrule
    Higher-Order Transformer \citep{Kim+2021} & \kwl{k} & $\mathcal{O}(n^{2k})$ & 1 & \ding{55} \\
    $k$-order Transformer \citep{Kim+2022} & \kwl{k} & $\mathcal{O}(n^{2k})$ & \textit{bell}(2k) & \ding{51} \\
    \kgt{k} (ours) & \cellcolor{GreenCellColor} $\sqsubset$ \kwl{k} & $\mathcal{O}(n^{2k})$ & \cellcolor{GreenCellColor} $2k$ & \ding{51} \\
    \bottomrule
    \end{tabular}}
    \label{tab:comparison_to_Kim2022}
\end{table*}

Now, \citet[Theorem 1]{Mor+2022b} shows that, for $k > 1$, the \klwl{(k + 1,1)} is strictly stronger than the \klwl{(k,1)}. Further, note that the \klwl{(1,1)} is equal to the  \wlone{}. Using this result, we can prove the following; see \Cref{app:proof_kgt} for proof details.
\begin{theorem}\label{theorem:k_s_wl}
For all $k > 1$, the \kgt{(k,1)} is strictly more expressive than the \kgt{(k-1,1)}. Further, the \kgt{(2,1)} is strictly more expressive than the \wlone{}.
\end{theorem}
Note that the \kgt{(2,1)} only requires $\mathcal{O}(n + m)$ tokens, where $m$ is the number of edges, and consequently has a runtime complexity of $\mathcal{O}((n + m)^2)$, the same as TokenGT. Hence, the \kgt{(2,1)} improves the expressivity result of TokenGT, which is shown to have \wlone{} expressive power \citep{Kim+2022} while having the same runtime complexity.

In summary, our alignment of standard transformers with the \kwl{k} hierarchy has brought forth a variety of improved theoretical results for pure transformers, providing a strict increase in provable expressive power for every order $k > 0$ compared to previous works.
See \Cref{tab:comparison_to_Kim2022} for a direct comparison of our results with \citet{Kim+2021} and \citet{Kim+2022}. Most importantly, \Cref{theorem:k_s_wl} leads to a hierarchy of increasingly expressive pure transformers, taking one step closer to answering our guiding question in \Cref{sec:intro}. In what follows, %
we make our transformers feasible in practice.

\section{Implementation details}\label{sec:implementation}
Here, we discuss the implementation details of pure transformers. In particular, we demonstrate that Laplacian PEs \citep{Kre+2021} are sufficient to be used as node-level PEs for our transformers. Afterward, we introduce \new{order transfer}, a way to transfer model weights between different orders of the WL hierarchy; see \Cref{sec:add_implementation} for additional implementation details. 

\subsection{Node and adjacency-identifying PEs}
Here, we develop PEs based on Laplacian eigenvectors and -values that are both sufficiently node- \emph{and} adjacency-identifying, avoiding computing separate PEs for node- and adjacency identification; see \Cref{sec:add_implementation_laplacian} for details about the graph Laplacian and some background on PEs based on the spectrum of the graph Laplacian. 

Let $\lambda \coloneqq (\lambda_1, \dots, \lambda_l)^T$ denote the vector of the $l$ smallest, possibly repeated, eigenvalues of the graph Laplacian for some graph with $n$ nodes and let $\mathbf{V} \in \mathbb{R}^{n \times l}$ be a matrix such that the $i$-th column of $\mathbf{V}$ is the eigenvector corresponding to eigenvalue $\lambda_i$. We will now briefly introduce LPE, based on \citet{Kre+2021}, as well as SPE from \citet{Huang+2023} and show that these node-level PEs are node- \emph{and} adjacency-identifying.

\paragraph{LPE} Here, we define a slight generalization of the Laplacian PEs in \citet{Kre+2021} that enables sufficiently node and adjacency-identifying PEs. 
Concretely, we define LPE as
\begin{equation}\label{eq:1wl_proof_lap_def}
  \resizebox{0.9\hsize}{!}{$\embed{LPE}(\mathbf{V}, \mathbf{\lambda}) = \rho \Big( \begin{bmatrix}
      \phi(\vec{V}^T_1, \lambda + \epsilon) & \dots & \phi(\vec{V}^T_n, \lambda + \epsilon)
  \end{bmatrix}
 \Big),$}
\end{equation}
where $\epsilon \in \mathbb{R}^l$ is a learnable, zero-initialized vector, we introduce to show our result. Here, $\phi \colon \mathbb{R}^2 \rightarrow \mathbb{R}^d$ is an FFN that is applied row-wise and $\rho \colon \mathbb{R}^{l \times d} \rightarrow \mathbb{R}^{d}$ is a permutation-equivariant neural network, applied row-wise. One can recover the Laplacian PEs in \citet{Kre+2021} by setting $\epsilon = \mathbf{0}$ and implementing $\rho$ as first performing a sum over the first dimension of its input, resulting in an $d$-dimensional vector and then subsequently applying an FFN to obtain a $d$-dimensional output vector. Note that then,
\Cref{eq:1wl_proof_lap_def} forms a DeepSet \citep{ZaheerNIPS2017DeepSets} over the set of eigenvector components, paired with their ($\epsilon$-perturbed) eigenvalues.
While slightly deviating from the Laplacian PEs defined in \citet{Kre+2021}, the $\epsilon$ vector is necessary to ensure that no spectral information is lost when applying $\rho$.

We will now show that $\embed{LPE}$ is sufficiently node- and adjacency-identifying. Further, this result holds irrespective of whether the Laplacian is normalized. As a result, $\embed{LPE}$ can be used with \kgt{1}, \kgt{k}, and \kgt{(k,s)}; see \Cref{app:proof_lpe} for proof details.
\begin{theorem}\label{theorem:node_and_adj_ident}
Structural embeddings with $\embed{LPE}$ as node-level PE are sufficiently node- and adjacency-identifying.%
\end{theorem}

\paragraph{SPE} We also show that SPE \citep{Huang+2023} are sufficiently node- and adjacency-identifying. \citet{Huang+2023} define the SPE encodings as
\begin{equation*}
    \resizebox{0.975\hsize}{!}{$\mathsf{SPE}(\mathbf{V}, \vec{\lambda}) = \rho\Big(\begin{bmatrix}
        \mathbf{V} \text{diag}(\phi_1(\mathbf{\lambda})) \mathbf{V}^T & \dots & \mathbf{V} \text{diag}(\phi_n(\mathbf{\lambda})) \mathbf{V}^T
    \end{bmatrix}\Big),$}
\end{equation*}
where $\phi_1, \dots, \phi_n \colon \mathbb{R}^l \rightarrow \mathbb{R}^l$ are equivariant FFNs and $\rho \colon \mathbb{R}^{n \times l} \rightarrow \mathbb{R}^{d}$ is a permutation equivariant neural network, applied row-wise. For example, $\rho$ can be implemented by first performing a sum over the first dimension of its input, resulting in an $l$-dimensional vector, and then subsequently applying an FFN to obtain a $d$-dimensional output vector.

SPE have the benefit of being \new{stable}, meaning a small perturbation in the graph structure causes only a small change in the resulting SPE encoding. As a result, stability can be related to OOD generalization \citep{Huang+2023}.
We can show that SPE are sufficiently node- and adjacency-identifying. As a result, SPE encodings can be used with \kgt{1}, \kgt{k} and \kgt{(k,s)}; see \Cref{app:proof_spe} for proof details.
\begin{theorem}\label{theorem:spe_node_and_adj_ident}
Structural embeddings with SPE as node-level PE are sufficiently node- and adjacency-identifying.
\end{theorem}

\subsection{Order transfer}
Here, we describe order transfer, a strategy for effectively using the \kgt{k} for $k > 2$ in practice. The idea of order transfer consists of pre-training a transformer on a lower-order tokenization, such as on the \kgt{(2,1)} tokens, and subsequently fine-tuning the pre-trained weights on a higher-order transformer, such as the \kgt{2}, \kgt{3} or the \kgt{(3,1)}. We evaluate order transfer empirically in \Cref{sec:experiments_ft}. %

Order transfer could be useful when pre-training a higher-order model from scratch is infeasible. However, using a higher-order model on a smaller downstream task is feasible and beneficial to performance. Examples of such scenarios are (a) a large difference in the number of graphs in the pre-training dataset compared to the fine-tuning task, (b) a large difference in the size of the graphs in the pre-training dataset compared to the fine-tuning task, or (c) a specific fine-tuning task benefiting strongly from higher-order expressivity. As a result, order transfer is a promising application of expensive yet expressive higher-order models.

\section{Experimental evaluation}\label{sec:experiments}
Here, we conduct an experimental evaluation of our theoretical results. In particular, we consider two settings. To motivate the use of large-scale pure transformers for graph learning, we evaluate the empirical performance of our transformers under a pre-training/fine-tuning paradigm typically employed with large-scale language, vision, and graph models \citep{Devlin+2019, Dosovitskiy+2021, Ying2021}. This approach is especially promising in molecular learning, where downstream datasets often only contain a few thousand examples \citep{hu+2020+strategies, Lucio+2022}; see \Cref{sec:experiments_pt} for pre-training and \Cref{sec:experiments_ft} for fine-tuning.
Moreover, we benchmark the \kgt{(2,1)} and \kgt{(3,1)} on how well these models can leverage their expressivity in practice; see \Cref{sec:experiments_brec}.
The source code for all experiments is available at \url{https://github.com/luis-mueller/wl-transformers}.

\subsection{Pre-training}\label{sec:experiments_pt}
For pre-training, we train on \textsc{PCQM4Mv2}, one of the largest molecular regression datasets available \citep{Hu2021}. The dataset consists of roughly 3.8M molecules, and the task is to predict the HOMO-LUMO energy gap. Here, we pre-train the \kgt{(2,1)} for 2M steps with a cosine learning rate schedule with 60K warm-up steps on two A100 NVIDIA GPUs; see \Cref{tab:pcqm4mv2_params} in \Cref{app:experiments} for details on the hyper-parameters. %
For model evaluation, we use the code provided by \citet{Hu2021}, available at \url{https://github.com/snap-stanford/ogb}. Further, to demonstrate the effectiveness of our transformers on large node-level tasks, we additionally benchmark the \kgt{(2,1)} on \textsc{CS} and \textsc{Photo}, two transductive node classification datasets \citep{Shchur+2018+Pitfalls}.
On all three tasks, we compare to Graphormer \citep{Ying2021}, a strong graph transformer with modified attention, and TokenGT \citep{Kim+2022} as a pure transformer baseline. We use $12$ transformer layers, a hidden dimension of $768$, $16$ attention heads, and a GELU non-linearity \citep{Hendrycks+2016}. As node-level PEs, we compare both LPE and SPE, as defined in \Cref{sec:implementation}.

We present our results in \Cref{tab:pre_training} and observe that the \kgt{(2,1)} further closes the gap between pure transformers and graph transformers with graph inductive bias such as Graphormer \citep{Ying2021}. Considering that the \kgt{(2,1)} has still relatively low graph inductive bias and in light of recent advances in terms of quality and size of pre-training dataset \citep{Beaini+2023}, we expect this gap to further shrink with increasing data scale. Further, we find that on all three datasets, LPE are favorable or comparable to SPE.

\subsection{Fine-tuning}\label{sec:experiments_ft}
Here, we describe our fine-tuning experiments. When fine-tuning, we re-use the pre-trained weights for the transformer layers and randomly initialize new weights for task-specific feature encodings and the task head. To demonstrate that fine-tuning large pre-trained transformers for small downstream tasks is feasible even with a smaller compute budget, we ensure that all experiments can be run on a single A10 Nvidia GPU with 24GB RAM.

\paragraph{Molecular regression}
To determine whether pre-training can improve the fine-tuning performance of our transformers, we choose \textsc{Alchemy (12K)}, a small-scale molecular dataset with 12K molecules \citep{Mor+2022b}. Specifically, we evaluate three settings: (a) training the \kgt{(2,1)} from scratch, (b) fine-tuning the pre-trained \kgt{(2,1)} and (c) order transfer from \kgt{(2,1)} to \kgt{(3,1)}, where we re-use the pre-trained weights from the \kgt{(2,1)} pre-training but learn a new tokenizer with \kgt{(3,1)} tokens. We evaluate these settings both for LPE and SPE as node-level PEs.
We compare the results to three GNNs with node-level PEs based on the eigenvectors and -values of the graph Laplacian: SignNet, BasisNet, and SPE \citep{Huang+2023}; see \Cref{tab:fine_tuning_1} for results. Here, we find that even without pre-training, the \kgt{(2,1)} already performs well on \textsc{Alchemy}. Most notably, the \kgt{(2,1)} with SPE without fine-tuning already performs on par with SignNet.
Nonetheless, we observe significant improvements through pre-training \kgt{(2,1)}. Most notably, fine-tuning the pre-trained \kgt{(2,1)} with LPE or SPE beats all GNN baselines. 
Interestingly, training the \kgt{(2,1)} with SPE from scratch results in much better performance than training the \kgt{(2,1)} with LPE from scratch. However, once pre-trained, the \kgt{(2,1)} performs better with LPE than with SPE.
Moreover, we observe no improvements when performing order transfer to the \kgt{(3,1)}. However, the \kgt{(3,1)} with LPE and pre-trained weights from the \kgt{(2,1)} beats SignNet \citep{Lim+2022}, a strong GNN baseline. Further, \kgt{(3,1)} with SPE performs on par with SPE in \citet{Huang+2023}, the best of our GNN baselines.
Interestingly, order transfer improves over the \kgt{(2,1)} trained from scratch for both LPE and SPE. As a result, we hypothesize that LPE and SPE provide sufficient expressivity for this task but that pre-training is required to fully leverage their potential. Further, we hypothesize that the added \kgt{(3,1)} tokens lead to overfitting. 

We conclude that pre-training can help our transformers' downstream performance. Further, order transfer is a promising technique for fine-tuning downstream tasks, particularly those that benefit from higher-order representations. However, its benefits might be nullified in the presence of sufficiently expressive node-level PEs in combination with large-scale pre-training.

\begin{table}
    \centering
    \caption{Comparison of \kgt{(2,1)} to Graphormer and TokenGT. Results on \textsc{PCQM4Mv2} over a single random seed, as well as \textsc{CS} and \textsc{Photo} over 7 random seeds where we also report standard deviation. Results for baselines are taken from \citet{Kim+2022}.
    We highlight \first{best} and \second{second} best model on each dataset.}
    \label{tab:pre_training}
\resizebox{0.975\columnwidth}{!}{\begin{tabular}{lcccccc}\toprule
\multirow{2}{*}{\textbf{Model}} & \textsc{PCQM4Mv2} & \textsc{CS} & \textsc{Photo} \\ %
 &  \small Validation MAE $\downarrow$ & \small Accuracy $\uparrow$ & \small Accuracy $\uparrow$ \\\midrule
 Graphormer & \first{0.0864} & 0.791 \tiny $\pm$ 0.015 & 0.894 \tiny $\pm$ 0.004 \\
 TokenGT  & 0.0910 & 0.903 \tiny $\pm$ 0.004 & \first{0.949 \tiny $\pm$ 0.007} \\
 \kgt{(2,1)} + LPE   & \second{0.0870} & \first{0.924 \tiny $\pm$ 0.008} & 0.933 \tiny $\pm$ 0.013 \\
 \kgt{(2,1)} + SPE   & 0.0888 & \second{0.920 \tiny $\pm$ 0.002} & \second{0.933 \tiny $\pm$ 0.011} \\
\bottomrule
\end{tabular}}
\end{table}

\begin{table}[ht]
    \centering
    \caption{Effect of pre-training for fine-tuning performance on \textsc{Alchemy} (12K). Pre-trained weights for both \kgt{(2,1)} and \kgt{(3,1)} are taken from the \kgt{(2,1)} model in \Cref{tab:pre_training}. We report mean and standard deviation over 3 random seeds. Baseline results for SignNet, BasisNet and SPE are taken from \citet{Huang+2023}. We highlight the \first{best}, \second{second} best and \third{third} best model.}
    \label{tab:fine_tuning_1}
\resizebox{0.85\hsize}{!}{\begin{tabular}{lccccc}\toprule
\multirow{2}{*}{\textbf{Model}} & \multirow{2}{*}{Pre-trained} & \textsc{Alchemy} (12K) \\ \cmidrule{3-3}
 & & MAE $\downarrow$ \\\midrule
 SignNet & \ding{55}  & 0.113 \tiny $\pm$ 0.002  \\ 
 BasisNet & \ding{55} & 0.110 \tiny $\pm$ 0.001  \\
 SPE & \ding{55} & \third{0.108 \tiny $\pm$ 0.001}  \\ \midrule
 \multirow{2}{*}{\kgt{(2,1)} + LPE} & \ding{55} & 0.124 \tiny $\pm$ 0.001  \\
 & \ding{51} &  \first{0.101 \tiny $\pm$ 0.001} \\
 \kgt{(3,1)} + LPE & \ding{51} & 0.114 \tiny $\pm$ 0.001  \\ \midrule
 \multirow{2}{*}{\kgt{(2,1)} + SPE} & \ding{55} & 0.112 \tiny $\pm$ 0.000  \\
 & \ding{51} & \second{0.103 \tiny $\pm$ 0.002} \\ 
  \kgt{(3,1)} + SPE & \ding{51} & \third{0.108 \tiny $\pm$ 0.001}  \\
\bottomrule
\end{tabular}}
\end{table}

\begin{table*}
    \centering
    \caption{Fine-tuning results on small OGB molecular datasets compared to a fully-equipped and -tuned GIN. Results over 6 random seeds. %
    We highlight \first{best}, \second{second} best and \third{third} best model for each dataset. Ties are broken by smaller standard deviation.
    For comparison, we also report pre-training results from \citet{hu+2020+strategies}.}
    \label{tab:ogb}
\resizebox{0.85\textwidth}{!}{ 	
        \begin{tabular}{lcccccc}\toprule
\multirow{2}{*}{\textbf{Model}} & \multirow{2}{*}{Pre-trained} & \textsc{BBBP} & \textsc{BACE} & \textsc{ClinTox} & \textsc{Tox21} & \textsc{ToxCast} \\ \cmidrule{3-7}
 & & AUROC $\uparrow$ & AUROC $\uparrow$ & AUROC $\uparrow$ & AUROC $\uparrow$ & AUROC $\uparrow$ \\\midrule
 \multirow{2}{*}{GIN \citep{hu+2020+strategies}} & \ding{55}  & 0.658 {\tiny $\pm$ 0.045} & 0.701 {\tiny $\pm$ 0.054} & 0.580 {\tiny $\pm$ 0.044} & 0.740 {\tiny $\pm$ 0.008} & 0.634 {\tiny $\pm$ 0.006}  \\
 
  & \ding{51} & 0.688 {\tiny $\pm$ 0.008} & \first{0.845 {\tiny $\pm$ 0.007}} & 0.737 {\tiny $\pm$ 0.028} & \first{0.783 {\tiny $\pm$ 0.003}} & \second{0.665 {\tiny $\pm$ 0.003}} \\ \midrule
 GINE+LPE & \ding{55} & 0.670 \tiny $\pm$	0.016 & 0.763 \tiny $\pm$ 0.007 & 0.824 \tiny $\pm$	0.016 & 0.763 \tiny $\pm$	0.011 & \third{0.659 \tiny $\pm$ 0.005} \\
 
 GINE+SPE & \ding{55} & 0.613 \tiny $\pm$ 0.036 & 0.677 \tiny $\pm$ 0.047 & 0.641 \tiny $\pm$ 0.066 & 0.709 \tiny $\pm$	0.023 & 0.593 \tiny $\pm$ 0.018 \\ \midrule
 
 \kgt{(2,1)} +LPE & \ding{55} & \second{0.703\tiny $\pm$	0.025} & 0.786\tiny $\pm$	0.019 & 0.821\tiny $\pm$	0.045 & \third{0.763\tiny $\pm$	0.003} & \first{0.667\tiny $\pm$	0.007} \\

  \kgt{(2,1)} +LPE & \ding{51} & 0.679\tiny $\pm$	0.012 & \second{0.815\tiny $\pm$	0.017} & \first{0.867\tiny $\pm$	0.020 }& \second{0.780\tiny $\pm$	0.005} & 0.642\tiny $\pm$	0.007 \\

 \midrule

  \kgt{(2,1)}+SPE & \ding{55} & \third{0.694\tiny $\pm$	0.020} & 0.776\tiny $\pm$	0.030 & \third{0.845\tiny $\pm$	0.017} & 0.758\tiny $\pm$	0.006 & 0.657\tiny $\pm$	0.007 \\

  \kgt{(2,1)}+SPE & \ding{51} & \first{0.703\tiny $\pm$	0.013} & \third{0.793\tiny $\pm$	0.025} & \second{0.859\tiny $\pm$	0.021} & 0.762\tiny $\pm$	0.003 & 0.637\tiny $\pm$	0.003 \\
\bottomrule
\end{tabular}}
\end{table*}

\paragraph{Molecular classification}
Here, we evaluate whether fine-tuning a large pre-trained transformer can be competitive with a strong GNN baseline on five small-scale molecular classification tasks, namely \textsc{BBBP}, \textsc{BACE}, \textsc{ClinTox}, \textsc{Tox21}, and \textsc{ToxCast}  \citep{hu2020ogb}. The number of molecules in any of these datasets is lower than 10K, making them an ideal choice to benchmark whether large-scale pre-trained models such as the \kgt{(2,1)} can compete with a task-specific GNN.
Specifically, following \citet{Toenshoff+2023+Gap}, we aim for a fair comparison by carefully designing and hyper-parameter tuning a GINE model \citep{Xu+2019} with residual connections, batch normalization, dropout, GELU non-linearities and, most crucially, the same node-level PEs as the \kgt{(2,1)}.  
We followed \citet{hu+2020+strategies} in choosing the GINE layer over other GNN layers due to its guaranteed \wlone{} expressivity. It is worth noting that our GINE with LPE encodings significantly outperforms the GIN in \citet{hu+2020+strategies} without pre-training on all five datasets, demonstrating the quality of this baseline. Interestingly, GINE with SPE as node-level PEs underperforms against the GIN in \citet{hu+2020+strategies} on all but one datasets.
In addition, we report the best pre-trained model for each task in \citet{hu+2020+strategies}.
We fine-tune the \kgt{(2,1)} for 5 epochs and train the GINE model for 45 epochs. In addition, we also train the \kgt{(2,1)} for 45 epochs from scratch to study the impact of pre-training on the molecular classification tasks.
Note that we do not pre-train our GINE model to evaluate whether a large pre-trained transformer can beat a small GNN model trained from scratch; see \Cref{tab:ogb} for results. 

First, we find that the pre-trained \kgt{(2,1)} with SPE is better than the corresponding GINE with SPE on all five datasets.
The pre-trained \kgt{(2,1)} with LPE is better than GINE with LPE on all but one dataset. Moreover, the \kgt{(2,1)} with LPE as well as the \kgt{(2,1)} with SPE are each better or on par with the best pre-trained GIN in \citet{hu+2020+strategies} in two out of five datasets.
When studying the impact of pre-training, we observe that pre-training leads mostly to performance improvement. The notable exception is on the \textsc{ToxCast} dataset, where the pre-trained \kgt{(2,1)} underperforms even the GIN baselines without pre-training.
Hence, we conclude that the pre-training/fine-tuning paradigm is viable for applying pure transformers such as the \kgt{(2,1)} to small-scale problems.

\subsection{Expressivity tests}\label{sec:experiments_brec}
Since we only provide expressivity lower bounds, we empirically investigate the expressive power of the \kgt{k}. To this end, we evaluate the \textsc{Brec} benchmark offering fine-grained and scalable expressivity tests \citep{WangZhang2023BREC}. The benchmark is comprised of $400$ graph pairs that range from being \wlone{} to \kwl{4} indistinguishable and pose a challenge even to the most expressive models; see \Cref{tab:brec} for the expressivity results of \kgt{(2,1)} and \kgt{(3,1)} on \textsc{Brec}. We mainly compare to our GINE baseline, Graphormer \citep{Ying2021} as a graph transformer baseline, and the \kwl{3} as a potential expressivity upper-bound. In addition, we compare our GINE model, \kgt{(2,1)} and \kgt{(3,1)} for both LPE and SPE. We find that SPE encodings consistently lead to better performance. Further, we observe that increased expressivity also leads to better performance, as for both LPE and SPE, the \kgt{(2,1)} beats our GINE baseline, and the \kgt{(3,1)} beats the \kgt{(2,1)} across all tasks. Finally, we find that both the \kgt{(2,1)} and the \kgt{(3,1)} improve over Graphormer while still being outperformed by the \kwl{3}, irrespective of the choice of PE.

\begin{table}%
\caption{Results on the \textsc{BREC} benchmark over a single seed. Baseline results are taken from \citet{WangZhang2023BREC}. We highlight the best model for each PE and category (excluding the \kwl{3}, which is not a model and merely serves as a reference point). We additionally report the results of Graphormer from \citet{WangZhang2023BREC}.}
\label{tab:brec}
\centering
\resizebox{\columnwidth}{!}{ 	
\begin{tabular}{lcccccc}\toprule
\textbf{Model} & \textbf{PE} & Basic & Regular & Extension & CFI & \textit{All} \\ \midrule
GINE & \multirow{3}{*}{LPE} & 19 & 20 & 37 & \textbf{3}  & 79 \\
\kgt{(2,1)} &  & 36 & 28 & 51 & \textbf{3} & 118 \\
\kgt{(3,1)} &  & \textbf{55} & \textbf{46} & \textbf{55} & \textbf{3} & \textbf{159} \\\midrule
GINE & \multirow{3}{*}{SPE} & 56 & 48 & 93 & 20 & 217 \\
\kgt{(2,1)} &  & \textbf{60} & \textbf{50} & \textbf{98} & 18 &  226 \\
\kgt{(3,1)} &  & \textbf{60} & \textbf{50} & 97 & \textbf{21} &  \textbf{228}\\ \midrule \midrule
Graphormer  &  & 16 & 12 & 41 & 10 & 79 \\
\kwl{3} &  & 60 & 50 & 100 & 60 & 270 \\
\bottomrule
\end{tabular}} 
\end{table}

\section{Conclusion}
In this work, we propose a hierarchy of expressive pure transformers that are also feasible in practice. We improve existing pure transformers such as TokenGT \citep{Kim+2022} in several aspects, both theoretically and empirically. Theoretically, our hierarchy has stronger provable expressivity for each $k$ and is more scalable. For example, our \kgt{(2,1)} and \kgt{(3,1)} have a provable expressivity strictly above \wlone{} but are also feasible in practice.
Empirically, we verify our claims about practical feasibility and show that our transformers improve over existing pure transformers, closing the gap between pure transformers and graph transformers with strong graph inductive bias on the large-scale \textsc{PCQM4Mv2} dataset. Further, we show that fine-tuning our pre-trained transformers is a feasible and resource-efficient approach for applying pure transformers to small-scale datasets. We discover that higher-order transformers can efficiently re-use pre-trained weights from lower-order transformers during fine-tuning, indicating a promising direction for utilizing higher-order expressivity to improve results on downstream tasks. Future work could explore aligning transformers to other variants of Weisfeiler--Leman. %
Finally, pre-training pure transformers on truly large-scale datasets such as those recently proposed by \citet{Beaini+2023} could be a promising direction for graph transformers.

\newpage
\section*{Impact statement}
This paper presents work whose goal is to advance the field of machine learning. Our work has many potential societal consequences, none of which must be specifically highlighted here.

\section*{Acknowledgements}
CM and LM are partially funded by a DFG Emmy Noether grant (468502433) and RWTH Junior Principal Investigator Fellowship under Germany’s Excellence Strategy. 

\bibliography{bibliography}
\bibliographystyle{icml2024}

\newpage
\appendix
\onecolumn

\section{Additional implementation details}\label{sec:add_implementation}
Here, we outline additional implementation details that make our pure transformers more feasible in practice.

\subsection{Positional encodings based on Laplacian eigenmaps}\label{sec:add_implementation_laplacian}
Encoding Laplacian eigenmaps are a popular choice for PEs in GNNs and graph transformers \citep{Kre+2021,Lim+2022,Mue+2023,rampavsek2022recipe}. In particular, \citet{Lim+2022} show that their PEs based on Laplacian eigenmaps generalize other PEs, such as those based on random walks or heat kernels. In what follows, we will briefly review some background on Laplacian eigenvectors and -values as input to machine learning models and then show how such PEs can be constructed to be node- and adjacency-identifying.

For an $n$-order graph $G$, the \new{graph Laplacian} is defined as $\vec{L} \coloneqq \vec{D} - \vec{A}(G)$, where $\vec{D}$ denotes the degree matrix and $\vec{A}(G)$ denotes the adjacency matrix. We then consider the eigendecomposition of $\vec{L} = \vec{V}\vec{\Sigma}\vec{V}^T$, the column vectors of $\vec{V}$ correspond to eigenvectors and the diagonal matrix $\vec{\Sigma}$ contains the corresponding eigenvalues of the graph Laplacian. Since $\vec{L}$ is real and symmetric, such decomposition always exists; see, e.g., Corollary 2.5.11 in \citet{Horn+2012}. In addition, the eigenvalues of $\vec{L}$ are real and non-negative.
Some works also consider the \new{normalized graph Laplacian} $\Tilde{\vec{L}} \coloneqq D^{-\frac{1}{2}} \vec{L} D^{-\frac{1}{2}}$ \citep{Kre+2021, Lim+2022}. The statements we made for the graph Laplacian also hold for the normalized graph Laplacian, i.e., $\Tilde{\vec{L}}$ is real and symmetric, and its eigenvalues are real and non-negative. Further, the eigenvectors of real and symmetric matrices are orthonormal \citep{Lim+2022}.

We need to address the following to use Laplacian eigenvectors for machine learning. Let $\vec{v}$ be an eigenvector of a matrix $\vec{M}$ with eigenvalue $\lambda$. Then, $-\vec{v}$ is also an eigenvector of $\vec{M}$ with eigenvalue $\lambda$. As a result, we want a machine learning model to be invariant to the signs of the eigenvectors. \citet{Kre+2021} address this issue by randomly flipping the sign of each eigenvector during training for the model to learn this invariance from data. Concretely, given $k$ eigenvectors $\vec{v}_1, \dots, \vec{v}_k$, we uniformly sample $k$ independent sign values $s_1, \dots, s_k \in \{-1, 1\}$ and then plug eigenvectors $s_1\cdot\vec{v}_1, \dots, s_k \cdot \vec{v}_k$ into the model during training. %

\subsection{Implementation of LPE and SPE}
Here, we describe our concrete implementations of LPE and SPE based on our definition in \Cref{sec:implementation}.

For LPE, we implement $\rho$ by performing a sum over its input's first dimension and applying an FFN, as described in \Cref{sec:implementation}. We choose this implementation to stay as close to implementing the Laplacian PEs in \citet{Kre+2021}.

For SPE, we follow the implementation in \citet{Huang+2023}. Concretely, let
\begin{equation*}
    \vec{Q} \coloneqq \begin{bmatrix}
        \mathbf{V} \text{diag}(\phi_1(\mathbf{\lambda})) \mathbf{V}^T & \dots & \mathbf{V} \text{diag}(\phi_n(\mathbf{\lambda})) \mathbf{V}^T
    \end{bmatrix} \in \mathbb{R}^{n \times n \times l}
\end{equation*}
be the input tensor to $\rho$. \citet{Huang+2023} propose to partition $\vec{Q}$ along the second axis into $n$ matrices of size $n \times l$ and then to apply a GIN to each of those $n$ matrices in parallel where we use the adjacency matrix of the original graph. Concretely, the GIN maps each $n \times l$ matrix to a matrix of shape $n \times d$, where $d$ is the output dimension of $\rho$. Finally, the output of $\rho$ is the sum of all $n \times d$ matrices. We adopt this implementation for our experiments. 
For prohibitively large graphs, we propose the following modification to the above implementation of $\rho$. Specifically, we let $\mathbf{V}_{:m}$ denote the matrix containing as columns the eigenvectors corresponding to the $m$ smallest eigenvalues. Then, we define
\begin{equation}\label{eq:spe_low_rank}
    \vec{Q}_{:m} \coloneqq \begin{bmatrix}
        \mathbf{V} \text{diag}(\phi_1(\mathbf{\lambda})) \mathbf{V}^T_{:m} & \dots & \mathbf{V} \text{diag}(\phi_n(\mathbf{\lambda})) \mathbf{V}^T_{:m}
    \end{bmatrix} \in \mathbb{R}^{n \times m \times l}
\end{equation}
and use the same $\rho$ as for the case $m = n$.

\subsection{Atomic types from edge embeddings}
Here, we discuss a practical implementation of the atomic type embeddings $\embed{atp}$ in \Cref{sec:kwl_pure}. In particular, we use the fact that most graph learning datasets also include edge types, for example, the bond type in molecular datasets. Now, for an undirected graph $G$ and a tuple $\vec{v} = (v_1, \dots, v_k) \in V(G)^k_s$, let $\vec{E}(v_i, v_j) \in \mathbb{R}^d$ denote a learnable embedding of the edge type between nodes $v_i$ and $v_j$, where we additionally set $\vec{E}(v_i, v_j) = \vec{0}$ if, and only, if $v_i$ and $v_j$ do not share an edge in $G$. Further, since we consider graphs without self-loops, we can assign a special learnable vector $\vec{s}$ to $\vec{E}(v_i, v_j)$ for $i = j$.
Then, we can define 
\begin{equation*}
    \embed{atp}(\vec{v}) \coloneqq \big[
        \vec{E}(v_i, v_j)
    \big]_{i \geq j}^k \vec{W},
\end{equation*}
where $\vec{W} \in \mathbb{R}^{\sfrac{d(k^2 - k)}{2} \times d}$ is a learnable weight matrix projecting the concatenated edge embeddings to the target dimension $d$. To see that the above faithfully encodes the atomic type, recall the matrix $\vec{K}$ over $\{1,2,3\}$, determining the atomic type that we introduced in \Cref{sec:kwl_pure}. For undirected graphs, for all $i \geq j$, $\vec{K}_{ij} = 1$ if $\vec{E}(v_i, v_j) \neq \vec{0}$ and $\vec{E}(v_i, v_j) \neq \vec{s}$, $\vec{K}_{ij} = 2$ if $\vec{E}(v_i, v_j) = \vec{s}$ and $\vec{K}_{ij} = 3$ if $\vec{E}(v_i, v_j) = \vec{0}$.
As a result, by including additional edge information into the \kgt{k}, we simultaneously obtain atomic type embeddings $\embed{atp}$.

\begin{table}
    \centering
    \caption{Hyper-parameters for \kgt{(2,1)} pre-training on \textsc{PCQM4Mv2}.}
    \label{tab:pcqm4mv2_params}
\resizebox{.3\textwidth}{!}{
            \begin{tabular}{c|c}
    \toprule
      \textbf{Parameter}   & \textbf{Value}  \\ \midrule
        Learning rate & $2e-4$ \\
        Weight decay & 0.1 \\
        Attention dropout & 0.1 \\
        Post-attention dropout & 0.1 \\
        Batch size & 256 \\
        \# gradient steps & 2M \\
        \# warmup steps & 60K \\
        precision & \texttt{bfloat16} \\
    \bottomrule
    \end{tabular}}
\end{table}

\section{Additional experimental details}\label{app:experiments}
Here, we present the hyper-parameters used during pre-training in \Cref{tab:pcqm4mv2_params}. Further, we describe the hyper-parameter selection strategies employed for the different experiments. Across all experiments, we always select the hyper-parameters based on the best validation score and then evaluate on the test set.

For \textsc{CS} and \textsc{Photo}, we set the learning rate to $0.0001$ and tune dropout over $\{0.5, 0.1\}$, the hidden dimension over $\{512, 1024\}$, the number of attention heads over $\{1, 2, 4\}$, the number of eigenvalues used over $\{64, 96\}$. For the node classification datasets, we disable the learning rate scheduler. We train for 100 epochs. Due to the large number of nodes and edges on these datasets, we use the SPE node-level PEs according to \Cref{eq:spe_low_rank} with $m=384$.

For \textsc{Alchemy}, we only tune the learning rate. Specifically, for the \kgt{(2,1)} we sweep over $\{0.001, 0.0005, 0.0003\}$ and for the \kgt{(3,1)}, due to the additional computational demand over $\{0.0005, 0.0003\}$. For the \kgt{(2,1)} we train for $2000$ epochs and for the \kgt{(3,1)} we train for $500$ epochs, again due to the additional computational demand.

For the molecular classification datasets, we tune the transformers on learning rate and dropout. Specifically, we sweep the learning rate over $\{0.0005, 0.0001, 0.00005, 0.00001\}$ and dropout over $\{0.5, 0.1\}$. 
To ensure that the GINE baseline is representative, we invest more time for hyper-parameter tuning. Specifically, we sweep over the number of layers over $\{3, 6, 12\}$, the hidden dimension over $\{384, 768\}$. Further, we sweep the learning rate over $\{0.0005, 0.0001, 0.00005, 0.00001\}$ and dropout over $\{0.5, 0.1\}$. Finally, we sweep the choice of pooling function over $\{\text{mean}, \text{sum}\}$, used for pooling the node embeddings of the GINE into a single graph-level representation. This results in $12\times$ more hyper-parameter configurations for the GINE baseline than the transformers. 

For \textsc{BREC}, we use a batch size of $16$, $6$ layers and a hidden dimension of $768$ for both GNN and \kgt{(2,1)}. For the \kgt{(3,1)}, we use a batch size of $2$, $3$ layers and a hidden dimension of $192$.

\section{Extended notation}\label{app:notation}
The \new{neighborhood} of a vertex $v$ in $V(G)$ is denoted by $N(v) \coloneqq  \{ u \in V(G) \mid (v, u) \in E(G) \}$ and the \new{degree} of a vertex $v$ is  $|N(v)|$. Two graphs $G$ and $H$ are \new{isomorphic}, and we write $G \simeq H$ if there exists a bijection $\varphi \colon V(G) \to V(H)$ preserving the adjacency relation, i.e., $(u,v)$ is in $E(G)$ if and only if $(\varphi(u),\varphi(v))$ is in $E(H)$. Then $\varphi$ is an \new{isomorphism} between $G$ and $H$. In the case of labeled graphs, we additionally require that $l(v) = l(\varphi(v))$ for $v$ in $V(G)$, and similarly for attributed graphs. 
We further define the atomic type $\text{atp} \colon V(G)^k \to \Nb$, for $k > 0$, such that $\text{atp}(\vec{v}) = \text{atp}(\vec{w})$ for $\vec{v}$ and $\vec{w}$ in $V(G)^k$ if and only if the mapping $\varphi\colon V(G)^k \to V(G)^k$ where $v_i \mapsto w_i$ induces a partial isomorphism, i.e., we have $v_i = v_j \iff w_i = w_j$ and $(v_i,v_j) \in E(G) \iff (\varphi(v_i),\varphi(v_j)) \in E(G)$. 
Let $\vec{M} \in \mathbb{R}^{n \times p}$ and $\vec{N} \in \mathbb{R}^{n \times q}$ be two matrices then
\begin{equation*}
    \begin{bmatrix} \vec{M} & \vec{N} \end{bmatrix} \in \mathbb{R}^{n \times p+q}
\end{equation*}
denotes column-wise matrix concatenation.
Further, let $\vec{M} \in \mathbb{R}^{p \times n}$ and $\vec{N} \in \mathbb{R}^{q \times n}$ be two matrices then
\begin{equation*}
    \begin{bmatrix} \vec{M} \\ \vec{N} \end{bmatrix} \in \mathbb{R}^{p+q \times n}
\end{equation*}
denotes row-wise matrix concatenation. For a matrix $\vec{X} \in \mathbb{R}^{n \times d}$, we denote with $\vec{X}_i$ the $i$-th row vector. In the case where the rows of $\vec{X}$ correspond to nodes in a graph $G$, we use $\vec{X}_v$ to denote the row vector corresponding to the node $v \in V(G)$.

\section{Transformers}\label{sec:background_transformer}
Here, we define the (standard) \new{transformer} \citep{Vaswani2017}, a stack of alternating blocks of \new{multi-head attention} and fully-connected \new{feed-forward networks}.
In each layer, $t > 0$, given token embeddings $\mathbf{X}^{(t-1)} \in \mathbb{R}^{L \times d}$ for $L$ tokens, we compute
\begin{equation}\label{eq:standard_transformer}
    \vec{X}^{(t)} \coloneqq \mathsf{FFN} \big( \vec{X}^{(t-1)} + \big[
        h_1(\vec{X}^{(t-1)}) \, \ldots \, h_M(\vec{X}^{(t-1)})
    \big] \vec{W}^O  \big),
\end{equation}
where $[\cdot]$ denotes column-wise concatenation of matrices, $M$ is the number of heads, $h_i$ denotes the $i$-th transformer head, $\vec{W}^O \in \mathbb{R}^{Md_v \times d}$ denotes a final projection matrix applied to the concatenated heads, and $\mathsf{FFN}$ denotes a feed-forward neural network applied row-wise.
We define the $i$-th head as
\begin{equation}\label{eq:1wl_tf_head}
    h_i(\vec{X}) \coloneqq \mathsf{softmax} \big( \frac{1}{\sqrt{d_k}} \mathbf{X}\mathbf{W}^{Q,i}\big(\mathbf{X}\mathbf{W}^{K,i}\big)^T \big) \mathbf{X} \vec{W}^{V, i},
\end{equation}
where the softmax is applied row-wise and $\vec{W}^{Q, i}, \vec{W}^{K, i}  \in \Rb^{d \times d_k}$,  $\vec{W}^{V, i}  \in \Rb^{d \times d_v}$, and $d_v$ and $d$ are the head dimension and embedding dimension, respectively. We omit layer indices $t$ and optional bias terms for clarity.

\section{Weisfeiler--Leman}\label{app:preliminaries}
Here, we discuss additional background for the Weisfeiler--Leman hierarchy. We begin by describing the Weisfeiler--Leman algorithm, starting with the \wlone. The \wlone{} or color refinement is a well-studied heuristic for the graph isomorphism problem, originally proposed by~\citet{Wei+1968}.\footnote{Strictly speaking, the \wlone{} and color refinement are two different algorithms. The \wlone{} considers neighbors and non-neighbors to update the coloring, resulting in a slightly higher expressive power when distinguishing vertices in a given graph; see~\cite {Gro+2021} for details. For brevity, we consider both algorithms to be equivalent.} 
Intuitively, the algorithm determines if two graphs are non-isomorphic by iteratively coloring or labeling vertices. Formally, let $G = (V,E,\ell)$ be a labeled graph, in each iteration, $t > 0$, the \wlone{} computes a vertex coloring $C^1_t \colon V(G) \to \Nb$, depending on the coloring of the neighbors. That is, in iteration $t>0$, we set
\begin{equation*}
	C^1_t(v) \coloneqq \REL\Big(\!\big(C^1_{t-1}(v),\oms C^1_{t-1}(u) \mid u \in N(v)  \cms \big)\! \Big),
\end{equation*}
for all vertices $v$ in $V(G)$,
where $\REL$ injectively maps the above pair to a unique natural number, which has not been used in previous iterations. In iteration $0$, the coloring $C^1_{0}\coloneqq \ell$. To test if two graphs $G$ and $H$ are non-isomorphic, we run the above algorithm in ``parallel'' on both graphs. If the two graphs have a different number of vertices colored $c$ in $\Nb$ at some iteration, the \wlone{} \new{distinguishes} the graphs as non-isomorphic. 
It is easy to see that the algorithm cannot distinguish all non-isomorphic graphs~\citep{Cai+1992}. Several researchers, e.g.,~\citet{Bab1979,Cai+1992}, devised a more powerful generalization of the former, today known as the $k$-dimensional Weisfeiler--Leman algorithm (\kwl{k}), operating on $k$-tuples of vertices rather than single vertices.

\subsection{The \texorpdfstring{$k$}{k}-dimensional Weisfeiler--Leman algorithm}\label{kwl_intro}

Due to the shortcomings of the $\wlone$ or color refinement in distinguishing non-isomorphic graphs, several researchers, e.g.,~\citet{Bab1979,Cai+1992}, devised a more powerful generalization of the former, today known
as the $k$-dimensional Weisfeiler-Leman algorithm, operating on $k$-tuples of vertices rather than single vertices. 

Intuitively, to surpass the limitations of the \wlone, the \kwl{k} colors vertex-ordered $k$-tuples instead of a single vertex. More precisely, given a graph $G$, the \kwl{k} colors the tuples from $V(G)^k$ for $k \geq 2$ instead of the vertices. By defining a neighborhood between these tuples, we can define a coloring similar to the \wlone. Formally, let $G$ be a graph, and let $k \geq 2$. In each iteration, $t \geq 0$, the algorithm, similarly to the \wlone, computes a
\new{coloring} $C^k_t \colon V(G)^k \to \Nb$. In the first iteration, $t=0$, the tuples $\vec{v}$ and $\vec{w}$ in $V(G)^k$ get the same
color if they have the same atomic type, i.e.,
$C^k_{0}(\vec{v}) \coloneqq \text{atp}(\vec{v})$. Then, for each iteration, $t > 0$, $C^k_{t}$ is defined by
\begin{equation}\label{vr_ext_app}
	C^k_{t}(\vec{v}) \coloneqq \REL \big(C^k_{t-1}(\vec{v}), M_t(\vec{v}) \big),
\end{equation}
with $M_t(\vec{v})$ the multiset
\begin{equation}\label{mi}
	M_t(\vec{v}) \coloneqq  \big( \{\!\! \{  C^{k}_{t-1}(\phi_1(\vec{v},w)) \mid w \in V(G) \} \!\!\}, \dots, \{\!\! \{  C^{k}_{t-1}(\phi_k(\vec{v},w)) \mid w \in V(G) \} \!\!\} \big),
\end{equation}
and where
\begin{equation*}
	\phi_j(\vec{v},w)\coloneqq (v_1, \dots, v_{j-1}, w, v_{j+1}, \dots, v_k).
\end{equation*}
That is, $\phi_j(\vec{v},w)$ replaces the $j$-th component of the tuple $\vec{v}$ with the vertex $w$. Hence, two tuples are \new{adjacent} or \new{$j$-neighbors} if they are different in the $j$-th component (or equal, in the case of self-loops). Hence, two tuples $\vec{v}$ and $\vec{w}$ with the same color in iteration $(t-1)$ get different colors in iteration $t$ if there exists a $j$ in $[k]$ such that the number of $j$-neighbors of $\vec{v}$ and $\vec{w}$, respectively, colored with a certain color is different.

We run the \kwl{k} algorithm until convergence, i.e., until for $t$ in $\Nb$
\begin{equation*}
	C^k_{t}(\vec{v}) = C^k_{t}(\vec{w}) \iff C^k_{t+1}(\vec{v}) = C^k_{t+1}(\vec{w}),
\end{equation*}
for all $\vec{v}$ and $\vec{w}$ in $V(G)^k$, holds.

Similarly to the \wlone, to test whether two graphs $G$ and $H$ are non-isomorphic, we run the \kwl{k} in ``parallel'' on both graphs. Then, if the two graphs have a different number of vertices colored $c$, for $c$ in $\Nb$, the \kwl{k} \textit{distinguishes} the graphs as non-isomorphic. By increasing $k$, the algorithm gets more powerful in distinguishing non-isomorphic graphs, i.e., for each $k \geq 2$, there are non-isomorphic graphs distinguished by $(k+1)$\text{-}\textsf{WL} but not by \kwl{k}~\citep{Cai+1992}. In the following, we define some variants of the \kwl{k}.

\subsection{The \texorpdfstring{$\delta$-$k$}{delta-k}-dimensional Weisfeiler--Leman algorithm}
\citet{Mal2014} introduced the following variant of the \kwl{k}, the \deltakwl, which updates $k$-tuples according to
\begin{align*}
	M^{\text{adj}}_{t}(\vec{v}) \coloneqq \big( &\oms  (C^{k,\text{adj}}_{t-1}(\phi_1(\vec{v},w)), \text{adj}(v_1,w) ) \mid w \in V(G) \cms, \dots,\\
	&\oms (C^{k,\text{adj}}_{t-1}(\phi_k(\vec{v},w)),  \text{adj}(v_k,w)) \mid w \in V(G) \cms \big),
\end{align*}
 where 
\begin{equation*}
	\text{adj}(v,w) \coloneqq \begin{cases}
		1, \text{ if } (v,w) \in E(G) \\
		0, \text{ otherwise, }
	\end{cases}
\end{equation*}
resulting in the coloring function $C^{k,\text{M}}_{t} \colon V(G)^k \to \Nb$. \citet{Morris2020b} showed that this variant is slightly more powerful in distinguishing non-isomorphic graphs compared to the \kwl{k}.

\subsection{The local \texorpdfstring{$\delta$-$k$}{delta-k}-dimensional Weisfeiler--Leman algorithm}
\citet{Morris2020b} introduced a more efficient variant of the \kwl{k}, the \new{local $\delta$-$k$-dimensional Weisfeiler--Leman algorithm} (\localkwl), which updates $k$-tuples according to
\begin{equation*}\label{eqnmidd_ext}
	\begin{split}
		M^{\delta}_t(\vec{v}) =   \big( \{\!\! \{ C^{k, \delta}_{t-1}(\phi_1(\vec{v},w)) \mid w \in N(v_1) \} \!\!\}, \dots, \{\!\! \{  C^{k, \delta}_{t-1}(\phi_k(\vec{v},w)) \mid w \in N(v_k) \}  \!\!\} \big),
	\end{split}
\end{equation*}		
resulting in the coloring function $C^{k,\delta}_{t} \colon V(G)^k \to \Nb$.

\subsection{The local \texorpdfstring{$(k,s)$}{(k,s)}-dimensional Weisfeiler--Leman algorithm}

Let $G$ be a graph. Then $\text{\#comp}(G)$ denotes the number of (connected) components of $G$. Further, let $k \geq 1$ and $1 \leq s \leq k$, then
\begin{equation*}
	V(G)^k_s \coloneqq \{ \vec{v} \in V(G)^k \mid \text{\#comp}(G[\vec{v}]) \leq s  \}
\end{equation*}
is the set of \new{$(k,s)$-tuples} of nodes, i.e, $k$-tuples which induce (sub-)graphs with at most $s$ (connected) components. In contrast to the algorithms of~\cref{vr_ext_app}, the $(k,s)$-\textsf{LWL} colors tuples from  $V(G)^k_s$ instead of the entire $V(G)^k$, resulting in the coloring $C_t^{k,s} \colon V(G)^k_s \to \mathbb{N}$.
For more details; see \citet{Mor+2022b}.

\subsection{Comparing \texorpdfstring{\kwl{k}}{k-WL} variants} Let $A_1$ and $A_2$ denote \kwl{k}-like algorithms, we write $A_1 \sqsubseteq A_2$ if $A_1$ distinguishes between all non-isomorphic pairs $A_2$ does, and $A_1 \equiv A_2$ if both directions hold. The corresponding strict relation is denoted by $\sqsubset$. We can extend these relations to neural architectures as follows. Given two non-isomorphic graphs $G$ and $H$, a neural network architecture \emph{distinguishes} them if a parameter assignment exists such that $\hb_G \neq \hb_H$. 

\subsection{The Weisfeiler--Leman hierarchy and permutation-invariant function approximation}\label{connect}
The Weisfeiler--Leman hierarchy is a purely combinatorial algorithm for testing graph isomorphism. However,  the graph isomorphism function, mapping non-isomorphic graphs to different values, is the hardest to approximate permutation-invariant function. Hence, the Weisfeiler--Leman hierarchy has strong ties to GNNs' capabilities to approximate permutation-invariant or equivariant functions over graphs. For example,~\citet{Mor+2019,Xu+2018b} showed that the expressive power of any possible GNN architecture is limited by the \wlone{} in terms of distinguishing non-isomorphic graphs. \citet{Azi+2020} refined these results by showing that if an architecture is capable of simulating the \kwl{k} (on the set of $n$-order graphs) and allows the application of universal neural networks on vertex features, it will be able to approximate any permutation-equivariant function below the expressive power of the \kwl{k}; see also~\citep{Che+2019}. Hence, if one shows that one architecture distinguishes more graphs than another, it follows that the corresponding GNN can approximate more functions. These results were refined in \citep{geerts2022} for color refinement and taking into account the number of iterations of the \kwl{k}.

\subsection{A generalized adjacency matrix}
Finally, we define a generalization of the adjacency matrix from nodes to $k$-tuples. Specifically, let $G$ be a graph with $n$ nodes and let $\gamma \in \{-1, 1\}$. Then, for all $j \in [k]$, the \new{generalized adjacency matrix} $\vec{A}^{(k, j, \gamma)} \in \{0,1\}^{n^k \times n^k}$ is defined as
\begin{equation}\label{eq:generalized_adjacency_matrix}
    \vec{A}^{(k, j, \gamma)}_{il} \coloneqq \begin{cases}
    \begin{cases}
        1 & \exists w \in V(G) \colon \vec{u}_l = \phi_j(\vec{u}_i, w) \wedge \text{adj}(\vec{u}_{ij}, w) \\
        0 & \text{ else}
    \end{cases} & \gamma = 1 \\ \\
    \begin{cases}
        1 & \exists w \in V(G) \colon  \vec{u}_l = \phi_j(\vec{u}_i, w) \wedge \neg \text{adj}(\vec{u}_{ij}, w) \\
        0 & \text{ else},
    \end{cases} & \gamma = -1
    \end{cases}
\end{equation}
where $\vec{u}_i$ denotes the $i$-th $k$-tuple in a fixed but arbitrary ordering over $V(G)^k$, $\vec{u}_{ij}$ denotes the $j$-th node of $\vec{u}_{i}$ and where $\gamma$ controls whether the generalized adjacency matrix considers ($\gamma = 0$) all $j$-neighbors, ($\gamma > 1$) $j$-neighbors where the swapped nodes are adjacent in $G$ or ($\gamma < 0$) $j$-neighbors where the swapped nodes are not adjacent in $G$. With the generalized adjacency matrix as defined above we can represent the local and global $j$-neighborhood adjacency defined for the \deltakwl{} with $\vec{A}^{(k, j, 1)}$ and $\vec{A}^{(k, j, -1)}$, respectively. Further, the $j$-neighborhood adjacency of 
\kwl{k} can be described via $\vec{A}^{(k, j, 1)} + \vec{A}^{(k, j, -1)}$ and the local $j$-neighborhood adjacency of \localkwl{} can be described with $\vec{A}^{(k, j, 1)}$.

\section{Structural embeddings}\label{app:missing_proofs}
Before we give the missing proofs from \Cref{sec:implementation}, we develop a theoretical framework to analyze structural embeddings.

\subsection{Sufficiently node and adjacency-identifying}
We begin by proving the following lemma which shows an important bound for sufficiently node and adjacency-identifying matrices.
\begin{lemma}\label{lemma:adjacency_ident_approx}
Let $\vec{P}, \vec{Q} \in \mathbb{R}^{n \times d}$ be matrices such that
\begin{equation*}
    ||\vec{P} - \vec{Q}||_\text{F} < \epsilon,
\end{equation*}
for an arbitrary but fixed $\epsilon > 0$.
Let $\vec{P}$ be node or adjacency-identifying with projection matrices $\vec{W}^Q, \vec{W}^K \in \mathbb{R}^{d \times d}$ and let
\begin{align*}
    \Tilde{\vec{P}} &= \frac{1}{\sqrt{d_k}}(\vec{P} \vec{W}^Q)(\vec{P} \vec{W}^K)^T \text{, and} \\
    \Tilde{\vec{Q}} &= \frac{1}{\sqrt{d_k}}(\vec{Q} \vec{W}^Q)(\vec{Q} \vec{W}^K)^T.
\end{align*}
Then, there exists a monotonic strictly increasing function $f$ such that
\begin{equation*}
    ||\Tilde{\vec{P}} - \Tilde{\vec{Q}}||_\text{F} < f(\epsilon).
\end{equation*}
\end{lemma}
\begin{proof}
Our goal is to show that if the error between $\vec{P}$ and $\vec{Q}$ is bounded, so is the error between $\Tilde{\vec{P}}$ and $\Tilde{\vec{Q}}$. We now show that this error can be described by a monotonic strictly increasing function $f$, i.e., if $\epsilon_1 < \epsilon_2$, then $f(\epsilon_1) < f(\epsilon_2)$. We will first prove the existence of $f$.

First note that we can write,
\begin{equation*}
    \Tilde{\vec{P}} - \Tilde{\vec{Q}} = \frac{1}{\sqrt{d_k}} \cdot \Big( \vec{P} \vec{W}^Q (\vec{P} \vec{W}^K - \vec{Q} \vec{W}^K)^T + (\vec{P} \vec{W}^Q - \vec{Q} \vec{W}^Q)(\vec{Q} \vec{W}^K)^T \Big).
\end{equation*}
Note further that we are guaranteed that 
\begin{align}
    ||\vec{P}||_\text{F} &> 0 \label{eq:lemma_pq_approx_adj_ident_claim_1} \\
    ||\vec{W}^Q||_\text{F} &> 0 \label{eq:lemma_pq_approx_adj_ident_claim_2} \\
    ||\vec{W}^K||_\text{F} &> 0 \label{eq:lemma_pq_approx_adj_ident_claim_3}
\end{align}
since otherwise at least one of the above matrices is zero, in which case $\Tilde{\vec{P}} = \vec{0}$, which is not node or adjacency-identifying in general, a contradiction.
Now we can write,
\begin{align*}
   ||\Tilde{\vec{P}} - \Tilde{\vec{Q}}||_\text{F} &= \Big|\Big| \frac{1}{\sqrt{d_k}} \cdot \Big( \vec{P} \vec{W}^Q (\vec{P} \vec{W}^K - \vec{Q} \vec{W}^K)^T + (\vec{P} \vec{W}^Q - \vec{Q} \vec{W}^Q)(\vec{Q} \vec{W}^K)^T \Big) \Big|\Big|_\text{F} \\
   &\overset{\text{(a)}}{\leq} \frac{1}{\sqrt{d_k}} \cdot \Bigg( \Big|\Big| \vec{P} \vec{W}^Q (\vec{P} \vec{W}^K - \vec{Q} \vec{W}^K)^T \Big|\Big|_\text{F} + \Big|\Big| (\vec{P} \vec{W}^Q - \vec{Q} \vec{W}^Q)(\vec{Q} \vec{W}^K)^T \Big|\Big|_\text{F} \Bigg) \\
   &\overset{\text{(b)}}{\leq} \frac{1}{\sqrt{d_k}} \cdot \Bigg( \Big|\Big|\vec{P} \vec{W}^Q\Big|\Big|_\text{F} \cdot \Big|\Big|(\vec{P} \vec{W}^K - \vec{Q} \vec{W}^K)^T\Big|\Big|_\text{F} + \Big|\Big|(\vec{Q} \vec{W}^K)^T\Big|\Big|_\text{F} \cdot \Big|\Big|\vec{P} \vec{W}^Q - \vec{Q} \vec{W}^Q \Big|\Big|_\text{F} \Bigg) \\
   &\overset{\text{(c)}}{\leq} \frac{1}{\sqrt{d_k}} \cdot \Bigg( \Big|\Big|\vec{P} \vec{W}^Q\Big|\Big|_\text{F} \cdot \Big|\Big|\vec{P} \vec{W}^K - \vec{Q} \vec{W}^K\Big|\Big|_\text{F} + \Big|\Big|\vec{Q} \vec{W}^K\Big|\Big|_\text{F} \cdot \Big|\Big|\vec{P} \vec{W}^Q - \vec{Q} \vec{W}^Q \Big|\Big|_\text{F} \Bigg) \\
   &\overset{\text{(d)}}{\leq} \frac{1}{\sqrt{d_k}} \cdot \Bigg( \Big|\Big|\vec{P}\Big|\Big|_\text{F} \Big|\Big|\vec{W}^Q\Big|\Big|_\text{F} \cdot \Big|\Big|\vec{W}^K\Big|\Big|_\text{F} \Big|\Big|\vec{P} - \vec{Q}\Big|\Big|_\text{F} + \Big|\Big|\vec{Q}\Big|\Big|_\text{F} \Big|\Big|\vec{W}^Q\Big|\Big|_\text{F} \Big|\Big|\vec{W}^K\Big|\Big|_\text{F} \cdot \Big|\Big|\vec{P} - \vec{Q} \Big|\Big|_\text{F} \Bigg) \\
   &= \frac{1}{\sqrt{d_k}} \cdot \Big|\Big|\vec{W}^Q\Big|\Big|_\text{F} \Big|\Big|\vec{W}^K\Big|\Big|_\text{F} \cdot \Big|\Big|\vec{P} - \vec{Q} \Big|\Big|_\text{F} \cdot \Big( \Big|\Big|\vec{P}\Big|\Big|_\text{F} + \Big|\Big|\vec{Q}\Big|\Big|_\text{F} \Big)  \\
   &\overset{\text{(e)}}{<} \frac{\epsilon}{\sqrt{d_k}} \cdot \Big|\Big|\vec{W}^Q\Big|\Big|_\text{F} \Big|\Big|\vec{W}^K\Big|\Big|_\text{F} \cdot \Big( \Big|\Big|\vec{P}\Big|\Big|_\text{F} + \Big|\Big|\vec{Q}\Big|\Big|_\text{F} \Big).
\end{align*}
Here, we used (a) the triangle inequality, (b) the Cauchy-Schwarz inequality (c) the fact that for any matrix $\vec{X}$, $||\vec{X}||_\text{F} = ||\vec{X}^T||_\text{F}$, (d) again Cauchy-Schwarz and finally (e) the Lemma statement, namely that $||\vec{P} - \vec{Q}||_\text{F} < \epsilon$ combined with the facts in \Cref{eq:lemma_pq_approx_adj_ident_claim_1}, (\ref{eq:lemma_pq_approx_adj_ident_claim_2}) and (\ref{eq:lemma_pq_approx_adj_ident_claim_3}), guaranteeing that $||\Tilde{\vec{P}} - \Tilde{\vec{Q}}||_\text{F} > 0$. We set
\begin{equation*}
    f(\epsilon) \coloneqq \frac{\epsilon}{\sqrt{d_k}} \cdot \Big|\Big|\vec{W}^Q\Big|\Big|_\text{F} \Big|\Big|\vec{W}^K\Big|\Big|_\text{F} \cdot \Big( \Big|\Big|\vec{P}\Big|\Big|_\text{F} + \Big|\Big|\vec{Q}\Big|\Big|_\text{F} \Big),
\end{equation*}
which is clearly monotonic strictly increasing, since the norms as well as $\frac{1}{\sqrt{d_k}}$ are non-negative and \Cref{eq:lemma_pq_approx_adj_ident_claim_1}, (\ref{eq:lemma_pq_approx_adj_ident_claim_2}) and (\ref{eq:lemma_pq_approx_adj_ident_claim_3}) ensure that $f(\epsilon) > 0$.
As a result, we can now write
\begin{equation*}
    ||\Tilde{\vec{P}} - \Tilde{\vec{Q}}||_\text{F} < f(\epsilon).
\end{equation*}
This concludes the proof.
\end{proof}

We use sufficiently node or adjacency-identifying matrices for approximately recovering \new{weighted indicator matrices}, which we define next.
\begin{definition}[Weighted indicators]
Let $\vec{x} = (x_1, \dots, x_n) \in \{0,1\}^n$ be an $n$-dimensional binary vector. We call
\begin{equation*}
    \Tilde{\vec{x}} \coloneqq \frac{\vec{x}}{\sum_{i=1}^n x_i},
\end{equation*}
the weighted indicator vector of $\vec{x}$. Further, let $\vec{X} \in \{0,1\}^{n \times n}$ be a binary matrix. Let now $\Tilde{\vec{X}} \in \mathbb{R}^{n \times n}$ be a matrix such that the $i$-th row of $\Tilde{\vec{X}}$ is the weighted indicator vector of the $i$-th row of $\vec{X}$. We call $\Tilde{\vec{X}}$ the weighted indicator matrix of $\vec{X}$.
\end{definition}

The following lemma ties (sufficiently) node or adjacency-identifying matrices and weighted indicators together.
\begin{lemma}\label{lemma:sufficiently_indicator}
Let $\Tilde{\vec{P}} \in \mathbb{R}^{n \times n}$ be a matrix and let $\vec{X}$ be a binary matrix with weighted indicator matrix $\Tilde{\vec{X}}$, such that for every $i, j \in [n]$,
\begin{equation*}
    \Tilde{\vec{P}}_{ij} = \max_{k} \Tilde{\vec{P}}_{ik}
\end{equation*}
if, and only, if $\vec{X}_{ij} = 1$. Then, for a matrix $\Tilde{\vec{Q}} \in \mathbb{R}^{n \times n}$ and all $\epsilon > 0$, there exist an $\delta > 0$ and a $b > 0$ such that if
\begin{equation*}
    \big|\big| \Tilde{\vec{P}} - \Tilde{\vec{Q}} \big|\big|_\text{F} < \delta,
\end{equation*}
then,
\begin{equation*}
    \big|\big| \textsf{softmax}\Big( b \cdot \Tilde{\vec{Q}} \Big) - \Tilde{\vec{X}} \big|\big|_\text{F} < \epsilon,
\end{equation*}
where $\text{softmax}$ is applied row-wise.
\end{lemma}
\begin{proof}
We begin by reviewing how the softmax acts on a vector $\vec{z} = (z_1, \dots, z_n) \in \mathbb{R}^n$. 
Let $z_\text{max} \coloneqq \max_i z_i$ be the maximum value in $\vec{z}$. Further, let $\vec{x} = (x_1, \dots, x_n) \in \{0,1\}^n$ be a binary vector such that
\begin{equation*}
    x_i = \begin{cases}
        1 & z_i = z_\text{max} \\
        0 & \text{ else}
    \end{cases}.
\end{equation*}
Now, for a scalar $b > 0$,
\begin{equation*}
    \lim_{b \rightarrow \inf} \textsf{softmax}\Big(b \cdot \vec{z} \Big) = \dfrac{x}{\sum_{i=1}^n x_i},
\end{equation*}
i.e., as $b$ goes to infinity, the softmax converges $b \cdot \vec{z}$ to the weighted indicator vector 
\begin{equation*}
    \Tilde{\vec{x}} = \dfrac{x}{\sum_{i=1}^n x_i}.
\end{equation*}
Let us now generalize this to matrices. Specifically, let $\Tilde{\vec{P}} \in \mathbb{R}^{n \times n}$ be a matrix, let $\vec{X}$ be a binary matrix with weighted indicator matrix $\Tilde{\vec{X}}$ such that for every $i, j \in [n]$,
\begin{equation*}
    \Tilde{\vec{P}}_{ij} = \max_{k} \Tilde{\vec{P}}_{ik}
\end{equation*}
if and only if $\vec{X}_{ij} = 1$. Then, for a scalar $b > 0$,
\begin{equation*}
    \lim_{b \rightarrow \inf} \textsf{softmax}\Big(b \cdot \Tilde{\vec{P}} \Big) = \Tilde{\vec{X}},
\end{equation*}
which follows from the fact that the softmax is applied independently to each row and each row of $\Tilde{\vec{X}}$ is a weighted indicator vector. We now show the proof statement.
First, note that for any $b > 0$ we can choose a $\delta < f(b)$, where $f: \mathbb{R} \rightarrow \mathbb{R}$ is some strictly monotonically \textit{decreasing} function of $b$ that shrinks faster than linear, e.g., $f(b) = \frac{1}{b^2}$. This function is well-defined since by assumption $b > 0$.
As a result an increase in $b$ implies a non-linearly growing decrease of
\begin{equation*}
    \big|\big| \textsf{softmax}\Big( b \cdot \Tilde{\vec{P}} \Big) - \textsf{softmax}\Big( b \cdot \Tilde{\vec{Q}} \Big) \big|\big|_\text{F}
\end{equation*}
and thus,
\begin{equation*}
    \lim_{b \rightarrow \inf} \textsf{softmax}\Big(b \cdot \Tilde{\vec{Q}} \Big) = \Tilde{\vec{X}},
\end{equation*}
and we have that for all $\epsilon > 0$ there exists a $b > 0$ and a $\delta < f(b)$ such that
\begin{equation*}
    \big|\big| \textsf{softmax}\Big( b \cdot \Tilde{\vec{Q}} \Big) - \Tilde{\vec{X}} \big|\big|_\text{F} < \epsilon.
\end{equation*}
This concludes the proof.
\end{proof}

In the above proof, the approximation with $\delta$ and the scaling with $b$ act as two "opposing forces". The proof then chooses $b$ such that $\epsilon$ acts stronger than $b$ and hence with $b \rightarrow \inf$, the approximation converges to the weighted indicator matrix; see \Cref{fig:sufficiently_indicator_explainer} for a visual explanation of this concept.

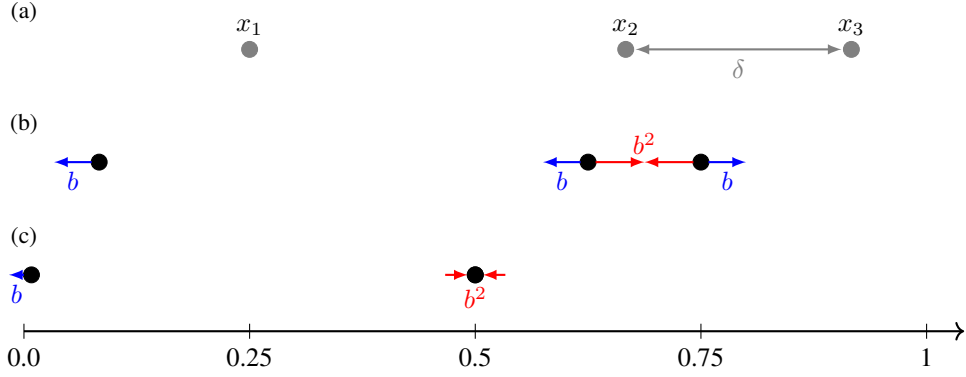
\begin{figure}
    \centering
    \begin{tikzpicture}

    \filldraw[gray] (0,0) circle (3pt) node[midway] (x1a) {};
    \filldraw[gray] (5,0) circle (3pt);
    \filldraw[gray] (8,0) circle (3pt);
    \draw[<->, thick, >=latex, shorten <=-2pt, shorten >=-2pt, gray] (5.2, 0) -- node[below] {$\delta$} (7.8, 0);

    \node at (0,0.3) {$x_1$};
    \node at (5,0.3) {$x_2$};
    \node at (8,0.3) {$x_3$};

    \node at (-3,0.5) {\small(a)};
    \node at (-3,-1) {\small(b)};
    \node at (-3,-2.5) {\small(c)};

    \filldraw[black] (4.5, -1.5) circle (3pt);
    \draw[->, thick, >=latex, blue] (4.4, -1.5) -- node[below] {$b$} (4.4-0.5, -1.5);
    \draw[->, thick, >=latex, red] (4.6, -1.5) -- (4.6+0.65, -1.5);

    \filldraw[black] (6, -1.5) circle (3pt);
    \draw[->, thick, >=latex, red] (5.9, -1.5) -- (5.9-0.65, -1.5);
    \draw[->, thick, >=latex, blue] (6.1, -1.5) -- node[below] {$b$} (6.1+0.5, -1.5);
    \node[red] at (5.25,-1.25) {$b^2$};
    
    \filldraw[black] (-2, -1.5) circle (3pt) node[midway] (x1b) {};
    \draw[->, thick, >=latex, blue] (-2.1, -1.5) -- node[below] {$b$} (-2.1-0.5, -1.5);

    \filldraw[black] (3, -3) circle (3pt);
    \draw[->, thick, >=latex, red] (3-0.4, -3) -- (3.3-0.4, -3);
    \draw[<-, thick, >=latex, red] (3.1, -3) -- (3.4, -3);

    \filldraw[black] (3, -3) circle (3pt);
    \node[red] at (3,-3.3) {$b^2$};
    
    \filldraw[black] (-2.9, -3) circle (3pt);
    \draw[->, thick, >=latex, blue] (-3, -3) -- node[below] {$b$} (-3-0.2, -3);

    \draw[->, thick] (-3,-3.75) -- (9.5,-3.75);
    \foreach \x in {0.0,0.25,0.5,0.75,1} {
        \draw (\x*12 - 3, -3.75 + 0.1) -- (\x*12 - 3, -3.75 - 0.1) node[below] {\x};
    }

    \end{tikzpicture}
    \caption{Visual explanation of the "opposing forces" in \Cref{lemma:sufficiently_indicator}. In (a) \textbf{before softmax}: We consider three numbers $x_1$, $x_2$ and $x_3$, where $x_2$ and $x_3$ are less than $\delta$ apart. In (b) \textbf{after softmax}: An increase in $b$ (blue) pushes the maximum value $x_3$ away from $x_1$ and $x_2$. However, the approximation with $\delta$ acts stronger (red). As a result, $x_1$ gets pushed closer to $0$, but $x_2$ and $x_3$ get pushed closer. In (c) \textbf{after softmax}: Further increasing $b$ makes $x_1$ converge to $0$, but the approximation with $\delta$ pushes $x_2$ and $x_3$ closer together, and the softmax maps both values approximately to $\frac{1}{2}$ (here depicted with the same dot). Hence, with a sufficiently close approximation, we can approximate the weighted indicator matrix $\Tilde{\vec{X}}$. }
    \label{fig:sufficiently_indicator_explainer}
\end{figure}

From the definitions of sufficiently node and adjacency-identifying matrices, we can prove the following statement.
\begin{lemma}\label{lemma:approx_normalized_adjacency}
Let $G$ be a graph with adjacency matrix $\vec{A}(G)$ and degree matrix $\vec{D}$. Let $\vec{Q}$ be a sufficiently adjacency-identifying matrix. Then, there exists a $b > 0$ and projection matrices $\vec{W}^Q, \vec{W}^K \in \mathbb{R}^{d \times d}$ with
\begin{equation*}
    \Tilde{\vec{Q}} \coloneqq \frac{1}{\sqrt{d_k}}(\vec{Q} \vec{W}^Q)(\vec{Q} \vec{W}^K)^T
\end{equation*}
such that
\begin{equation*}
    \big|\big| \textsf{softmax}\Big( b \cdot \Tilde{\vec{Q}} \Big) - \vec{D}^{-1}\vec{A}(G) \big|\big|_\text{F} < \epsilon,
\end{equation*}
for all $\epsilon > 0$.
\end{lemma}
\begin{proof}
Note that $\vec{D}^{-1}\vec{A}(G)$ is a weighted indicator matrix, since $\vec{A}(G)$ has binary entries and left-multiplication with $\vec{D}^{-1}$ results in dividing every element in row $i$ of $\vec{A}(G)$ by the number of $1$'s in row $i$ of $\vec{A}(G)$, or formally,
\begin{equation*}
    \big(\vec{D}^{-1}\vec{A}(G)\big)_i = \begin{bmatrix}
       \frac{\vec{A}(G)_{i1}}{\sum_{j=1}^n \vec{A}(G)_{ij}} & \dots & \frac{\vec{A}(G)_{in}}{\sum_{j=1}^n \vec{A}(G)_{ij}}
    \end{bmatrix},
\end{equation*}
where we note that
\begin{equation*}
    \vec{D}_{ii} = \sum_{j=1}^n \vec{A}(G)_{ij}.
\end{equation*}
Further, since $\vec{Q}$ is sufficiently adjacency-identifying, there exists an adjacency-identifying matrix $\vec{P}$ and projection matrices $\vec{W}^Q$ and $\vec{W}^K$, such that for
\begin{equation*}
   \Tilde{\vec{P}} = \frac{1}{\sqrt{d_k}}(\vec{P} \vec{W}^Q)(\vec{P} \vec{W}^K)^T,
\end{equation*}
it holds that
\begin{equation*}
   \Tilde{\vec{P}}_{ij} = \max_k  \Tilde{\vec{P}}_{ik},
\end{equation*}
if and only if $\vec{A}(G) = 1$. Recall the definition of sufficiently adjacency-identifying, namely that we can choose projection matrices such that
\begin{equation*}
    ||\vec{P} - \vec{Q}||_\text{F} < \epsilon_0,
\end{equation*}
for all $\epsilon_0$.
Then, according to \Cref{lemma:adjacency_ident_approx}, for matrix
\begin{equation*}
   \Tilde{\vec{Q}} = \frac{1}{\sqrt{d_k}}(\vec{Q} \vec{W}^Q)(\vec{Q} \vec{W}^K)^T,
\end{equation*}
it holds that
\begin{equation*}
    \big|\big| \Tilde{\vec{P}} - \Tilde{\vec{Q}} \big|\big|_\text{F} < f(\epsilon_0),
\end{equation*}
where $f$ is a strictly monotonically increasing function of $\epsilon_0$.
We can then apply \Cref{lemma:sufficiently_indicator} to $\Tilde{\vec{P}}$, $\Tilde{\vec{Q}}$, $\vec{A}(G)$ as the binary matrix and $\vec{D}^{-1}\vec{A}(G)$ as its weighted indicator matrix and, for every $\epsilon$, choose $\epsilon_0$ small enough such that there exists a $b > 0$ with
\begin{equation*}
    \big|\big| \textsf{softmax}\Big( b \cdot \Tilde{\vec{Q}} \Big) - \vec{D}^{-1}\vec{A}(G) \big|\big|_\text{F} < \epsilon.
\end{equation*}
This concludes the proof.
\end{proof}

Finally, we generalize the above result to arbitrary $k$ via the generalized adjacency matrix.
\begin{lemma}[Approximating the generalized adjacency matrix]\label{lemma:approx_generalized_adjacency}
Let $G$ be graph with $n$ nodes and let $k > 1$. Let further, $\vec{Q} \in \mathbb{R}^{n \times \sfrac{d}{k}}$ be sufficiently node and adjacency-identifying structural embeddings. Let $\vec{v} \coloneqq (v_1, \dots, v_k) \in V(G)^k$ be the $i$-th and let $\vec{u} \coloneqq (u_1, \dots, u_k) \in V(G)^k$ be the $l$-th $k$-tuple in a fixed but arbitrary ordering over $V(G)^k$. 
Let $\Tilde{\vec{A}}^{(k,j,\gamma)}$ denote the weighted indicator matrix of the generalized adjacency matrix in \Cref{eq:generalized_adjacency_matrix}.
Let $\vec{Z} \in \mathbb{R}^{n^k \times n^k}$ be the unnormalized attention matrix such that
\begin{equation*}
   \vec{Z}_{il} \coloneqq \frac{1}{\sqrt{d_k}} \big[ \vec{Q}(v_j) \big]_{j=1}^k \vec{W}^{Q}\big(\big[ \vec{Q}(u_j) \big]_{j=1}^k\vec{W}^{K}\big)^T, 
\end{equation*}
with projection matrices $\vec{W}^Q, \vec{W}^K \in \mathbb{R}^{d \times d}$.
Then, for every $j \in [k]$ and every $\gamma \in \{-1, 1\}$, there exist projection matrices $\vec{W}^Q, \vec{W}^K$ such that for all $i \in [n^k]$
\begin{equation*}
    \Big\lVert \mathsf{softmax} \Big( \begin{bmatrix}
    \vec{Z}_{i1} & \dots & \vec{Z}_{in^k}
\end{bmatrix} \Big) -  \Tilde{\vec{A}}^{(k,j,\gamma)}_{i} \Big\rVert_F < \varepsilon,
\end{equation*}
for any $\varepsilon > 0$.
\end{lemma}
\begin{proof}
Our proof strategy is to first construct the unnormalized attention matrix from a node- and adjacency-identifying matrix and then to invoke \Cref{lemma:sufficiently_indicator} to relax the approximation to the \textit{sufficiently} node- and adjacency-identifying $\vec{Q}$. 

First, by definition of sufficiently node and adjacency-identifying matrices, the existence of $\vec{Q}$ implies the existence of a node and adjacency-identifying matrix $\vec{P}$. We will now give projection matrices $\vec{W}^Q$ and $\vec{W}^K$ such that
\begin{equation}\label{eq:approx_generalized_adjacency_exactly_identifying_condition}
    \Big\lVert \mathsf{softmax} \Big( \begin{bmatrix}
    \vec{Z}^*_{i1} & \dots & \vec{Z}^*_{in^k}
\end{bmatrix} \Big) -  \Tilde{\vec{A}}^{(k,j,\gamma)}_{i} \Big\rVert_F < \varepsilon.
\end{equation}
for any $\varepsilon > 0$,
with 
\begin{equation*}
   \vec{Z}^*_{il} \coloneqq \frac{1}{\sqrt{d_k}} \big[ \vec{P}(v_j) \big]_{j=1}^k \vec{W}^{Q}\big(\big[ \vec{P}(u_j) \big]_{j=1}^k\vec{W}^{K}\big)^T, 
\end{equation*}
where we recall that $\vec{v} = (v_1, \dots, v_k) \in V(G)^k$ is the $i$-th and $\vec{u} = (u_1, \dots, u_k) \in V(G)^k$ is the $l$-th $k$-tuple in a fixed but arbitrary ordering over $V(G)^k$.
To show \Cref{eq:approx_generalized_adjacency_exactly_identifying_condition}, 
we expand sub-matrices in $\vec{W}^Q$ and $\vec{W}^K$ for each $j$, i.e., we write
\begin{equation*}
   \vec{W}^Q = \begin{bmatrix}
       \vec{W}^{Q,1} \\
       \vdots \\
       \vec{W}^{Q,k} \\
   \end{bmatrix} 
\end{equation*}
and
\begin{equation*}
   \vec{W}^K = \begin{bmatrix}
       \vec{W}^{K,1} \\
       \vdots \\
       \vec{W}^{K,k} \\
   \end{bmatrix},
\end{equation*}
where for all $j \in [k]$, $\vec{Q}(v_j)$ is projected by $\vec{W}^{Q,j}$ and $\vec{Q}(u_j)$ is projected by $\vec{W}^{K,j}$.
We further expand the sub-matrices of each $\vec{W}^{Q,j}$ and $\vec{W}^{K,j}$, writing
\begin{equation*}
   \vec{W}^{Q,j} = \begin{bmatrix}
       \vec{W}^{Q,j}_N \\
       \vec{W}^{Q,j}_A
   \end{bmatrix} 
\end{equation*}
and
\begin{equation*}
   \vec{W}^{K,j} = \begin{bmatrix}
       \vec{W}^{K,j}_N \\
       \vec{W}^{K,j}_A
   \end{bmatrix}.
\end{equation*}

Since $\vec{P}$ is node-identifying, there exists projection matrices $\vec{W}^{Q,*}_N$ and $\vec{W}^{K,*}_N$ such that
\begin{equation*}
    p_{ilj} \coloneqq \frac{1}{\sqrt{d_k}} \vec{P}^*\vec{W}^{Q,*}_N(\vec{P}^*\vec{W}^{K,*}_N)^T
\end{equation*}
is maximal if and only if $\vec{u}_{i,j} = \vec{u}_{l,j}$ is the same node.
Further, if $\vec{P}$ is adjacency-identifying, there exists projection matrices $\vec{W}^{Q,*}_A$ and $\vec{W}^{K,*}_A$ such that
\begin{equation*}
    q_{ilj} \coloneqq \frac{1}{\sqrt{d_k}} \vec{P}^*\vec{W}^{Q,*}_A(\vec{P}^*\vec{W}^{K,*}_A)^T,
\end{equation*}
is maximal if and only if nodes $\vec{u}_{i,j}$ and $\vec{u}_{l,j}$ share an edge in $G$. Consequently, we also know that then $-q_{ilj}$ is maximal if and only if nodes $\vec{u}_{i,j}$ and $\vec{u}_{l,j}$ do not share an edge in $G$. Note that because we assume that $G$ has no self-loops, $q_{iij} = 0$ for all $i \in [n]$ and all $j \in [k]$.
Now, let us write out the unnormalized attention score $\vec{Z}^*_{il}$ as
\begin{align*}
  \vec{Z}^*_{il} &= \frac{1}{\sqrt{d_k}} \big[ \vec{P}(v_j) \big]_{j=1}^k \vec{W}^{Q}\big(\big[ \vec{P}(u_j) \big]_{j=1}^k\vec{W}^{K}\big)^T \\
  &= \frac{1}{\sqrt{d_k}} \sum_{j=1}^k \vec{P}(v_j) \vec{W}^{Q,j} \big(\vec{P}(u_j)\vec{W}^{K,j} \big)^T \\
  &= \frac{1}{\sqrt{d_k}} \sum_{j=1}^k \vec{P}(v_j) \vec{W}^{Q,j}_N \big(\vec{P}(u_j)\vec{W}^{K,j}_N \big)^T + \frac{1}{\sqrt{d_k}}\sum_{j=1}^k \vec{P}(v_j) \vec{W}^{Q,j}_A \big(\vec{P}(u_j)\vec{W}^{K,j}_A \big)^T,
\end{align*}
where in the second line we simply write out the dot-product and in the last line we split the sum to distinguish node- and adjacency-identifying terms.

Now for all $\gamma$ and a fixed $j$, we set $\vec{W}^{Q,j}_N = \vec{0}$ and $\vec{W}^{K,j}_N = \vec{0}$ and $\vec{W}^{Q,o}_N = b \cdot \vec{W}^{Q, *}_N$ and $\vec{W}^{K,o}_N = \vec{W}^{K, *}_N$ for $o \neq j$, for some $b > 0$ that we need to be arbitary to use it later in \Cref{lemma:sufficiently_indicator}.

We now distinguish between the different choices of $\gamma$. In the first case, let $\gamma > 0$.
Then, we set $\vec{W}^{Q,j}_A = \vec{W}^{Q, *}_A$ and $\vec{W}^{K,j}_A = \vec{W}^{K, *}_A$.  
In the third case, $\gamma < 0$. Then, we set $\vec{W}^{Q,j}_A = -\vec{W}^{Q, *}_A$ and $\vec{W}^{K,j}_A = \vec{W}^{K, *}_A$.

Then, in all cases, for tuples $\vec{u}_i$ and $\vec{u}_l$,
\begin{equation}\label{eq:kwl_proof_emb_j_nb}
   \vec{Z}^*_{il} = \gamma \cdot q_{ilj} + \sum_{o \neq j} b \cdot p_{ilo}.
\end{equation}
Note that $b$ is the same for all pairs $\vec{u}_i, \vec{u}_l$ and the dot-product is positive.
We again distinguish between two cases. 
In the first case, let $\gamma > 0$. Then, the above sum attains its maximum value if and only if 
\begin{equation*}
    \forall o \neq j \colon \vec{u}_{i,o} = \vec{u}_{l,o} \wedge A_{ilj} = 1.
\end{equation*}
Now, since $\vec{u}_{i,o}$ and $\vec{u}_{l,o}$ denote nodes at the $o$-th component of $\vec{u}_i$ and $\vec{u}_l$, respectively, the above statement is equivalent to saying that $\vec{u}_l$ is a $j$-neighbor of $\vec{u}_i$ and that $\vec{u}_l$ is adjacent to $\vec{u}_i$. 
In the second case, let $\gamma < 0$. Then, the above sum attains its maximum value if and only if 
\begin{equation*}
    \forall o \neq j \colon \vec{u}_{i,o} = \vec{u}_{l,o} \wedge A_{ilj} = 0.
\end{equation*}
Again, since $\vec{u}_{i,o}$ and $\vec{u}_{l,o}$ denote nodes at the $o$-th component of $\vec{u}_i$ and $\vec{u}_l$, respectively, the above statement is equivalent to saying that $\vec{u}_l$ is a $j$-neighbor of $\vec{u}_i$ and that $\vec{u}_l$ is \textit{not} adjacent to $\vec{u}_i$. 

Now, recall the generalized adjacency matrix in \Cref{eq:generalized_adjacency_matrix} as
\begin{equation*}
    \vec{A}^{(k, j, \gamma)}_{il} \coloneqq \begin{cases}
    \begin{cases}
        1 & \exists w \in V(G) \colon \vec{u}_l = \phi_j(\vec{u}_i, w) \wedge \text{adj}(\vec{u}_{ij}, w) \\
        0 & \text{ else}
    \end{cases} & \gamma = 1 \\ \\
    \begin{cases}
        1 & \exists w \in V(G) \colon  \vec{u}_l = \phi_j(\vec{u}_i, w) \wedge \neg \text{adj}(\vec{u}_{ij}, w) \\
        0 & \text{ else}.
    \end{cases} & \gamma = -1
    \end{cases}
\end{equation*}
Then, we can say that for our construction of $\vec{Z}^*$, for all $i, l \in [n^k]$
\begin{equation*}
    \vec{Z}^*_{il} = \max_{k} \vec{Z}^*_{ik}
\end{equation*}
if and only if $\vec{A}^{k,j,\gamma}_{il} = 1$.
Consequently, we can apply \Cref{lemma:sufficiently_indicator} to $\vec{Z}^*$, $\vec{Z}$, $\vec{A}^{k,j,\gamma}$ as the binary matrix and $\Tilde{\vec{A}}^{k,j,\gamma}$ as its weighted indicator matrix and obtain that for each $\epsilon > 0$ there exists a $b > 0$ such that
\begin{equation*}
    \Big\lVert \mathsf{softmax} \Big( \begin{bmatrix}
    \vec{Z}_{i1} & \dots & \vec{Z}_{in^k}
\end{bmatrix} \Big) -  \Tilde{\vec{A}}^{k,j,\gamma}_i \Big\rVert_F < \varepsilon.
\end{equation*}
This completes the proof.
\end{proof}

\subsection{adjacency-identifying via graph Laplacian}
Here, we show a few results for how factorizations of the (normalized) graph Laplacian can be adjacency-identifying.

\begin{lemma}\label{lemma:laplacian_factorization_node_and_adj}
Let $G$ be a graph with graph Laplacian $\vec{L}$. Then, a matrix $\vec{P}$ is node- and adjacency-identifying if there exists matrices $\vec{W}^Q$ and $\vec{W}^K$ such that
\begin{equation*}
   \frac{1}{\sqrt{d_k}}(\vec{P} \vec{W}^Q)(\vec{P} \vec{W}^K)^T = \vec{L}.
\end{equation*}
\end{lemma}
\begin{proof}
Note that the graph Laplacian is node-identifying, since all off-diagonal elements are $\leq 0$ and the diagonal is always $> 0$ since we consider graphs $G$ without self-loops and without isolated nodes.
Note further, that if there exists matrices $\vec{W}^Q$ and $\vec{W}^K$ such that
\begin{equation*}
   \frac{1}{\sqrt{d_k}}(\vec{P} \vec{W}^Q)(\vec{P} \vec{W}^K)^T = \vec{L},
\end{equation*}
then there also exist matrix $\vec{W}_*^Q = -\vec{W}^Q$ such that
\begin{equation*}
   \frac{1}{\sqrt{d_k}}(\vec{P} \vec{W}_*^Q)(\vec{P} \vec{W}^K)^T = -\vec{L}.
\end{equation*}
Now, note that the negative graph Laplacian is $-\vec{L} = \vec{A}(G) - \vec{D}$. Because we subtract the degree matrix from the adjacency matrix, the maximum element of each row of the negative graph Laplacian is $1$.
Since $\vec{D}$ is diagonal, for each row $i$ and each column $j$,
\begin{equation*}
    -\vec{L}_{ij} = 1 \Longleftrightarrow \vec{A}(G)_{ij} = 1,
\end{equation*}
for $i \neq j$. Further, in the case where $i = j$,
\begin{equation*}
    -\vec{L}_{ij}  \leq 0,
\end{equation*}
since we consider graphs $G$ without self-loops. Hence, we obtain that
\begin{equation*}
    -\vec{L}_{ij}  = \max_k (-\vec{L}_{ik}) 
\end{equation*}
if and only if $\vec{A}(G)_{ij} = 1$. This shows the statement.
\end{proof}

\begin{lemma}\label{lemma:laplacian_factorization}
Let $G$ be a graph with graph Laplacian $\vec{L}$. Then, a matrix $\vec{P}$ is node and adjacency-identifying if
\begin{equation*}
    \vec{L} = \vec{P}\vec{P}^T.
\end{equation*}
\end{lemma}
\begin{proof}
Our goal is to show that there exist matrices $\vec{W}^Q$ and $\vec{W}^K$, such that
\begin{equation*}
   \frac{1}{\sqrt{d_k}}(\vec{P} \vec{W}^Q)(\vec{P} \vec{W}^K)^T = \vec{L},
\end{equation*}
and then to invoke \Cref{lemma:laplacian_factorization_node_and_adj}.
We set
\begin{align*}
    \vec{W}^Q &= \sqrt{d_k}\vec{I} \\
    \vec{W}^K &= \vec{I},
\end{align*}
where $\vec{I}$ is the identity matrix. Then,
\begin{equation*}
    \frac{1}{\sqrt{d_k}}(\vec{P} \vec{W}^Q)(\vec{P} \vec{W}^K)^T = \vec{P}\vec{P}^T = \vec{L}.
\end{equation*}
This shows the statement.
\end{proof}

\begin{lemma}\label{lemma:laplacian_factorization_pq}
Let $G$ be a graph with graph Laplacian $\vec{L}$. Then, a matrix $\begin{bmatrix}
    \vec{P} & \vec{Q}
\end{bmatrix}$ is node- and adjacency-identifying if
\begin{equation*}
    \vec{L} = \vec{P}\vec{Q}^T.
\end{equation*} 
\end{lemma}
\begin{proof}
Our goal is to show that there exist matrices $\vec{W}^Q$ and $\vec{W}^K$, such that
\begin{equation*}
   \frac{1}{\sqrt{d_k}}(\vec{P} \vec{W}^Q)(\vec{P} \vec{W}^K)^T = \vec{L},
\end{equation*}
and then to invoke \Cref{lemma:laplacian_factorization_node_and_adj}.
We set
\begin{align*}
    \vec{W}^Q &= \begin{bmatrix}
        \sqrt{d_k}\vec{I} & \vec{0}
    \end{bmatrix} \\
    \vec{W}^K &= \begin{bmatrix}
        \vec{0} & \sqrt{d_k}\vec{I}
    \end{bmatrix},
\end{align*}
where $\vec{I}$ is the identity matrix and $\vec{0}$ is the all-zero matrix. Then,
\begin{equation*}
    \frac{1}{\sqrt{d_k}} \big(\begin{bmatrix}
        \vec{P} & \vec{Q}
\end{bmatrix}\vec{W}^Q\big)\big(\begin{bmatrix}
        \vec{P} & \vec{Q}
    \end{bmatrix}\vec{W}^K\big)^T
    = 
    \mathbf{P}\mathbf{Q}^T = \vec{L}.
\end{equation*}
This shows the statement.
\end{proof}

For the next two lemmas, we first briefly define permutation matrices as binary matrices $\vec{M}$ that are right-stochastic, i.e., its rows sum to $1$, and such that each column of $\vec{M}$ is $1$ at exactly one position and $0$ elsewhere. It is well know that for permutation matrices $\vec{M}$ it holds that $\vec{M}^T\vec{M} = \vec{M}\vec{M}^T = \vec{I}$,
where $\vec{I}$ is the identity matrix. We now state the following lemmas.

\begin{lemma}\label{lemma:UE_adjacency_identifying}
Let $G$ be a graph with graph Laplacian $\vec{L}$. Let $\vec{L} = \vec{U}\vec{\Sigma}\vec{U}^T$ be the eigendecomposition of $\vec{L}$. Then, for any permutation matrix $\vec{M}$, the matrix $\vec{U}\vec{\Sigma}^{\frac{1}{2}}\vec{M}$ is adjacency-identifying.
\end{lemma}
\begin{proof}
We have that
\begin{align*}
    \vec{U}\vec{\Sigma}^{\frac{1}{2}}\vec{M}(\vec{U}\vec{\Sigma}^{\frac{1}{2}}\vec{M})^T &= \vec{U}\vec{\Sigma}^{\frac{1}{2}}\vec{M}\vec{M}^T\vec{\Sigma}^{\frac{1}{2}}\vec{U}^T\\
    &=  \vec{U}\vec{\Sigma}\vec{U}^T \\
    &= \vec{L}.
\end{align*}
Hence, by \Cref{lemma:laplacian_factorization}, $\vec{U}\vec{\Sigma}^{\frac{1}{2}}\vec{M}$ is adjacency-identifying.
\end{proof}

\begin{lemma}\label{lemma:DUE_adjacency_identifying}
Let $G$ be a graph with graph Laplacian $\vec{L}$ and normalized graph Laplacian $\Tilde{\vec{L}}$. Let $\vec{L} = \vec{U}\vec{\Sigma}\vec{U}^T$ be the eigendecomposition of $\vec{L}$ and let $\vec{D}$ denote the degree matrix. Then, for any permutation matrix $\vec{M}$ the matrix $\vec{D}^{\frac{1}{2}}\vec{U}\vec{\Sigma}^{\frac{1}{2}}\vec{M}$ is adjacency-identifying.
\end{lemma}
\begin{proof}
We have that
\begin{align*}
    \vec{D}^{\frac{1}{2}}\vec{U}\vec{\Sigma}^{\frac{1}{2}}\vec{M}(\vec{D}^{\frac{1}{2}}\vec{U}\vec{\Sigma}^{\frac{1}{2}}\vec{M})^T &= 
    \vec{D}^{\frac{1}{2}} \vec{U}\vec{\Sigma}^{\frac{1}{2}}\vec{M}\vec{M}^T\vec{\Sigma}^{\frac{1}{2}}\vec{U}^T\vec{D}^{\frac{1}{2}} \\    
    &= \vec{D}^{\frac{1}{2}} \vec{U}\vec{\Sigma}\vec{U}^T\vec{D}^{\frac{1}{2}} \\
    &= \vec{D}^{\frac{1}{2}}\Tilde{\vec{L}}\vec{D}^{\frac{1}{2}} \\
    &= \vec{D}^{\frac{1}{2}}(\vec{I} - \vec{D}^{-\frac{1}{2}} \vec{A} \vec{D}^{-\frac{1}{2}})\vec{D}^{\frac{1}{2}} \\
    & = \vec{D} - \vec{A} = \vec{L}.
\end{align*}
Hence, by \Cref{lemma:laplacian_factorization}, $\vec{D}^{\frac{1}{2}}\vec{U}\vec{\Sigma}^{\frac{1}{2}}$ is node or adjacency-identifying.
\end{proof}

\subsection{LPE}\label{app:proof_lpe}

Here, we show that the node-level PEs $\embed{LPE}$ as defined in \Cref{eq:1wl_proof_lap_def} are sufficiently node- and adjacency-identifying. We begin with the following useful lemma.

\begin{lemma}\label{lemma:node_and_adj_subspace}
Let $\vec{P} \in \mathbb{R}^{n \times d}$ be structural embeddings with sub-matrices $\vec{Q}_1 \in \mathbb{R}^{n \times d'}$ and $\vec{Q}_2 \in \mathbb{R}^{n \times d''}$, i.e.,
\begin{equation*}
    \vec{P} = \begin{bmatrix}
        \vec{Q}_1 & \vec{Q}_2
    \end{bmatrix},
\end{equation*}
where $d = d' + d''$.
If one of $\vec{Q}_1, \vec{Q}_2$ is (sufficiently) node-identifying, then $\vec{P}$ is (sufficiently) node-identifying. If one of $\vec{Q}_1, \vec{Q}_2$ is (sufficiently) adjacency-identifying, then $\vec{P}$ is (sufficiently) adjacency-identifying. If $\vec{Q}_1$ is (sufficiently) node-identifying and $\vec{Q}_2$ is (sufficiently) adjacency-identifying or vice versa, then $\vec{P}$ is (sufficiently) node \textbf{and} adjacency-identifying.
\end{lemma}
\begin{proof}
Let a sub-matrix $\vec{Q}$ be (sufficiently) node- or adjacency-identifying. Let $\vec{W}^{Q,Q}$ and $\vec{W}^{K,Q}$ be the corresponding projection matrices with which $\vec{Q}$ is (sufficiently) node- or adjacency-identifying. Then, if $\vec{Q} = \vec{Q}_1$, we define
\begin{align*}
    \vec{W}^Q &= \begin{bmatrix}
        \vec{W}^{Q,Q} & \vec{0} \\ 
        \vec{0} & \vec{0}
    \end{bmatrix} \\
    \vec{W}^K &= \begin{bmatrix}
        \vec{W}^{Q,K} & \vec{0} \\ 
        \vec{0} & \vec{0}
    \end{bmatrix} \\
\end{align*}
and if $\vec{Q} = \vec{Q}_2$, we define
\begin{align*}
    \vec{W}^Q &= \begin{bmatrix}
        \vec{0} & \vec{0} \\ 
        \vec{W}^{Q,Q} & \vec{0}
    \end{bmatrix} \\
    \vec{W}^K &= \begin{bmatrix}
        \vec{0} & \vec{0} \\ 
        \vec{W}^{Q,K} & \vec{0}
    \end{bmatrix}.
\end{align*}
In both cases, we have that
\begin{align*}
    \vec{P}\vec{W}^Q &= \begin{bmatrix}
        \vec{Q}\vec{W}^{Q,Q} & \vec{0}
    \end{bmatrix} \\
    \vec{P}\vec{W}^K &= \begin{bmatrix}
        \vec{Q}\vec{W}^{Q,K} & \vec{0}
    \end{bmatrix} \\
\end{align*}
and consequently,
\begin{equation*}
    \vec{P}\vec{W}^Q(\vec{P}\vec{W}^K)^T = \vec{Q}\vec{W}^{Q,Q} (\vec{Q}\vec{W}^{Q,K})^T.
\end{equation*}
Hence if $\vec{Q}_1$ is (sufficiently) node-identifying, then $\vec{P}$ is (sufficiently) node-identifying. If $\vec{Q}_2$ is (sufficiently) node-identifying, then $\vec{P}$ is (sufficiently) node-identifying.
If $\vec{Q}_1$ is (sufficiently) adjacency-identifying, then $\vec{P}$ is (sufficiently) adjacency-identifying. If $\vec{Q}_2$ is (sufficiently) adjacency-identifying, then $\vec{P}$ is (sufficiently) adjacency-identifying. If $\vec{Q}_1$ is (sufficiently) node-identifying and $\vec{Q}_2$ is (sufficiently) adjacency-identifying or vice versa, then $\vec{P}$ is (sufficiently) node \textbf{and} adjacency-identifying. This concludes the proof.
\end{proof}

We will now prove a slightly more general statement than \Cref{theorem:node_and_adj_ident}. %
\begin{theorem}[Slightly more general than \Cref{theorem:node_and_adj_ident} in main text] \label{theorem:node_and_adj_ident_full}
Structural embeddings with $\embed{LPE}$ according to \Cref{eq:1wl_proof_lap_def} as node-level PEs with embedding dimension $d$ are sufficiently node- and adjacency-identifying, irrespective of whether the underlying Laplacian is normalized or not. Further, for graphs with $n$ nodes, the statement holds for $d \geq (2n + 1)$.
\end{theorem}
\begin{proof}
We begin by stating that the domain of $\embed{LPE}$ is compact since eigenvectors are unit-norm and eigenvalues are bounded by twice the maximum node degree for graphs without self-loops \citep{AndersonMorley+1985}. Further, the domain of the structural embeddings is compact, since $\embed{LPE}$ is compact and the node degrees are finite. Hence, the domain of $\embed{deg}$ is compact.

Let $G$ be a graph with graph Laplacian $\vec{L}$. Let $\vec{L} = \vec{U}\vec{\Sigma}\vec{U}^T$ be the eigendecomposition of $\vec{L}$, where the $i$-th column of $\vec{U}$ contains the $i$-th eigenvector, denoted $\vec{v}_i$, and the $i$-th diagonal entry of $\vec{\Sigma}$ contains the $i$-th eigenvalue, denoted $\lambda_i$.

Recall that structural embeddings with $\embed{LPE}$ as node-level PEs are defined as 
\begin{equation*}
    \vec{P}(v) = \textsf{FFN}(\embed{deg}(v) + \embed{LPE}(v)),
\end{equation*}
for all $v \in V(G)$.
Since the domains of both $\embed{deg}$ and $\embed{LPE}$ are compact, there exist parameters for both of these embeddings such that without loss of generality we may assume that for $v \in V(G)$,
\begin{equation*}
    \embed{deg}(v) + \embed{LPE}(v) = \begin{bmatrix}
        \embed{deg}(v) & \embed{LPE}(v)
    \end{bmatrix} \in \mathbb{R}^{d}
\end{equation*}
and that $\embed{deg}$ and $\embed{LPE}$ are instead embedded into some smaller $p$-dimensional and $s$-dimensional sub-spaces, respectively, where it holds that $d = p + s$.

To show node \textbf{and} adjacency identifiability, we divide up the $s$-dimensional embedding space of $\embed{LPE}$ into two $\sfrac{s}{2}$-dimensional sub-spaces $\embed{node}$ and $\embed{adj}$, i.e., we write
\begin{equation*}
   \embed{LPE}(v) = \begin{bmatrix}
       \embed{node}(v) & \embed{adj}(v)
   \end{bmatrix} \in \mathbb{R}^{n \times s},
\end{equation*}
where
\begin{align*}
    \embed{node}(v) &= \rho_\text{node} \Big( \sum_{j=1}^k
      \phi_\text{node}(\vec{v}_{ji}, \lambda_j + \epsilon_j)
 \Big) \\
    \embed{adj}(v) &= \rho_\text{adj} \Big( \sum_{j=1}^k
      \phi_\text{adj}(\vec{v}_{ji}, \lambda_j + \epsilon_j)
 \Big),
\end{align*}
and where we have chosen $\rho$ to be a sum over the first dimension of its input, followed by FFN $\rho_\text{node}$ and $\rho_\text{adj}$, for $\embed{node}$ and $\embed{adjacency}$, respectively.
Note that we have written out \Cref{eq:1wl_proof_lap_def} for a single node $v$.
For convenience, we also fix an arbitrary ordering over the nodes in $V(G)$ and define
\begin{align*}
    \vec{Q}^D &= \begin{bmatrix}
        \embed{deg}(v_1) \\
        \vdots \\
        \embed{deg}(v_n)
    \end{bmatrix} \\
    \vec{Q}^N &= \begin{bmatrix}
        \embed{node}(v_1) \\
        \vdots \\
        \embed{node}(v_n)
    \end{bmatrix} \\
    \vec{Q}^A &= \begin{bmatrix}
        \embed{adj}(v_1) \\
        \vdots \\
        \embed{adj}(v_n)
    \end{bmatrix},
\end{align*}
where $v_i$ is the $i$-th node in our ordering.
Further, note that $\embed{node}$ and $\embed{adj}$ are DeepSets \citep{ZaheerNIPS2017DeepSets} over the set
\begin{equation*}
    M_i \coloneqq \{ (\vec{v}_{ji}, \lambda_j + \epsilon_j) \}_{j=1}^k,
\end{equation*}
where $\vec{v}_{j}$ is again the $j$-th eigenvector with corresponding eigenvalue $\lambda_j$ and $\epsilon_j$ is a learnable scalar.
With DeepSets we can universally approximate permutation invariant functions. We will use this fact in the following, where we show that there exists a parameterization of $\vec{Q}^N$ and $\vec{Q}^A$ such that $\vec{Q}^N$ is sufficiently node-identifying and $\vec{Q}^A$ is sufficiently adjacency-identifying. Then, it follows from \Cref{lemma:node_and_adj_subspace} that $\vec{P}$ is sufficiently node and adjacency-identifying. We will further use $\vec{Q}^D$ later in the proof. 

We begin by showing that $\vec{Q}^N$ is sufficiently node-identifying. Observe that the eigenvector matrix $\vec{U}$ is already node-identifying since $\vec{U}$ forms an orthonormal basis and hence $\vec{U}\vec{U}^T$ is the $n$-dimensional identity matrix $\vec{I}$. Clearly for $\vec{I}$ it holds that
\begin{equation*}
    \vec{I}_{ij} = \max_k \vec{I}_{ik},
\end{equation*}
if and only if $i = j$.
Moreover, let $\vec{M}$ be any permutation matrix, i.e., each column of $\vec{M}$ is $1$ at exactly one position and $0$ else and the rows of $\vec{M}$ sum to $1$. Then,
\begin{equation*}
    \vec{U}\vec{M}\vec{M}^T\vec{U}^T = \vec{U}\vec{U}^T = \vec{I}
\end{equation*}
is also node-identifying.
We will now approximate $\vec{U}\vec{M}$ for some $\vec{M}$ with $\embed{node}$ arbitrarily close. Specifically, for the $i$-th node $v_i$ in our node ordering, we choose $\rho_\text{node}$ and $\phi_\text{node}$ such that
\begin{equation*}
    \big|\big| \embed{node}(v_i) - \vec{U}_i\vec{M} \big|\big|_\text{F} < \epsilon,
\end{equation*}
for all $\epsilon > 0$ and all $i$. 
Since DeepSets can universally approximate permutation invariant functions, it remains to show that there exists a permutation invariant function $f$ such that $f(M_i) = \vec{U}_i\vec{M}$ for all $i$ and for some $\vec{M}$. To this end, note that for a graph $G$, there are only at most $n$ unique eigenvalues of the corresponding (normalized) graph Laplacian. Hence, we can choose $\epsilon_j$ such that $\lambda_j + \epsilon_j$ is unique for each unique $j$. In particular, let
\begin{equation*}
    \epsilon_j = j \cdot \delta,
\end{equation*}
where we choose $\delta < \min_{l,o} |\lambda_l - \lambda_o| > 0$, i.e., the smallest non-zero difference between two eigenvalues.
We now define $f$ as
\begin{equation*}
    f(\{ (\vec{v}_{ji}, \lambda_j + \epsilon_j) \}_{j=1}^k) = \begin{bmatrix}
        \vec{v}_{1i} & \dots & \vec{v}_{ki}
    \end{bmatrix} = \vec{U}_i\vec{M},
\end{equation*}
where the order of the components is according to the sorted $\lambda_j + \epsilon_j$ in ascending order. This order is reflected in some permutation matrix $\vec{M}$.
Hence, $f$ is permutation invariant and can be approximated by a DeepSet arbitrarily close. 
As a result, we have
\begin{equation*}
    \big|\big| \embed{node}(v_i) - \vec{U}_i\vec{M} \big|\big|_\text{F} < \epsilon,
\end{equation*}
for all $i$ and an arbitrarily small $\epsilon > 0$.
In matrix form, we have that
\begin{equation*}
    \big|\big| \vec{U}\vec{M} - \vec{Q}^N \big|\big|_\text{F} < \epsilon
\end{equation*}
and since $\vec{U}\vec{M}$ is node-identifying, we can invoke \Cref{lemma:adjacency_ident_approx}, to conclude that $\vec{Q}^N$ is sufficiently node-identifying.
As a result, $\vec{P}$ has a sufficiently node-identifying sub-space and is thus, also sufficiently node-identifying according to \Cref{lemma:node_and_adj_subspace}.

We continue by showing that $\vec{Q}^A$ is sufficiently adjacency-identifying. Note that according to \Cref{lemma:UE_adjacency_identifying}, $\vec{U}\vec{\Sigma}^{\frac{1}{2}}$ is adjacency-identifying.
We will now approximate $\vec{U}\vec{\Sigma}^{\frac{1}{2}}$ with $\embed{adj}$ arbitrarily close. Specifically, for the $i$-th node $v_i$ in our node ordering, we choose $\rho_\text{adj}$ and $\phi_\text{adj}$ such that
\begin{equation*}
    \big|\big| \embed{adj}(v_i) -  \vec{U}\vec{\Sigma}^{\frac{1}{2}}_i \big|\big|_\text{F} < \epsilon,
\end{equation*}
for all $\epsilon > 0$ and all $i$. 
To this end, we first note that right-multiplication of $\vec{U}$ by $\vec{\Sigma}^{\frac{1}{2}}$ is equal to multiplying the $i$-th column of $\vec{U}$, i.e., the eigenvector $\vec{v}_i$, with the $j$-th diagonal element of $\vec{\Sigma}^{\frac{1}{2}}$, i.e., $\sqrt{\lambda_j}$. Hence, for the $j$-node $v_j$ it holds
\begin{equation*}
    \vec{U}\vec{\Sigma}^{\frac{1}{2}}_i = \begin{bmatrix}
       \vec{v}_{1j} \cdot \sqrt{\lambda_1} & \dots & \vec{v}_{nj} \cdot \sqrt{\lambda_n}
    \end{bmatrix} \in \mathbb{R}^n,
\end{equation*}
where $\vec{v}_{ij}$ denotes the $j$-th component of $\vec{v}_i$.
Since DeepSets can universally approximate permutation invariant functions, it remains to show that there exists a permutation invariant function $f$ such that $f(M_i) = \vec{U}\vec{\Sigma}^{\frac{1}{2}}_i\vec{M}$ for all $i$ and for some $\vec{M}$. To this end, note that for a graph $G$, there are only at most $n$ unique eigenvalues of the corresponding (normalized) graph Laplacian. Hence, we can choose $\epsilon_j$ such that $\lambda_j + \epsilon_j$ is unique for each unique $j$. In particular, let
\begin{equation*}
    \epsilon_j = j \cdot \delta,
\end{equation*}
where we choose $\delta < \min_{l,o} |\lambda_l - \lambda_o| > 0$, i.e., the smallest non-zero difference between two eigenvalues.
We now define $f$ as
\begin{equation*}
    f(\{ (\vec{v}_{ji}, \lambda_j + j \cdot \delta) \}_{j=1}^k) = \begin{bmatrix}
        \vec{v}_{1i} \cdot \sqrt{\lambda_1 + \epsilon_1 + \delta} & \dots & \vec{v}_{ki} \cdot \sqrt{\lambda_k + k \cdot \delta}
    \end{bmatrix},
\end{equation*}
where the order of the components is according to the sorted $\lambda_j + \epsilon_j$ in ascending order. This order is reflected in some permutation matrix $\vec{M}$, which we will use next.
Now, since we can choose $\delta$ arbitrarily low, we can choose it such that
\begin{equation*}
   \big|\big| f(\{ (\vec{v}_{ji}, \lambda_j + j \cdot \delta) \}_{j=1}^k) -  \vec{U}\vec{\Sigma}^{\frac{1}{2}}_i\vec{M} ||_\text{F} < \epsilon,
\end{equation*}
for any $\epsilon > 0$. 
Further, $f$ is permutation invariant and can be approximated by a DeepSet arbitrarily close. 
As a result, we have
\begin{equation*}
    \big|\big| \embed{adj}(v_i) - \vec{U}\vec{\Sigma}^{\frac{1}{2}}_i\vec{M} \big|\big|_\text{F} < \epsilon,
\end{equation*}
for all $i$ and an arbitrarily small $\epsilon > 0$.
In matrix form, we have that
\begin{equation*}
    \big|\big| \vec{U}\vec{\Sigma}^{\frac{1}{2}}\vec{M} - \vec{Q}^A \big|\big|_\text{F} < \epsilon.
\end{equation*}

Now, we need to distinguish between the non-normalized and normalized Laplacian. 
In the first case, we consider the graph Laplacian underlying the eigendecomposition to be non-normalized, i.e., $\vec{L} = \vec{D} - \vec{A}(G)$.
First, we know from \Cref{lemma:UE_adjacency_identifying} that $\vec{U}\vec{\Sigma}^{\frac{1}{2}}\vec{M}$ is adjacency-identifying. Further, we know from \Cref{lemma:adjacency_ident_approx} that then $\vec{Q}^A$ is sufficiently adjacency-identifying, since $\vec{Q}^A$ can approximate $\vec{U}\vec{\Sigma}^{\frac{1}{2}}\vec{M}$ arbitrarily close. As a result, $\vec{P}$ has a sufficiently adjacency-identifying sub-space and is, thus, also sufficiently adjacency-identifying.

In the second case, we consider the graph Laplacian underlying the eigendecomposition to be normalized, i.e., $\Tilde{\vec{L}} = \vec{I} - \vec{D}^{-\frac{1}{2}}\vec{A}(G)\vec{D}^{-\frac{1}{2}}$. Here, recall our construction for $\vec{P}$, namely
\begin{equation*}
    \vec{P}(v) = \textsf{FFN}\Big(\begin{bmatrix}
        \embed{deg}(v) & \embed{node}(v) & \embed{adj}(v)
    \end{bmatrix}\Big)
\end{equation*}
or in matrix form
\begin{equation*}
    \vec{P} = \textsf{FFN}\Big(\begin{bmatrix}
        \vec{Q}^D & \vec{Q}^N & \vec{Q}^A
    \end{bmatrix}\Big).
\end{equation*}
We will now use $\textsf{FFN}$ to approximate the following function $f$, defined as
\begin{equation*}
    f\Big(\begin{bmatrix}
        \vec{Q}^D & \vec{Q}^N & \vec{Q}^A
    \end{bmatrix}\Big) = \begin{bmatrix}
        \vec{Q}^D & \vec{Q}^N & \vec{D}^{\frac{1}{2}}\vec{Q}^A
    \end{bmatrix}.
\end{equation*}
Note that a FFN can approximate such a function arbitrarily close since our domain is compact and left-multiplication by $\vec{D}^{\frac{1}{2}}$ is equivalent to multiplying the $i$-th row of $\vec{Q}^A$ with $\sqrt{d_i}$, where $d_i$ is the degree of node $i$. Further, the $i$-th row of $\vec{Q}^D$ is an embedding $\embed{deg}(v_i)$ of $d_i$. We can choose this embedding to be
\begin{equation*}
    \embed{deg}(v_i) = \begin{bmatrix}
        \sqrt{d_i} & 0 & \dots & 0
    \end{bmatrix} \in \mathbb{R}^{p},
\end{equation*}
where we write $\sqrt{d_i}$ into the first component and pad the remaining vector with zeros to fit the target dimension $p$ of $\embed{deg}$. Hence, with the FFN we can approximate a sub-space containing $\vec{D}^{\frac{1}{2}}\vec{Q}^A$ arbitrarily close. Further, since we already showed that $\vec{Q}^A$ can approximate $\vec{U}\vec{\Sigma}^{\frac{1}{2}}\vec{M}$ arbitrarily close, we can thus approximate $\vec{D}^{\frac{1}{2}}\vec{U}\vec{\Sigma}^{\frac{1}{2}}\vec{M}$ arbitrarily close.
First, we know from \Cref{lemma:DUE_adjacency_identifying} that $\vec{D}^{\frac{1}{2}}\vec{U}\vec{\Sigma}^{\frac{1}{2}}\vec{M}$ is adjacency-identifying. Further, we know from \Cref{lemma:adjacency_ident_approx} that then $\vec{D}^{\frac{1}{2}}\vec{Q}^A$ is sufficiently adjacency-identifying, since $\vec{D}^{\frac{1}{2}}\vec{Q}^A$ can approximate $\vec{D}^{\frac{1}{2}}\vec{U}\vec{\Sigma}^{\frac{1}{2}}\vec{Q}^A$ arbitrarily close. As a result, $\vec{P}$ has a sufficiently adjacency-identifying sub-space and is, thus, also sufficiently adjacency-identifying according to \Cref{lemma:node_and_adj_subspace}.

Finally, for $\vec{Q}^D$ we need $p \geq 1$, for $\vec{Q}^N$ and $\vec{Q}^A$ we need $\frac{s}{2} \geq n$. As a result, the above statements hold for $d \geq (2n + 1)$. This concludes the proof.
\end{proof}

\subsection{SPE}\label{app:proof_spe}
Here, we show that SPE are sufficiently node- and adjacency-identifying. We begin with useful lemma.
\begin{lemma}\label{lemma:laplacian_degree_norm}
Let $G$ be a graph with graph Laplacian $\vec{L}$ and normalized graph Laplacian $\Tilde{\vec{L}}$. Then, if for a node-level PE $\vec{Q}$ with compact domain there exists matrices $\vec{W}^Q$ and $\vec{W}^K$ such that
\begin{equation*}
   \frac{1}{\sqrt{d_k}}(\vec{Q} \vec{W}^Q)(\vec{Q} \vec{W}^K)^T = \Tilde{\vec{L}},
\end{equation*}
there exists a parameterization of the structural embedding $\vec{P}$ with $\vec{Q}$ as node-level PE such that $\vec{P}$ is sufficiently node- and adjacency-identifying.
\end{lemma}
\begin{proof}
Recall that structural embeddings with $\vec{P}$ as node-level PEs are defined as 
\begin{equation*}
    \vec{P}(v) = \textsf{FFN}(\embed{deg}(v) + \vec{Q}(v)),
\end{equation*}
for all $v \in V(G)$.
Since the domains of both $\embed{deg}$ and $\vec{Q}$ are compact, there exists parameters for both of these embeddings such that without loss of generality we may assume that for $v \in V(G)$,
\begin{equation*}
    \embed{deg}(v) + \vec{Q}(v) = \begin{bmatrix}
        \embed{deg}(v) & \vec{Q}(v)
    \end{bmatrix} \in \mathbb{R}^{d}
\end{equation*}
and that $\embed{deg}$ and $\vec{Q}$ are instead embedded into some smaller $p$-dimensional and $s$-dimensional sub-spaces, respectively, where it holds that $d = p + s$.

Next, we choose the embedding $\embed{deg}(v)$ to be
\begin{equation*}
    \embed{deg}(v) = \begin{bmatrix}
        \sqrt{d_v} & 0 & \dots & 0
    \end{bmatrix} \in \mathbb{R}^{p},
\end{equation*}
where $d_v$ is the degree of node $v$ and we write $\sqrt{d_v}$ into the first component and pad the remaining vector with zeros to fit the target dimension $p$ of $\embed{deg}$.
We will now use $\textsf{FFN}$ to approximate the following function $f$, defined as
\begin{equation*}
    f\Big(\begin{bmatrix}
        \embed{deg}(v) & \vec{Q}(v)
    \end{bmatrix}\Big) = \begin{bmatrix}
        \embed{deg}(v) & \sqrt{d_v}\vec{Q}(v)
    \end{bmatrix}.
\end{equation*}
Note that a FFN can approximate such a function arbitrarily close since our domain is compact.
Hence, we have that
\begin{equation*}
    \big|\big| \vec{P}(v) - \begin{bmatrix}
        \embed{deg}(v) & \sqrt{d_v}\vec{Q}(v)
    \end{bmatrix} \big|\big|_\text{F} < \epsilon_1,
\end{equation*}
for all $\epsilon_1 > 0$. Further, in matrix form this becomes
\begin{equation*}
    \Bigg|\Bigg| \vec{P} - \begin{bmatrix}
        \begin{bmatrix}
            \embed{deg}(v_1) \\
            \vdots \\
            \embed{deg}(v_n)
        \end{bmatrix}
         & \vec{D}^\frac{1}{2}\vec{Q}
    \end{bmatrix} \Bigg|\Bigg|_\text{F} < \epsilon_2,
\end{equation*}
for all $\epsilon_2 > 0$, since left multiplication of a matrix by $\vec{D}^\frac{1}{2}$ corresponds to an element-wise multiplication of the $i$-th row of $\vec{Q}$ with $\sqrt{d_{v_i}}$, the square-root of the degree of the $i$-th node $v_i$ in an arbitrary but fixed node ordering.
Now, since we have that there exists matrices $\vec{W}^Q$ and $\vec{W}^K$ such that
\begin{equation*}
   \frac{1}{\sqrt{d_k}}(\vec{Q} \vec{W}^Q)(\vec{Q} \vec{W}^K)^T = \Tilde{\vec{L}},
\end{equation*}
then, for the same matrices $\vec{W}^Q$ and $\vec{W}^K$ we know that
\begin{equation*}
   \frac{1}{\sqrt{d_k}}(\vec{D}^\frac{1}{2}\vec{Q} \vec{W}^Q)(\vec{D}^\frac{1}{2}\vec{Q} \vec{W}^K)^T = \vec{D}^\frac{1}{2}\Tilde{\vec{L}}\vec{D}^\frac{1}{2} = \vec{L}
\end{equation*}
Finally, since $\vec{P}$ has a sub-matrix that can approximate $\vec{D}^\frac{1}{2}\vec{Q}$ arbitrarily close and since, by \Cref{lemma:laplacian_factorization}, $\vec{D}^\frac{1}{2}\vec{Q}$ is node- and adjacency-identifying, $\vec{P}$ is sufficiently node- and adjacency-identifying. This shows the statement.
\end{proof}

We now show \Cref{theorem:spe_node_and_adj_ident}.
\begin{theorem}[\Cref{theorem:spe_node_and_adj_ident} in the main paper]
Structural embeddings with SPE as node-level PE are sufficiently node- and adjacency-identifying.
\end{theorem}
\begin{proof}
We begin by stating that the domain of SPE is compact, since eigenvectors are unit-norm and eigenvalues are bounded by twice the maximum node degree for graphs without self-loops \citep{AndersonMorley+1985}. Further, the domain of the structural embeddings is compact, since SPE is compact and the node degrees are finite and hence the domain of $\embed{deg}$ is compact.

Let $G$ be a graph with (normalized or non-normalized) graph Laplacian $\vec{L}$. Let $\vec{L} = \vec{V}\vec{\Sigma}\vec{V}^T$ be the eigendecomposition of $\vec{L}$, where the $i$-th column of $\vec{V}$ contains the $i$-th eigenvector, denoted $\vec{v}_i$, and the $i$-th diagonal entry of $\vec{\Sigma}$ contains the $i$-th eigenvalue, denoted $\lambda_i$.

Recall that structural embeddings with SPE as node-level PEs are defined as 
\begin{equation*}
    \vec{P}(v) = \textsf{FFN}(\embed{deg}(v) + \text{SPE}(v)),
\end{equation*}
for all $v \in V(G)$, where
\begin{equation*}
   \text{SPE}(v) = \rho\Big(\begin{bmatrix}
        \mathbf{V} \text{diag}(\phi_1(\mathbf{\lambda})) \mathbf{V}^T & \dots & \mathbf{V} \text{diag}(\phi_n(\mathbf{\lambda})) \mathbf{V}^T
    \end{bmatrix}\Big)(v). 
\end{equation*}

To show node \textbf{and} adjacency identifiability, we first replace the neural networks $\rho, \phi_1, \dots, \phi_n$ in SPE with functions permutation equivariant functions $g, h_1, \dots, h_n$ over the same domain and co-domain, respectively, and choose $g$ and $h_1, \dots, h_n$ such that the resulting encoding, which we call $f_\text{SPE}$, is sufficiently node- and adjacency-identifying. Then, we show that $\rho$ and $\phi_1, \dots, \phi_n$ can approximate $g$ and $h_1, \dots, h_n$ arbitrarily close, which gives the proof statement.

We first define, for all $\ell \in [n]$,
\begin{equation*}
    h_\ell(\lambda) = \begin{bmatrix}
        0 \\
        \vdots \\
        \lambda_\ell^\frac{1}{2} \\
             \vdots \\
0  
\end{bmatrix},
\end{equation*}
such that $\lambda_\ell$ is the $\ell$-th sorted eigenvalue where ties are broken arbitrarily\footnote{Note that despite eigenvalues with higher multiplicity, the output of $h_\ell$ is the same, no matter how ties are broken.} and $\lambda_\ell$ is the $\ell$-th entry of $h_\ell(\vec{\lambda})$. Note that due to the sorting, $h_\ell$ is equivariant to permutations of the nodes. 

We then define the matrix $\mathbf{M}_i$ as $i$-th input to $g$ in $f_\text{SPE}$, for all $i \in [n]$. We have
\begin{align*}
    \mathbf{M}_i &\coloneqq \begin{bmatrix}
        \mathbf{V} \text{diag}(h_1(\mathbf{\lambda})) \mathbf{V}_i^T & \dots & \mathbf{V} \text{diag}(h_n(\mathbf{\lambda})) \mathbf{V}_i^T
    \end{bmatrix} \\
    &= \begin{bmatrix}
      \sum_j v_{1j} \cdot h_1(\vec{\lambda})_j \cdot v_{ij} & \dots & \sum_j v_{1j} \cdot h_n(\vec{\lambda})_j \cdot v_{ij} \\
    & \ddots & \\
    \sum_j v_{nj} \cdot h_1(\vec{\lambda})_j \cdot v_{ij} & \dots & \sum_j v_{nj} \cdot h_n(\vec{\lambda})_j \cdot v_{ij}
    \end{bmatrix} \\
    &= \begin{bmatrix}
       v_{11} \cdot \lambda_1^\frac{1}{2} \cdot v_{i1} & \dots & v_{1n} \cdot \lambda_n^\frac{1}{2} \cdot v_{in} \\
    & \ddots & \\
    v_{n1} \cdot \lambda_1^\frac{1}{2} \cdot v_{i1} & \dots & v_{nn} \cdot \lambda_n^\frac{1}{2} \cdot v_{in}
    \end{bmatrix},
\end{align*}
where the last equality follows from the fact that $h_\ell(\vec{\lambda})_j \neq 0$ if and only if $\ell = j$.
Now, we choose a linear set function $f$, satisfying
\begin{equation*}
    f(\{ a \cdot x \mid x \in X \}) = a \cdot f(X),
\end{equation*}
for all $a \in \mathbb{R}$ and all $X \subset \mathbb{R}$, e.g., \textsf{sum} or \textsf{mean}. 
Since $f$ is a function over sets, $f$ is equivariant to the ordering of $X$. 
We now define the function $g$ step-by-step.
In $g$, we first apply $f$ to the columns of $\mathbf{M}_i$ for all $i \in [n]$. To this end, let $\mathbf{M}_{ij}$ denote the \textit{set} of entries in the $j$-column of matrix $\mathbf{M}_i$. Then, we define
\begin{align*}
    \mathbf{f}_{i} 
     &= \begin{bmatrix}
         f(\mathbf{M}_{i1}) & \dots & f(\mathbf{M}_{in}) 
    \end{bmatrix} \\
    &= \begin{bmatrix}
         \lambda_1^\frac{1}{2} \cdot v_{i1} \cdot f(\{v_{o1} \mid o \in [n] \})& \dots & \lambda_n^\frac{1}{2} \cdot v_{in} \cdot f(\{v_{on} \mid o \in [n] \}) 
    \end{bmatrix},
\end{align*}
We then multiply $\mathbf{f}_i$ element-wise to each row of $\mathbf{M}_i$ and obtain a matrix
\begin{align*}
    \mathbf{P}_i &= \mathbf{M}_i \cdot \mathbf{f}_i \\ &= \begin{bmatrix}
       v_{11} \cdot (\lambda_1^\frac{1}{2})^2 \cdot (v_{i1})^2 & \dots & v_{1n} \cdot (\lambda_1^\frac{1}{2})^2 \cdot (v_{in})^2 \\
    & \ddots & \\
    v_{n1} \cdot (\lambda_1^\frac{1}{2})^2 \cdot (v_{i1})^2 & \dots & v_{nn} \cdot (\lambda_n^\frac{1}{2})^2 \cdot (v_{in})^2
        \end{bmatrix} \\
&= \begin{bmatrix}
       v_{11} \cdot \lambda_1 \cdot (v_{i1})^2 & \dots & v_{1n} \cdot \lambda_n \cdot (v_{in})^2 \\
    & \ddots & \\
    v_{n1} \cdot \lambda_1 \cdot (v_{i1})^2 & \dots & v_{nn} \cdot \lambda_n \cdot (v_{in})^2
        \end{bmatrix}.
\end{align*}
Further, we divide each row of $\mathbf{M}_i$ element-wise by $\mathbf{f}_i$ and obtain a matrix
\begin{align*}
    \mathbf{Q}_i &= \mathbf{M}_i / \mathbf{f}_i \\ &= \begin{bmatrix}
       \dfrac{v_{11}}{f(\{v_{o1} \mid o \in [n] \})} & \dots & \dfrac{v_{1n}}{f(\{v_{on} \mid o \in [n] \})}\\
    & \ddots & \\
    \dfrac{v_{n1}}{f(\{v_{o1} \mid o \in [n] \})} & \dots & \dfrac{v_{nn}}{f(\{v_{on} \mid o \in [n] \})}
        \end{bmatrix}.
\end{align*}
Hence, we also need to ensure that
\begin{equation*}
    f(\{v_{oj} \mid o \in [n] \}) \neq 0,
\end{equation*}
for all $j \in [n]$. Now, we define
\begin{equation*}
    \mathbf{P} = \sum_i \mathbf{P}_i
\end{equation*}
and
\begin{equation*}
    \mathbf{Q} = \sum_i \mathbf{Q}_i.
\end{equation*}
and have that
\begin{align*}
    (\mathbf{P}\mathbf{Q}^T)_{k,l} &=
        \sum_\ell \Big( v_{k\ell} \cdot \lambda_\ell \cdot f(\{v_{o\ell} \mid o \in [n] \})  \cdot \sum_i v_{i\ell}^2 \Big) \cdot \Big( \dfrac{v_{l\ell}}{f(\{v_{o\ell} \mid o \in [n] \})} \Big) \\
&= \sum_\ell v_{k\ell} \cdot v_{l\ell} \cdot \lambda_\ell,
\end{align*}
where we recall that the eigenvectors are normalized and hence, $\sum_i v_{i\ell}^2 = 1$.

In the last step of $g$, we return the matrix
\begin{equation*}
    \begin{bmatrix}
        \vec{P} & \vec{Q}
    \end{bmatrix} \in \mathbb{R}^{n \times 2n}.
\end{equation*}
We will now show that this matrix is node- and adjacency-identifying.
To this end, note that
\begin{equation*}
    (\mathbf{P}\mathbf{Q}^T)_{k,l} = (\vec{V}\Sigma\vec{V}^T)_{k, l}
\end{equation*}
and hence,
\begin{equation*}
    \mathbf{P}\mathbf{Q}^T = \vec{V}\Sigma\vec{V}^T = \vec{L}.
\end{equation*}
We distinguish between two cases. If $\vec{L}$ is the non-normalized graph Laplacian, then according to \Cref{lemma:laplacian_factorization_pq}, the matrix
\begin{equation*}
    \begin{bmatrix}
        \vec{P} & \vec{Q}
    \end{bmatrix}
\end{equation*}
is node- and adjacency-identifying and hence, according to \Cref{lemma:node_and_adj_subspace}, $f_\text{SPE}$ is node- and adjacency-identifying.
Further, if $\vec{L}$ is the normalized graph Laplacian, according to \Cref{lemma:laplacian_degree_norm}, $f_\text{SPE}$ is sufficiently node- and adjacency-identifying.

To summarize, we have shown that there exist functions $g, h_1, \dots, h_n$ such that if, in SPE, we replace $\rho$ with $g$ and $\phi_\ell$ with $h_\ell$, then the resulting encoding $f_\text{SPE}$ is sufficiently node- and adjacency-identifying.
To complete the proof, we will now show that we can approximate $g, h_1, \dots, h_n$ with $\rho, \phi_1, \dots, \phi_n$ arbitrarily close. 
Since our domain is compact, we can use $\phi_\ell$ to approximate $h_\ell$ arbitrarily close, i.e., we have that
\begin{equation*}
    \big|\big| \phi_\ell - h_\ell \big|\big|_\text{F} < \epsilon,
\end{equation*}
for all $\epsilon > 0$.
Further, $g$ consists of a sequence of permutation equivariant steps and is thus, permutation equivariant. Since the domain of $g$ is compact and, by construction, $g$ and $\rho$ have the same domain and co-domain, and since $\rho$ can approximate any permutation equivariant function on its domain and co-domain arbitrarily close, we have that
$\rho$ can also approximate $g$ arbitrarily close. As a result, SPE can approximate a sufficiently node- and adjacency-identifying encoding arbitrarily close, and hence, by definition, we have that SPE is also sufficiently node- and adjacency-identifying. This completes the proof.
\end{proof}

\section{Learnable embeddings}
Here, we pay more attention to learnable embeddings, which play an important role throughout our work. Although our definition of learnable embeddings is fairly abstract, learnable embeddings are very commonly used in practice, e.g., in tokenizers of language models \citep{Vaswani2017}. Since many components of graphs, such as the node labels, edge types, or node degrees are discrete, learnable embeddings are useful to embed these discrete features into an embedding space. In our work, different embeddings are typically joined via sum, e.g., in \Cref{eq:1_wl_token_embeddings} or \Cref{eq:kwl_token_embeddings}. Summing embeddings is much more convenient in practice since every embedding has the same dimension $d$. In contrast, joining embeddings via concatenation can lead to very large $d$, unless the underlying embeddings are very low-dimensional.
However, in our theorems joining embeddings via concatenation is much more convenient. To bridge theory and practice in this regard, in what follows we establish under what conditions summed embeddings can express concatenated embeddings and vice versa. Specifically, we show that summed embeddings can still act as if concatenating lower-dimensional embeddings but with more flexibility in terms of joining different embeddings.

The core idea is that we show under which operations one may replace a learnable embedding $\embed{e}$ to $\mathbb{R}^d$ with a lower-dimensional learnable embedding $\embed{e}'$ to $\mathbb{R}^{d'}$ while preserving the injectivity of $\embed{e}$. This property is useful since it allows us to, for example, add two learnable embeddings without losing expressivity. As a result, we can design a tokenizer that is practically useful without sacrificing theoretical guarantees.

We begin with a projection of a concatenation of learnable embeddings.
\begin{lemma}[Projection of learnable embeddings]\label{lemma:proj_learnable_embed}
Let $\embed{e}'_1, \dots, \embed{e}'_k$ be learnable embeddings from $\mathbb{N}$ to $\mathbb{R}^p$ for some $p \geq 1$. If $d \geq k$ then, there exists learnable embeddings $\embed{e}_1, \dots, \embed{e}_k$ as well as a projection matrix $\vec{W} \in \mathbb{R}^{d \cdot k \times k \cdot p}$ such that for $v_1, \dots, v_k \in \mathbb{N}$
\begin{equation*}
\big[\embed{e}_i(v_i)\big]_{i=1}^k \vec{W} = \big[\embed{e}'_i(v_i)\big]_{i=1}^k \in \mathbb{R}^{k \cdot p},
\end{equation*}
where 
it holds for all $v, w \in \mathbb{N}$ and all $i \in \{1, \dots, k\}$ that
\begin{equation*}
    \embed{e}_i(v) = \embed{e}_i(w) \Longleftrightarrow \embed{e}'_i(v) = \embed{e}'_i(w).
\end{equation*}
\end{lemma}
\begin{proof}
For a learnable embedding $\embed{e}$ and some $v \in \mathbb{N}$, we denote with $\embed{e}(v)_j$ the $j$-th element of the vector $\embed{e}(v)$.
Since learnable embeddings are from $\mathbb{N}$ to $\mathbb{R}^d$, for every $i$ we define $\embed{e}_i$ such that for every $v \in \mathbb{N}$, $\embed{e}_i(v)_1 = \embed{e}'_i(v)$ and $\embed{e}_i(v)_j = 0$ for $j > 0$.

Now, we define $\vec{W}$ as follows. For row $l$ and column $j$, we set $\vec{W}_{lj} = 1$ if $(l - 1)$ is a multiple of $d$, i.e., $l = 1, (d+1), 2(d+1), \dots$. Otherwise we set $\vec{W}_{lj} = 0$. Intuitively, applying $\vec{W}$ to $\big[\embed{e}_i(v_i)\big]_{i=1}^k$ will select the first, $d$-th, $2d$-th, $\dots$, $kd$-th, element of $\big[\embed{e}_i(v_i)\big]_{i=1}^k$, corresponding to $v_1, \dots, v_k$, respectively, according to our construction of $\embed{e}_i$.
Now, we have that %
\begin{equation*}
\big[\embed{e}_i(v_i)\big]_{i=1}^k \vec{W} = \big[\embed{e}'_1(v_1) \, \dots \, \embed{e}'_k(v_k)\big] = \big[\embed{e}'_i(v_i)\big]_{i=1}^k.
\end{equation*}
Further we have by construction that for all $v, w \in \mathbb{N}$ and all $i \in \{1, \dots, k\}$
\begin{equation*}
    \embed{e}_i(v) = \embed{e}_i(w) \Longleftrightarrow \embed{e}'_i(v) = \embed{e}'_i(w).
\end{equation*}
This shows the statement.
\end{proof}
Next, we turn to a composition of learnable embeddings, namely the structural embeddings $\vec{P}$, as defined in \Cref{sec:1wl_pure}. In particular, we show next that a projection of multiple structural embeddings preserves node- and adjacency identifiability of the individual structural embeddings in the new embedding space.

\begin{lemma}[Projection of structural embeddings]\label{lemma:proj_structural_embed}
Let $G$ be a graph with $n$ nodes and let $\vec{P} \in \mathbb{R}^{n \times d}$ be structural embeddings for $G$. Let $k > 1$. If $d \geq k \cdot (2n + 1)$ then, there exists a parameterization of $\vec{P}$ as well as a projection matrix $\vec{W} \in \mathbb{R}^{k \cdot d \cdot k \times (2n + 1)}$ such that for any $v_1, \dots, v_k \in V(G)$,
\begin{equation*}
\big[\vec{P}(v_i)\big]_{i=1}^k \vec{W} = \big[\vec{P}'(v_i)\big]_{i=1}^k \in \mathbb{R}^{n \times d},
\end{equation*}
where $\vec{P}' \in \mathbb{R}^{n \times d'}$ is also a structural embedding for $G$ and it holds that if $\vec{P}$ is sufficiently node- and adjacency-identifying, then so is $\vec{P}'$.
\end{lemma}
\begin{proof}
We know from \Cref{theorem:node_and_adj_ident_full} that there exists sufficiently node- and adjacency-identifying structural embeddings $\vec{P}'$ in $\mathbb{R}^{n \times (2d + 1)}$.
Since $d \geq k \cdot (2n + 1)$, we define 
\begin{equation*}
   \vec{P} = \begin{bmatrix}
       \vec{P}' & \vec{0}
   \end{bmatrix},
\end{equation*}
where $\vec{0} \in \mathbb{R}^{n \times (k-1) \cdot (2n + 1)}$ is an all-zero matrix. Since $\vec{P}$ has a sufficiently node- and adjacency-identifying subspace $\vec{P}'$, by \Cref{lemma:node_and_adj_subspace}, $\vec{P}$ is sufficiently node- and adjacency-identifying. Further, by setting
\begin{equation*}
   \vec{W} = \vec{I}_{n \times k \cdot (2n + 1)},
\end{equation*}
where $\vec{I}_{n \times k \cdot (2n + 1)}$ are the first $n$ rows of the $k \cdot (2n + 1)$-dimensional identity matrix and $\vec{0} \in \mathbb{R}^{n \times (k-1) \cdot (2n + 1)}$ is again an all-zero matrix.

Then, we have that for nodes $v_1, \dots, v_k \in V(G)$,
\begin{equation*}
\big[\vec{P}(v_i)\big]_{i=1}^k \vec{W} =
\big[\vec{P}'(v_i)\big]_{i=1}^k \in \mathbb{R}^{n \times d},
\end{equation*}
and $\vec{P}'$ is by definition sufficiently node- and adjacency-identifying, from which the statement follows.
\end{proof}

We continue by showing that our construction of token embeddings in \Cref{eq:1_wl_token_embeddings} can be equally well represented by another form that will be easier to use in proving \Cref{theorem:1_wl}.
\begin{lemma}\label{lemma:concatenated_1wl_token_embeddings}
Let $G$ be a graph with node features $\vec{F} \in \mathbb{R}^{n \times d}$. For every tokenization $\vec{X}^{(0,1)} \in \mathbb{R}^d$ according to \Cref{eq:1_wl_token_embeddings} with (sufficiently) adjacency reconstructing structural embeddings $\vec{P}(v)$, there exists a parameterization of $\vec{X}^{(0,1)}$ such that for node $v \in V(G)$ as
\begin{equation*}
  \vec{X}^{(0,1)}(v) = \begin{bmatrix}
     \vec{F}'(v) & \vec{0} & \embed{deg}'(v) & \vec{P}'(v)
  \end{bmatrix},
\end{equation*}
with
\begin{equation*}
   \vec{P}'(v) = \textsf{FFN}'(\embed{deg}'(v) + \embed{PE}(v)),
\end{equation*}
such that $\vec{F}'(v) \in \mathbb{R}^{p}$, $\vec{0} \in \mathbb{R}^{p}$, $\embed{deg}': \mathbb{N} \rightarrow \mathbb{R}^{r}$ and $\textsf{FFN}': \mathbb{R}^d \rightarrow \mathbb{R}^{s}$ for $d = 2p + r + s$ and where it holds for every $v,w \in V(G)$ that
\begin{equation*}
  \vec{F}(v) = \vec{F}(w) \Longleftrightarrow \vec{F}'(v) = \vec{F}'(w)
\end{equation*}
and
\begin{equation*}
  \embed{deg}(v) = \embed{deg}(w) \Longleftrightarrow \embed{deg}'(v) = \embed{deg}'(w)
\end{equation*}
and
\begin{equation*}
  \vec{P}(v) = \vec{P}(w) \Longleftrightarrow \vec{P}'(v) = \vec{P}'(w)
\end{equation*}
and $\vec{P}'(v)$ is (sufficiently) adjacency-identifying.
\end{lemma}
\begin{proof}
We set $p = \sfrac{(d-(2n + 1))}{2} - 1$, $r = 2$ and $s = (2n + 1)$. Indeed, we have that
\begin{equation*}
    2p + r + s = 2(\sfrac{(d-(2n + 1))}{2} - 1) + 2 + (2n + 1) = d-(2n + 1) + (2n + 1) = d.
\end{equation*}
We continue by noting that the node features $\vec{F}$ are mapped from $\mathbb{N}$ to $\mathbb{R}^d$ and can be obtained from a learnable embedding. More importantly, we can define
\begin{equation*}
    \vec{F}'(v) = \begin{bmatrix}
        \ell(v) & 0 & \dots & 0
    \end{bmatrix},
\end{equation*}
where $\vec{F}'(v) \in \mathbb{R}^{d}$ is another learnable embedding for which it holds that
\begin{equation*}
   \vec{F}(v) = \vec{F}(w) \Longleftrightarrow \vec{F}'(v) = \vec{F}'(w), 
\end{equation*}
for all $v,w \in V(G)$.

Further, we can write
\begin{equation*}
    \embed{deg}(v) + \embed{PE}(v) = \begin{bmatrix}
        \embed{deg}'(v) & \embed{PE}(v)
    \end{bmatrix},
\end{equation*}
where $\embed{deg}'$ is another learnable embedding for the node degrees mapping to $\mathbb{R}^{r}$.
Moreover, we know from \Cref{theorem:node_and_adj_ident_full} that there exist (sufficiently) adjacency-identifying structural embeddings of dimension $(2n + 1)$. We choose the $\textsf{FFN}$ in \Cref{eq:1_wl_structural_embeddings} such that
\begin{equation*}
    \vec{P}(v) = \begin{bmatrix}
        \vec{0} & \embed{deg}'(v) & \vec{P}'
    \end{bmatrix},
\end{equation*}
where $\vec{0} \in \mathbb{R}^{2p}$ is an all-zero vector and
\begin{equation*}
    \vec{P}' = \textsf{FFN}'(\embed{deg}'(v) + \embed{PE}(v)),
\end{equation*}
with $\textsf{FFN}': \mathbb{R}^d \rightarrow \mathbb{R}^{s}$, another FFN inside of $\textsf{FFN}$. Note that such a $\textsf{FFN}$ exists without approximation. Further according to \Cref{lemma:proj_structural_embed}, $\vec{P}'$ is still (sufficiently) adjacency-identifying. Hence, we can write
\begin{equation*}
   \vec{X}^{(0,1)}(v) = \vec{F}'(v) + \vec{P}'(v) = \begin{bmatrix}
      \vec{F}'(v) & \vec{0} & \embed{deg}'(v) & \textsf{FFN}'(\embed{deg}'(v) + \embed{PE}(v))
    \end{bmatrix},
\end{equation*} 
where through the concatenation it is easy to verify that indeed $\vec{X}^{(0,1)}(v)$ has the desired properties, namely that
\begin{equation*}
  \vec{F}(v) = \vec{F}(w) \Longleftrightarrow \vec{F}'(v) = \vec{F}'(w)
\end{equation*}
and
\begin{equation*}
  \embed{deg}(v) = \embed{deg}(w) \Longleftrightarrow \embed{deg}'(v) = \embed{deg}'(w)
\end{equation*}
and
\begin{equation*}
  \vec{P}(v) = \vec{P}(w) \Longleftrightarrow \vec{P}'(v) = \vec{P}'(w)
\end{equation*}
and $\vec{P}'(v)$ is (sufficiently) adjacency-identifying. This shows the statement.
\end{proof}

Lastly, we show a similar result for \Cref{eq:kwl_token_embeddings}.
\begin{lemma}\label{lemma:concatenated_kwl_token_embeddings}
Let $G$ be a graph with node features $\vec{F} \in \mathbb{R}^{n \times d}$. For every tokenization $\vec{X}^{(0,k)} \in \mathbb{R}^d$ according to \Cref{eq:kwl_token_embeddings} with node- and adjacency reconstructing structural embeddings $\vec{P}(v)$, there exists a parameterization of $\vec{X}^{(0,k)}$ such that for $k$-tuple $\vec{v} = (v_1, \dots, v_k) \in V(G)^k$,
\begin{equation*}
  \vec{X}^{(0,k)}(v) = \begin{bmatrix}
     \vec{F}'(v_1) \,\,\,\, \dots \,\,\,\, \vec{F}'(v_k) & 
     \vec{0} &
     \embed{deg}'(v_1) & \dots & \embed{deg}'(v_k) & \vec{P}'(v_1) & \dots & \vec{P}'(v_k) & \embed{atp}'(\vec{v})
  \end{bmatrix},
\end{equation*}
with
\begin{equation*}
   \vec{P}'(v) = \textsf{FFN}'(\embed{deg}'(v) + \embed{PE}(v)),
\end{equation*}
such that $\vec{F}'(v) \in \mathbb{R}^{p}$, $\vec{0} \in \mathbb{R}^{k^2 \cdot p}$, $\embed{deg}': \mathbb{N} \rightarrow \mathbb{R}^{r}$, $\embed{atp}': \mathbb{N} \rightarrow \mathbb{R}^{o}$ and $\vec{P}' \in \mathbb{R}^{s}$ for $d = k^2 \cdot p + k \cdot p +  k\cdot r + k\cdot s + o$ and where it holds for every $v, w \in V(G)$,
\begin{equation*}
  \vec{F}'(v) = \vec{F}(w) \Longleftrightarrow \vec{F}'(v) = \vec{F}'(w)
\end{equation*}
and
\begin{equation*}
  \embed{deg}(v) = \embed{deg}(w) \Longleftrightarrow \embed{deg}'(v) = \embed{deg}'(w)
\end{equation*}
and
\begin{equation*}
  \vec{P}(v) = \vec{P}(w) \Longleftrightarrow \vec{P}'(v) = \vec{P}'(w)
\end{equation*}
and $\vec{P}'(v)$ is node and adjacency-identifying
and for every $\vec{v}, \vec{w} \in V(G)^k$
\begin{equation*}
  \embed{atp}(\vec{v}) = \embed{atp}(\vec{w}) \Longleftrightarrow \embed{atp}'(\vec{v}) = \embed{atp}'(\vec{w}).
\end{equation*}
\end{lemma}
\begin{proof}
We set $p = 1$, $r = 2$, $o = 1$ and $s = \frac{d - k^2 - 3k - 1}{k}$. Indeed, we have that
\begin{equation*}
     k^2 \cdot p + k \cdot p + k \cdot r + k \cdot s + o = k^2 + 3k + d - k^2 - 3k - 1 + 1 = d.
\end{equation*}

Recall \Cref{eq:kwl_token_embeddings} as
\begin{equation*}
    \vec{X}^{(0, k)}(\vec{v}) = \big[
        \vec{F}(v_i)
    \big]_{i=1}^k \mathbf{W}^F + \big[
        \vec{P}(v_i)
    \big]_{i=1}^k \mathbf{W}^P + \embed{atp}(\vec{v}),
\end{equation*}
where $\mathbf{W}^F \in \mathbb{R}^{d \cdot k \times d}$ and $\mathbf{W}^P \in \mathbb{R}^{d \cdot k \times d}$ are projection matrices, $\vec{P}$ are structural embeddings and $\vec{v} = (v_1, \dots, v_k)$ is a $k$-tuple. 

We continue by noting that the node features $\vec{F}$ are mapped from $\mathbb{N}$ to $\mathbb{R}^d$ and can be obtained from a learnable embedding. More importantly, we can define $\vec{F}'(v) = \ell(v)$ for which it clearly holds that
\begin{equation*}
   \vec{F}(v) = \vec{F}(w) \Longleftrightarrow \vec{F}'(v) = \vec{F}'(w), 
\end{equation*}
for all $v,w \in V(G)$.
Invoking \Cref{lemma:proj_learnable_embed}, there exists learnable embeddings $\vec{F}$ and a projection matrix $\mathbf{W}^{F,*} \in \mathbb{R}^{k \cdot d \times k \cdot p}$ such that
\begin{equation*}
    \big[
        \vec{F}(v_i)
    \big]_{i=1}^k \mathbf{W}^F = \begin{bmatrix}
       \vec{F}'(v_1) & \dots & \vec{F}'(v_k)
    \end{bmatrix} \in \mathbb{R}^{k \cdot p}.
\end{equation*}
Note that since we do not have an assumption on $d$, we can make $d$ arbitrarily large to ensure that $d \geq k$.
Hence, by defining
\begin{equation*}
    \mathbf{W}^{F} = \begin{bmatrix}
        \mathbf{W}^{F,*} & \vec{0}
    \end{bmatrix},
\end{equation*}
where $\vec{0} \in \mathbb{R}^{k \cdot d \times (d - k \cdot p)}$ is an all-zero matrix,
we can write
\begin{equation*}
    \big[
        \vec{F}(v_i)
    \big]_{i=1}^k \mathbf{W}^F = \begin{bmatrix}
       \vec{F}'(v_1) & \dots & \vec{F}'(v_k) & \vec{0}
    \end{bmatrix} \in \mathbb{R}^d,
\end{equation*}
where $\vec{0} \in \mathbb{R}^{d - k \cdot p}$ is an all-zero vector.

Moreover, we know from \Cref{theorem:node_and_adj_ident} that there exists adjacency-identifying structural embeddings of dimension $(2n + 1)$. We choose the $\textsf{FFN}$ in \Cref{eq:1_wl_structural_embeddings} such that
\begin{equation*}
    \vec{P}(v) = \begin{bmatrix}
        \vec{0} & \embed{deg}'(v) & \vec{P}'(v)
    \end{bmatrix},
\end{equation*}
where $\vec{0} \in \mathbb{R}^{k^2\cdot p + k\cdot p}$ is an all-zero vector and
\begin{equation*}
    \vec{P}'(v) = \textsf{FFN}'(\embed{deg}'(v) + \embed{PE}(v)),
\end{equation*}
with $\textsf{FFN}': \mathbb{R}^d \rightarrow \mathbb{R}^{s}$, another FFN inside of $\textsf{FFN}$. Note that such a $\textsf{FFN}$ exists without approximation. Further, according to \Cref{lemma:proj_structural_embed}, $\vec{P}'$ is still sufficiently adjacency-identifying, and there exists a projection matrix $\mathbf{W}^{P,*} \in \mathbb{R}^{k \cdot d \times d - o}$ such that
\begin{equation*}
    \big[
        \vec{P}(v_i)
    \big]_{i=1}^k \mathbf{W}^{P,*} = \begin{bmatrix}
       \vec{0} & \embed{deg}'(v_1) & \vec{P}'(v_1) & \dots & \vec{0} & \embed{deg}'(v_k) & \vec{P}'(v_k)
    \end{bmatrix} \in \mathbb{R}^{d-o}.
\end{equation*}
But then, there also exists a permutation matrix $\vec{Q} \in \{0, 1\}^{d - o \times d - o}$ such that
\begin{equation*}
    \big[
        \vec{P}(v_i)
    \big]_{i=1}^k \mathbf{W}^{P,*}\vec{Q} = \begin{bmatrix}
       \vec{0} & \embed{deg}'(v_1) & \dots & \embed{deg}'(v_k) & \vec{P}'(v_1) & \dots & \vec{P}'(v_k)
    \end{bmatrix} \in \mathbb{R}^{d -o},
\end{equation*}
where we simply re-order the entries of $\big[\vec{P}(v_i)\big]_{i=1}^k \mathbf{W}^{P,*}$. Note that now, $\vec{0} \in \mathbb{R}^{k^2\cdot p + k\cdot p}$, as we grouped the zero-vectors of $\big[\vec{P}(v_i)\big]_{i=1}^k \mathbf{W}^{P,*}$ into a single zero-vector of size $k^2\cdot p + k\cdot p$.

Note that since we do not have an assumption on $d$, we can make $d$ arbitrarily large to ensure that $d \geq k \cdot (2n + 1)$.
Hence, by defining
\begin{equation*}
    \mathbf{W}^{P} = \begin{bmatrix}
        \mathbf{W}^{P,*}\vec{Q} & \vec{0}
    \end{bmatrix},
\end{equation*}
where $\vec{0} \in \mathbb{R}^{o}$ is an all-zero vector,
we can write
\begin{equation*}
    \big[
        \vec{P}(v_i)
    \big]_{i=1}^k \mathbf{W}^P = \begin{bmatrix}
       \vec{0} & \embed{deg}'(v_1) & \dots & \embed{deg}'(v_k) & \vec{P}'(v_1) & \dots & \vec{P}'(v_k) & \vec{0}
    \end{bmatrix} \in \mathbb{R}^{d},
\end{equation*}
where $\vec{P}'(v_i) \in \mathbb{R}^s$. Further, since $\vec{P}$ is by assumption sufficiently node- and adjacency-identifying, \Cref{lemma:proj_structural_embed} guarantees that then so is $\vec{P}'$.

Finally, for tuple $\vec{v} \in V(G)^k$, we set
\begin{equation*}
    \embed{atp}(\vec{v}) = \begin{bmatrix}
        \vec{0} & \embed{atp}'(\vec{v})
    \end{bmatrix},
\end{equation*}
where $\embed{atp}'(\vec{v}) = \text{atp}(\vec{v})$ and we recall that $\text{atp}$ maps to $\mathbb{N}$ and $\vec{0} \in \mathbb{R}^{d - 1}$ is an all-zero vector.

Then, we can write
\begin{align*}
    \vec{X}^{(0, k)}(\vec{v}) &= \big[
        \vec{F}(v_i)
    \big]_{i=1}^k \mathbf{W}^F + \big[
        \vec{P}(v_i)
    \big]_{i=1}^k \mathbf{W}^P + \embed{atp}(\vec{v})\\
    &=\begin{bmatrix}
     \vec{F}'(v_1) \,\,\,\, \dots \,\,\,\, \vec{F}'(v_k) & 
     \vec{0} &
     \embed{deg}'(v_1) & \dots & \embed{deg}'(v_k) & \vec{P}'(v_1) & \dots & \vec{P}'(v_k) & \embed{atp}'(\vec{v})
  \end{bmatrix},
\end{align*}
where it holds for every $v, w \in V(G)$,
\begin{equation*}
  \vec{F}'(v) = \vec{F}(w) \Longleftrightarrow \vec{F}'(v) = \vec{F}'(w)
\end{equation*}
and
\begin{equation*}
  \embed{deg}(v) = \embed{deg}(w) \Longleftrightarrow \embed{deg}'(v) = \embed{deg}'(w)
\end{equation*}
and
\begin{equation*}
  \vec{P}(v) = \vec{P}(w) \Longleftrightarrow \vec{P}'(v) = \vec{P}'(w)
\end{equation*}
where $\vec{P}'(v)$ is node and adjacency-identifying
and for every $\vec{v}, \vec{w} \in V(G)^k$
\begin{equation*}
  \embed{atp}(\vec{v}) = \embed{atp}(\vec{w}) \Longleftrightarrow \embed{atp}'(\vec{v}) = \embed{atp}'(\vec{w}).
\end{equation*}
This shows the statement.
\end{proof}

\section{Expressivity of \texorpdfstring{\kgt{1}}{1-GT}}\label{app:proof_1gt}
Here, we prove \Cref{theorem:1_wl} in the main paper.
We first restate Theorem VIII.4 of \citep{Gro+2021}, showing that a simple GNN can simulate the \wlone.
\begin{lemma}[\citet{Gro+2021}, Theorem VIII.4]\label{lemma:grohe_1wl_equiv}
Let $G = (V(G), E(G), \ell)$ be a graph with $n$ nodes with adjacency matrix $\vec{A}(G)$ and node feature matrix $\vec{X}^{(0)} \coloneqq \vec{F} \in \mathbb{R}^{n \times d}$ consistent with $\ell$. Further, assume a GNN that for each layer, $t>0$, updates the vertex feature matrix 
\begin{equation*}
    \vec{X}^{(t)} \coloneqq \mathsf{FFN} \Big(  \vec{X}^{(t-1)} + 2 \vec{A}(G)\vec{X}^{(t-1)} \Big).
\end{equation*}
Then, for all $t \geq 0$, there exists a parameterization of $\mathsf{FFN}$ such that
\begin{equation*}
    C^1_t(v) = C^1_t(w) \Longleftrightarrow  \vec{X}^{(t)}(v) =  \vec{X}^{(t)}(w),
\end{equation*}
for all vertices $v, w \in V(G)$.
\end{lemma}

We now give the proof for a slight generalization of \Cref{theorem:1_wl}. Specifically, we relax the adjacency-identifying condition to \textit{sufficiently} adjacency-identifying.

\begin{theorem}[Generalization of \Cref{theorem:1_wl} in main paper]
Let $G = (V(G), E(G), \ell)$ be a labeled graph with $n$ nodes and $\vec{F} \in \mathbb{R}^{n \times d}$ be a node feature matrix consistent with $\ell$. Let $C^1_t$ denote the coloring function of the \wlone{} at iteration $t$.
Then, for all iterations $t \geq 0$, there exists a parametrization of the \kgt{1} with sufficiently adjacency-identifying structural embeddings, such that
\begin{equation*}
    C^1_{t}(v) = C^1_{t}(w) \Longleftrightarrow \vec{X}^{(t, 1)}(v) = \vec{X}^{(t, 1)}(w),
\end{equation*}
for all nodes $v, w \in V(G)$.
\end{theorem}
\begin{proof}
According to \Cref{lemma:concatenated_1wl_token_embeddings}, there exists a parameterization of $\vec{X}^{(0, 1)}$ such that for each $v \in V(G)$,
\begin{equation*}
   \vec{X}^{(0,1)}(v) = \vec{F}'(v) + \vec{P}'(v) = \begin{bmatrix}
      \vec{F}'(v) & \vec{0} & \embed{deg}'(v) & \vec{P}'(v)
    \end{bmatrix},
\end{equation*}
where it holds that
\begin{equation}\label{eq:1wl_proof_F_consistency}
  \vec{F}(v) = \vec{F}(w) \Longleftrightarrow \vec{F}'(v) = \vec{F}'(w)
\end{equation}
and
\begin{equation*}
  \embed{deg}(v) = \embed{deg}(w) \Longleftrightarrow \embed{deg}'(v) = \embed{deg}'(w)
\end{equation*}
and
\begin{equation*}
  \vec{P}(v) = \vec{P}(w) \Longleftrightarrow \vec{P}'(v) = \vec{P}'(w)
\end{equation*}
and $\vec{P}'(v)$ is adjacency-identifying. Further, $\vec{F}'(v), \vec{0} \in \mathbb{R}^p$, $\embed{deg}'(v) \in \mathbb{R}^r$ and $\vec{P}'(v) \in \mathbb{R}^s$, for some choice of $p, r, s$ where $d = 2p + r + s$, specified in \Cref{lemma:concatenated_1wl_token_embeddings}.
We will use this parameterization for $\vec{X}^{(0,1)}$ throughout the rest of the proof.

We will now prove the statement by induction over $t$. Since by definition
\begin{equation*}
    C^1_{0}(v) = C^1_{0}(w) \Longleftrightarrow \vec{F}(v) = \vec{F}(w),
\end{equation*}
the base case follows from \Cref{eq:1wl_proof_F_consistency}. We let $\vec{F}^{(t)}(v)$ denote the representation of the color of node $v$ at iteration $t$. Initially, we set $\vec{F}^{(t)}(v) = \vec{F}'(v)$. Further, we define the matrix $\vec{D}_\text{emb} \in \mathbb{R}^{n \times r}$ such that for the $i$-th row of $\vec{D}_{\text{emb},i} = \embed{deg}'(v_i)$, where $v_i$ is the $i$-th node in a fixed but arbitrary node ordering.
Hence, we can write
\begin{equation*}
   \vec{X}^{(0,1)}(v) = \begin{bmatrix}
      \vec{F}^{(0)}(v) & \vec{0} & \embed{deg}'(v) & \vec{P}'(v)
    \end{bmatrix},
\end{equation*}
for all $v \in V(G)$ and in matrix form
\begin{equation*}
   \vec{X}^{(0,1)} = \begin{bmatrix}
      \vec{F}^{(0)} & \vec{0} & \vec{D}_{\text{emb}} & \vec{P}'
    \end{bmatrix}.
\end{equation*}

Now, assume that the statement holds up to iteration $t-1$. For the induction, we show
\begin{equation}\label{eq:1wl_proof_F}
    C^1_{t}(v) = C^1_{t}(w) \Longleftrightarrow \vec{F}^{(t)}(v) = \vec{F}^{(t)}(w).
\end{equation}
That is, we compute the \kwl{1}-equivalent aggregation using the node color representations $\vec{F}^{(t)}$ and use the remaining columns in $\vec{X}^{(0,1)}$ to keep track of degree and structural embeddings. Clearly, if \Cref{eq:1wl_proof_F} holds for all $t$, then the statement follows. Thereto, we aim to update the vertex representations such that
\begin{equation}\label{eq:1wl_proof_X_update}
    \begin{split}
       \vec{X}^{(t, 1)} &= \begin{bmatrix}
        \vec{F}^{(t)} &
        \vec{0} &
        \vec{D}_{\text{emb}} &
        \vec{P}'
    \end{bmatrix}
    \end{split},
\end{equation}
where, following~\Cref{lemma:grohe_1wl_equiv},
\begin{equation}\label{eq:1wl_f_update_mlp}
    \vec{F}^{(t)} \coloneqq \mathsf{FFN} \Big( \vec{F}^{(t-1)} + 2\vec{A}(G) \vec{F}^{(t-1)} \Big).
\end{equation}
Hence, it remains to show that our graph transformer layer can update vertex representations according to \Cref{eq:1wl_proof_X_update}. To this end, we will require only a single transformer head in each layer. Specifically, we want to compute
\begin{align}\label{eq:1wl_proof_head1}
    h_1(\vec{X}^{(t-1, 1)}) &= \begin{bmatrix}
        \vec{0} & \vec{D}^{-1}\vec{A}(G))\vec{F}^{(t-1)} & \vec{0} & \vec{0}
    \end{bmatrix},
\end{align}
where $\vec{D}^{-1}$ denotes the inverse of the degree matrix and $\vec{I}$ denotes the identity matrix with $n$ rows. Note that since we only have one head, the head dimension $d_v = d$. We begin by re-stating \Cref{eq:1wl_tf_head} with expanded sub-matrices and then derive the instances necessary to obtain the head \Cref{eq:1wl_proof_head1}.
We re-state projection weights $\vec{W}^Q$ and $\vec{W}^K$ with expanded sub-matrices as
\begin{equation*}
    \vec{W}^Q = \begin{bmatrix}
        \mathbf{W}^Q_1 \\
        \mathbf{W}^Q_2 \\
        \mathbf{W}^Q_3 \\
        \mathbf{W}^Q_4
    \end{bmatrix}
\end{equation*}
and
\begin{equation*}
    \vec{W}^K = \begin{bmatrix}
        \mathbf{W}^K_1 \\
        \mathbf{W}^K_2 \\
        \mathbf{W}^K_3 \\
        \mathbf{W}^K_4
    \end{bmatrix},
\end{equation*}
where $\mathbf{W}^Q_1, \mathbf{W}^K_1 \in \mathbb{R}^{p \times d}, \mathbf{W}^Q_2, \mathbf{W}^K_2 \in \mathbb{R}^{p \times d}, \mathbf{W}^Q_3, \mathbf{W}^K_3 \in \mathbb{R}^{r \times d}, \mathbf{W}^Q_4, \mathbf{W}^K_4 \in \mathbb{R}^{s \times d}$. We then define
\begin{equation*}
   \vec{Z}^{(t-1)} \coloneqq \frac{1}{\sqrt{d_k}}(\vec{X}^{(t-1, 1)} \mathbf{W}^Q)(\vec{X}^{(t-1, 1)} \mathbf{W}^K)^T =  \frac{1}{\sqrt{d_k}}(\vec{X}^{(t-1, 1)} \begin{bmatrix}
        \mathbf{W}^Q_1 \\
        \mathbf{W}^Q_2 \\
        \mathbf{W}^Q_3 \\
        \mathbf{W}^Q_4
    \end{bmatrix})(\vec{X}^{(t-1, 1)} \begin{bmatrix}
        \mathbf{W}^K_1 \\
        \mathbf{W}^K_2 \\
        \mathbf{W}^K_3 \\
        \mathbf{W}^K_4
    \end{bmatrix})^T,
\end{equation*}
where
\begin{equation*}
    \vec{X}^{(t-1, 1)} = \begin{bmatrix}
        \vec{F}^{(t-1)} & \vec{0} & \vec{D}_{\text{emb}} & \vec{P}'
    \end{bmatrix},
\end{equation*}
by the induction hypothesis.
By setting $\mathbf{W}^Q_1, \mathbf{W}^Q_2, \mathbf{W}^Q_3, \mathbf{W}^K_1, \mathbf{W}^K_2, \mathbf{W}^K_3$ to zero,
we have
\begin{equation*}
    \vec{Z}^{(t-1)} = \frac{1}{\sqrt{d_k}} (\vec{P}' \mathbf{W}^Q_4)(\vec{P}' \mathbf{W}^K_4)^T.
\end{equation*}
Now, let 
\begin{equation*}
    \Tilde{\vec{P}} \coloneqq \frac{1}{\sqrt{d_k}} (\vec{P}' \mathbf{W}^Q_P)(\vec{P}' \mathbf{W}^K_P)^T,
\end{equation*}
for some weight matrices $\mathbf{W}^Q_P$ and $\mathbf{W}^K_P$.
We know from \Cref{lemma:approx_normalized_adjacency} that, since $\vec{P}'$ is sufficiently adjacency reconstructing, there exists
$\mathbf{W}^Q_P$ and $\mathbf{W}^K_P$ such that by setting $\mathbf{W}^Q_4 = b \cdot \mathbf{W}^Q_P$ and $\mathbf{W}^K_4 = \mathbf{W}^K_P$,
we have
\begin{equation*}
    \vec{Z}^{(t-1)} = b \cdot \Tilde{\vec{P}},
\end{equation*}
where for all $\varepsilon > 0$, there exists a $b > 0$, such that
\begin{equation*}
    \Big\lVert \mathsf{softmax}\Big(\vec{Z}^{(t-1)} \Big) - \vec{D}^{-1}\vec{A}(G) \Big\rVert_F < \varepsilon.
\end{equation*}
Hence, by choosing a large enough $b$, we can approximate the matrix $\vec{D}^{-1}\vec{A}(G)$ arbitrarily close. In the following, for clarity of presentation, we assume $\mathsf{softmax}\Big(\vec{Z}^{(t)} \Big) = \vec{D}^{-1}\vec{A}(G)$ although we only approximate it arbitrarily close. However, by choosing $\varepsilon$ small enough, we can still approximate the matrix $\vec{X}^{(t, 1)}$, see below, arbitrarily close.

We now again expand sub-matrices of $\vec{W}^V$ and express \Cref{eq:1wl_tf_head} as
\begin{equation*}
h_i(\vec{X}^{(t-1, 1)}) = \mathsf{softmax} \Big( \vec{Z}^{(t-1)} \Big) 
    \vec{X}^{(t-1, 1)} \begin{bmatrix}
        \mathbf{W}^V_1 \\
        \mathbf{W}^V_2 \\
        \mathbf{W}^V_3 \\
        \mathbf{W}^V_4
    \end{bmatrix},
\end{equation*}
where $\mathbf{W}^V_1 \in \mathbb{R}^{p \times d}, \mathbf{W}^V_2 \in \mathbb{R}^{p \times d}, \mathbf{W}^V_3\in \mathbb{R}^{r \times d}, \mathbf{W}^V_4 \in \mathbb{R}^{s \times d}$.
By setting
\begin{align*}
    \mathbf{W}^V_1 &= \begin{bmatrix}
        \vec{0} & \vec{I} & \vec{0} & \vec{0}
    \end{bmatrix} \\
    \mathbf{W}^V_2 &= \mathbf{W}^V_3 = \mathbf{W}^V_4 = \vec{0},
\end{align*}
where $\vec{I} \in \mathbb{R}^{p \times p}$, we can approximate
\begin{equation*}
       h_1(\vec{X}^{(t-1, 1)}) = \begin{bmatrix}
        \vec{0} & \vec{D}^{-1}\vec{A}(G)\vec{F}^{(t-1)} & \vec{0} & \vec{0}
    \end{bmatrix}
\end{equation*}
arbitrarily close.
We now conclude our proof as follows.
Recall that the transformer computes the final representation $\vec{X}^{(t, 1)}$ as
\begin{equation*}
\begin{split}
    \vec{X}^{(t, 1)} &= \mathsf{FFN}_\text{final} \Bigg( \vec{X}^{(t-1, 1)} + h_1(\vec{X}^{(t-1, 1)}) \vec{W}^O \Bigg) \\
    &=\mathsf{FFN}_\text{final} \Bigg( \begin{bmatrix}
        \vec{F}^{(t-1)} & \vec{0} &
        \vec{D}_{\text{emb}} &
        \vec{U}
    \end{bmatrix} + \begin{bmatrix}
        \vec{0} & \vec{D}^{-1}\vec{A}(G)\vec{F}^{(t-1)} & \vec{0} & \vec{0}
    \end{bmatrix} \vec{W}^O \Bigg) \\
    &\underset{\vec{W}^O \coloneqq \vec{I}}{=}\mathsf{FFN}_\text{final} \Bigg( \begin{bmatrix}
        \vec{F}^{(t-1)} &
        \vec{D}^{-1}\vec{A}(G)\vec{F}^{(t-1)} &
        \vec{D}_{\text{emb}} &
        \vec{U}
    \end{bmatrix}  \Bigg),
\end{split}
\end{equation*}
for some $\mathsf{FFN}_\text{final}$.
We now show that there exists an $\mathsf{FFN}_\text{final}$ such that
\begin{equation*}
    \vec{X}^{(t, 1)} = \mathsf{FFN}_\text{final}\Bigg( \begin{bmatrix}
        \vec{F}^{(t-1)} &
        \vec{D}^{-1}\vec{A}(G)\vec{F}^{(t-1)} &
        \vec{D}_{\text{emb}} &
        \vec{U}
    \end{bmatrix}  \Bigg) = \begin{bmatrix}
       \vec{F}^{(t)} & \vec{0} &
       \vec{D}_{\text{emb}} &
        \vec{U}
    \end{bmatrix},
\end{equation*}
which then implies the proof statement. To this end, we show that there exists a composition of functions $f_\text{FFN} \circ f_\text{add} \circ f_\text{deg} \circ f_{\times 2}$ such that
\begin{equation*}
   f_\text{FFN} \circ f_\text{add} \circ f_\text{deg} \circ f_{\times 2} \Bigg(\begin{bmatrix}
        \vec{F}^{(t-1)} &
        \vec{D}^{-1}\vec{A}(G)\vec{F}^{(t-1)} &
        \vec{D}_{\text{emb}} &
        \vec{U}
    \end{bmatrix} \Bigg) = \begin{bmatrix}
       \vec{F}^{(t)} & \vec{0} &
       \vec{D}_{\text{emb}} &
        \vec{U}
    \end{bmatrix},
\end{equation*}
where the functions are applied independently to each row. We denote
\begin{equation*}
    \vec{X}^{(t-1)}_\text{upd} \coloneqq \begin{bmatrix}
        \vec{F}^{(t-1)} &
        \vec{D}^{-1}\vec{A}(G)\vec{F}^{(t-1)} &
        \vec{D}_{\text{emb}} &
        \vec{U}
    \end{bmatrix} \in \mathbb{R}^{n \times d}
\end{equation*}
and obtain for node $v \in V(G)$,
\begin{equation*}
    \vec{X}^{(t-1)}_{\text{upd}}(v) = \begin{bmatrix}
        \vec{F}^{(t-1)}(v) &
        \sum_{w \in N(v)} \frac{1}{|N(v)|} \cdot \vec{F}^{(t-1)}(w) &
        \embed{deg}'(v) &
        \vec{P}'(v)
    \end{bmatrix} \in \mathbb{R}^d.
\end{equation*}
We can now define
\begin{equation*}
   \embed{deg}'(v) = \begin{bmatrix}
       |N(v)| & 0
   \end{bmatrix} \in \mathbb{R}^2.
\end{equation*}
Since our domain is compact, there exist choices of $f_\text{FFN}, f_\text{add}, f_\text{deg}, f_{\times 2}$ such that $f_\text{FFN} \circ f_\text{add} \circ f_\text{deg} \circ f_{\times 2}: \mathbb{R}^d \rightarrow \mathbb{R}^d$ is continuous. As a result,
$\mathsf{FFN}_\text{final}$ can approximate  $f_\text{FFN} \circ f_\text{add} \circ f_\text{deg} \circ f_{\times 2}$ arbitrarily close.

Concretely, we define
\begin{align*}
     f_{\times 2}\Bigg( & \begin{bmatrix}
        \vec{F}^{(t-1)}(v) &
        \sum_{w \in N(v)} \frac{1}{|N(v)|} \cdot \vec{F}^{(t-1)}(w) &
        |N(v)| &
        0 &
        \vec{P}'(v)
    \end{bmatrix}  \Bigg)  \\ &= \begin{bmatrix}
        \vec{F}^{(t-1)}(v) &
        2\sum_{w \in N(v)} \frac{1}{|N(v)|} \cdot \vec{F}^{(t-1)}(w) &
        |N(v)| &
        0 &
        \vec{P}'(v)
    \end{bmatrix},
\end{align*}
where $f_{\times 2}$ multiplies the second column with $2$.

Next, we define
\begin{align*}
     f_\text{deg}\Bigg( & \begin{bmatrix}
        \vec{F}^{(t-1)}(v) &
        2\sum_{w \in N(v)} \frac{1}{|N(v)|} \cdot \vec{F}^{(t-1)}(w) &
        |N(v)| &
        0 &
        \vec{P}'(v)        
    \end{bmatrix}  \Bigg)  \\ &= \begin{bmatrix}
        \vec{F}^{(t-1)}(v) &
        2\sum_{w \in N(v)} \frac{1}{|N(v)|} \cdot \vec{F}^{(t-1)}_j \cdot |N(v)| &
        |N(v)| &
        \vec{P}'(v)
    \end{bmatrix} \\ &= \begin{bmatrix}
        \vec{F}^{(t-1)}(v) &
        2\sum_{w \in N(v)} \vec{F}^{(t-1)}(w) &
        |N(v)| &
        0 &
        \vec{P}'(v) 
    \end{bmatrix}
\end{align*}
where $f_\text{deg}$ multiplies the second column with the degree of node $v$ in the third column.

Next, we define
\begin{align*}
     f_\text{add}\Bigg( & \begin{bmatrix}
        \vec{F}^{(t-1)}(v) &
        2\sum_{w \in N(v)} \vec{F}^{(t-1)}(w) &
        |N(v)| &
        0 &
        \vec{P}'(v) 
    \end{bmatrix}  \Bigg)  \\ &= \begin{bmatrix}
        \vec{F}^{(t-1)}(v) + 2\sum_{w \in N(v)} \vec{F}^{(t-1)}(w) &
         \vec{0} &
        |N(v)| &
        0 &
        \vec{P}'(v) 
    \end{bmatrix},
\end{align*}
where we add the second column to the first column and set the elements of the second column to zero.

Finally, we define
\begin{align*}
     f_\text{FFN}\Bigg( & \begin{bmatrix}
        \vec{F}^{(t-1)}(v) + 2\sum_{w \in N(v)} \vec{F}^{(t-1)}(w) &
         \vec{0} &
        |N(v)| &
        0 &
        \vec{P}'(v) 
    \end{bmatrix}  \Bigg)  \\ &=
    \begin{bmatrix}
        \mathsf{FFN} \Big(\vec{F}^{(t-1)}(v) + 2\sum_{w \in N(v)} \vec{F}^{(t-1)}(w) \Big) &
         \vec{0} &
        |N(v)| &
        0 &
        \vec{P}'(v) 
    \end{bmatrix}   \\ &=
    \begin{bmatrix}
        \vec{F}^{(t)}(v) &
         \vec{0} &
        \embed{deg}'(v) &
        \vec{P}'(v) 
    \end{bmatrix},
\end{align*}
where $\mathsf{FFN}$ denotes the FFN in \Cref{eq:1wl_f_update_mlp}, from which the last equality follows.
Clearly, applying $f_\text{FFN} \circ f_\text{add} \circ f_\text{deg} \circ f_{\times 2}$ to each row $i$ of $\vec{X}^{(t-1)}_\text{upd}$ results in
\begin{equation*}
   f_\text{FFN} \circ f_\text{add} \circ f_\text{deg} \circ f_{\times 2} \Big(\vec{X}^{(t-1)}_\text{upd} \Big) = \begin{bmatrix}
       \vec{F}^{(t)} & \vec{0} &
       \vec{D}_{\text{emb}} &
        \vec{P}'(v)
    \end{bmatrix},
\end{equation*}
from which our proof follows.
\end{proof}

\section{Expressivity of higher-order transformers}\label{app:proof_kgt}
Here, we prove \Cref{theorem:k_wl}, \Cref{theorem:delta_k_wl}, and \Cref{theorem:k_s_wl} in \Cref{sec:kwl_pure} in the main paper. To this end, we first construct new higher-order GNNs aligned with the \kwl{k}, $\delta$-\kwl{k}, \localkwl{}, and \kwl{(k,s)}, respectively. We will use these higher-order GNNs as an intermediate step to show the expressivity results for the \kgt{k} and \kgt{(k,s)}.

\subsection{Higher-order GNNs}
Here, we generalize the GNN from \citet{Gro+2021} to higher orders $k$. Higher-order GNNs with the same expressivity have been proposed in prior works by \citet{Azi+2020} and \citet{Mor+2019}. However, our GNNs have a special form that can be computed by a transformer.
To this end, we extend Theorem VIII.4 in \citet{Gro+2021} to the \kwl{k}. Specifically, for each $k > 1$ we devise neural architectures with \kwl{k} expressivity. Afterwards, we show that transformers can simulate these architectures.

Formally, let $S \subset \mathbb{N}$ be a finite subset.
First, we show that multisets over $S$ can be injectively mapped to a value in the closed interval $(0,1)$, which is a variant of Lemma VIII.5 in~\citet{Gro+2021}. Here, we outline a streamlined version of its proof, highlighting the key intuition behind representing multisets as $m$-ary numbers. 
Let $M \subseteq S$ be a multiset with multiplicities $a_1, \dots, a_k$ and distinct $k$ values. We define the \textit{order} of the multiset as $\sum_{i=1}^k a_i$.
We can write such a multiset as
a sequence $x^{(1)}, \dots, x^{(l)}$ where $l$ is the order of the multiset.
Note that the order of the sequence is arbitrary and that for $i \neq j$ it is possible to have $x^{(i)} = x^{(j)}$. We call such a sequence an $M$-sequence of length $l$.
We now prove the following, a slight variation of \citet{Gro+2021}.

\begin{lemma}\label{lemma:multiset_m_ary} 
For a finite $m \in \mathbb{N}$, let $M \subseteq S$ be a multiset of order $m - 1$ and let $x_i \in S$ denote the $i$-th number in a fixed but arbitrary ordering of $S$.
Given a mapping $g \colon S \rightarrow (0,1)$ where
\begin{equation*}
    g(x_i) \coloneqq m^{-i},
\end{equation*}
and an $M$-sequence of length $l$ given by $x^{(1)}, \dots, x^{(l)}$ with positions $i^{(1)}, \dots, i^{(l)}$ in $S$, the sum
\begin{equation*}
 \sum_{j \in [l]} g(x^{(j)}) = \sum_{j \in [l]} m^{-i^{(j)}}
\end{equation*}
is unique for every unique $M$.
\end{lemma}
\begin{proof}
By assumption, let $M \subseteq S$ denote a multiset of order $m - 1$. Further, let $x^{(1)}, \dots, x^{(l)} \in M$ be an $M$-sequence with $i^{(1)}, \dots, i^{(l)}$ in $S$. Given our fixed ordering of the numbers in $S$ we can equivalently write $M = ( (a_1, x_1), \dots, (a_n, x_n) )$, where $a_i$ denotes the multiplicity of $i$-th number in $M$ with position $i$ from our ordering over $S$. Note that for a number $m^{-i}$ there exists a corresponding $m$-ary number written as
\begin{equation*}
    0.0 \ldots \underbrace{1}_{i} \ldots
\end{equation*}
Then the sum,
\begin{align*}
    \sum_{j \in [l]} g(x^{(j)}) &= \sum_{j \in [l]} m^{-i^{(j)}} \\
    &= \sum_{i \in S} a_im^{-i} \in (0,1)
\end{align*}
and in $m$-ary representation
\begin{align*}
   0.a_1 \ldots a_n. 
\end{align*}
Note that $a_i = 0$ if and only if there exists no $j$ such that $i^{(j)} = i$. Since the order of $M$ is $m - 1$, it holds that $a_i < m$. Hence, it follows that the above sum is unique for each unique multiset $M$, implying the result.
\end{proof}

Recall that $S \subseteq \mathbb{N}$ and that we fixed an arbitrary ordering over $S$. Intuitively, we use the finiteness of $S$ to map each number therein to a fixed digit of the numbers in $(0,1)$. The finite $m$ ensures that at each digit, we have sufficient ``bandwidth'' to encode each $a_i$. Now that we have seen how to encode multisets over $S$ as numbers in $(0,1)$, we review some fundamental operations about the $m$-ary numbers defined above. We will refer to decimal numbers $m^{-i}$ as \textit{corresponding} to an $m$-ary number
\begin{equation*}
    0.0 \ldots \underbrace{1}_{i} \ldots,
\end{equation*}
where the $i$-th digit after the decimal point is $1$ and all other digits are $0$, and vice versa.

To begin with, addition between decimal numbers implements \textit{counting} in $m$-ary notation, i.e., 
\begin{equation*}
    m^{-i} + m^{-j} \text{ corresponds to } 0.0\ldots \underbrace{1}_{i} \ldots \underbrace{1}_{j} \ldots,
\end{equation*}
for digit positions $i \neq j$ and
\begin{equation*}
    m^{-i} + m^{-j} \text{ corresponds to } 0.0\ldots \underbrace{2}_{i=j}\ldots,
\end{equation*}
otherwise.
We used counting in the proof of the previous result to represent the multiplicities of a multiset. Next, multiplication between decimal numbers implements \textit{shifting} in $m$-ary notation, i.e.,
\begin{equation*}
    m^{-i} \cdot m^{-j} \text{ corresponds to } 0.0\ldots \underbrace{1}_{i+j}\ldots.
\end{equation*}
Shifting further applies to general decimal numbers in $(0,1)$. Let $x \in (0,1)$ correspond to an $m$-ary number with $l$ digits,
\begin{equation*}
   0.a_1 \ldots a_l.
\end{equation*}
Then,
\begin{equation*}
   m^{-i} \cdot x \text{ corresponds to } 0.0\ldots 0\underbrace{a_1 \ldots a_l}_{i+1, \dots, i+l}.
\end{equation*}

We continue by deriving a neural architecture simulating the \kwl{k}.
\begin{proposition}\label{prop:kwl_gnn}
Let  $G = (V(G), E(G), \ell)$ be a $n$-order (node-)labeled graph. Then for all $t \geq 1$, there exists a function $g^{(t)}$ and scalars $\beta_1, \dots, \beta_k$ with
\begin{equation*}
    g^{(t)}(\vec{u}) \coloneqq \mathsf{FFN} \Big( g^{(t-1)}(\vec{u}) + \sum_{j \in [k]} \beta_j \cdot \sum_{w \in V(G)} g^{(t-1)}(\phi_j(\vec{u}, w)) \Big),
\end{equation*}
where $g^{(0)}(\vec{u})$ is initialized consistent $\ell$ consistent with the atomic type,
such that for all $\vec{u}, \vec{v} \in V(G)^k$
\begin{equation}\label{eq:kwl_color_equiv_g}
    C_t^k(\vec{u}) = C_t^k(\vec{v}) \Longleftrightarrow g^{(t)}(\vec{u})= g^{(t)}(\vec{v}).
\end{equation}
\end{proposition}
\begin{proof}
Before we start, let us recall the relabeling computed by the \kwl{k} for a $k$-tuple $\vec{u}$ as
\begin{equation*}
 \REL \Big( C_t^k(\vec{u}), M_1^{(t)}(\vec{u}), \dots, M_k^{(t)}(\vec{u}) \Big),
\end{equation*}
with
\begin{equation*}
    M_j^{(t)}(\vec{u}) \coloneqq \{\!\! \{ C_{t-1}^k(\phi_j(\vec{u}, w)) \mid w \in V(G) \} \!\!\}.
\end{equation*}
To show our result, we show that there exist scalars $\beta_1, \dots, \beta_k$ such that the $m$-ary representations computed for $C_t^k(\vec{u})$ and $M_1^{(t)},\dots,M_k^{(t)}$ are pairwise unique. To this end, we show that a weighted sum can represent multiset counting in different exponent ranges of $m$-ary numbers in $(0,1)$. We then simply invoke \Cref{lemma:multiset_m_ary} to show that we can map each unique tuple $(C_t^k(\vec{u}), M_1^{(t)}(\vec{u}), \dots, M_k^{(t)}(\vec{u}))$ to a unique number in $(0,1)$. Finally, the $\mathsf{FFN}$ will be responsible for the relabeling.

Each possible color in $C^{k}_{t}$ is a unique number in $[n^k]$ as the maximum possible number of unique colors in $C^{k}_{t}$ is $n^k$. We then fix an arbitrary ordering over the $[n^k]$.

We show the statement via induction over $t$.
Let $m > 0$ such that $m - 1$ is the order of the multiset $M_j^{(0)}(\vec{u})$. Note that this is the same $m$ for each $\vec{u}$.
For a $k$-tuple $\vec{u}$ with initial color $C^k_0(\vec{u})$ at position $i$, we choose $g^{(0)}$ such as to approximate $m^{-i}$ arbitarily close, i.e.,
\begin{equation*}
\big|\big|g^{(0)}(\vec{u}) - m^{-i}\big|\big|_\text{F} < \epsilon, 
\end{equation*}
for an arbitrarily small $\epsilon > 0$. By choosing $\epsilon$ small enough, we have that $g^{(0)}(\vec{u})$ is unique for every unique position $i$ and hence \Cref{eq:kwl_color_equiv_g} holds for all pairs of $k$-tuples for $t=0$. Note that by construction, $i \leq n^k$ and hence, for tuple $\vec{u}$ at position $i$, the $m$-ary number corresponding to $m^{-i}$ is non-zero in at most the first $n^k$ digits.

For the inductive case, we assume that
\begin{equation*}
    C^{k}_{t-1}(\vec{u}) = C^{k}_{t-1}(\vec{v}) \Longleftrightarrow g^{(t-1)}(\vec{u})= g^{(t-1)}(\vec{v}).
\end{equation*}
Further, we assume that for tuple $\vec{u}$ at position $i$,
\begin{equation*}
\big|\big|g^{(t-1)}(\vec{u}) - m^{-i}\big|\big|_\text{F} < \epsilon, 
\end{equation*}
for an arbitrarily small $\epsilon > 0$.
Let now $\vec{u} \in V(G)^k$. We say that the $j$-neighbors of $\vec{u}$ w.r.t. $w$ have indices $l^{(1)}(w), \dots, l^{(k)}(w)$ in our ordering. Then, $\sum_{w \in V(G)} m^{-l^{(j)}(w)}$ is unique for each unique
\begin{equation}
   M_j^{(t)}(\vec{u}) = \{\!\! \{ C_t^k(\phi_j(\vec{u}, w)) \mid w \in V(G) \} \!\!\},
\end{equation}
according to \Cref{lemma:multiset_m_ary}.

We now obtain a unique value for the tuple
\begin{equation*}
\Big(M_1^{(t)}(\vec{u}), \dots, M_k^{(t)}(\vec{u})\Big),
\end{equation*}
by setting 
\begin{equation*}
    \beta_j \coloneqq m^{-(n^k) \cdot j},
\end{equation*}
for $j \in [k]$.
Specifically, let $a_1, \dots, a_{n^k}$ denote the multiplicities of multiset $M_j^{(t)}(\vec{u})$ and let
\begin{equation*}
    0.a_1 \ldots a_{n^k},
\end{equation*}
denote the $m$-ary number that corresponds to the sum
\begin{equation*}
\sum_{w \in V(G)} m^{-l^{(j)}(w)}.
\end{equation*}
Note that since each $m^{-l^{(j)}(w)}$ corresponds to a color under \kwl{k} at iteration $t-1$, all digits after the $n^k$-th digit are zero.
Then, multiplying the above sum with $\beta_j$ results in a shift in $m$-ary notation and, hence, for the $m$-ary number that corresponds to the term
\begin{equation}\label{eq:kwl_shifted_nk_j}
\beta_j \cdot \sum_{w \in V(G)} m^{-l^{(j)}(w)},
\end{equation}
we can write
\begin{equation*}
    0.0\ldots \underbrace{0}_{n^k \cdot j} \, \, \underbrace{a_1 \ldots a_{n^k}}_{n^k \cdot j + 1, \dots, n^k \cdot j + n^k} \, \, \underbrace{0}_{n^k \cdot (j+1) + 1} \ldots 0.
\end{equation*}
As a result, for all $l \neq j$, the non-zero digits of the $m$-ary number that corresponds to
\begin{equation*}
\beta_l \cdot \sum_{w \in V(G)} m^{-l^{(j)}(w)}
\end{equation*}
do not collide with the non-zero digits of the output of \Cref{eq:kwl_shifted_nk_j} and hence, the sum
\begin{equation}\label{eq:kwl_g_just_sum}
\sum_{j \in [k]} \beta_j \cdot \sum_{w \in V(G)} m^{-l^{(j)}(w)},
\end{equation}
is unique for each unique tuple
\begin{equation*}
    \Big(M_1^{(t)}(\vec{u}), \dots, M_k^{(t)}(\vec{u})\Big).
\end{equation*}

Let $\vec{u}_i$ be the $i$-th tuple in our fixed but arbitrary ordering. Then, $g^{(t-1)}(\vec{u})$ approximates the number $m^{-i}$ arbitrarily close.
Note that by the induction hypothesis, the $m$-ary number that corresponds to $m^{-i}$ is non-zero in at most the first $n^k$ digits, as $n^k$ is the maximum possible number of colors attainable under \kwl{k}.
Since the smallest shift in \Cref{eq:kwl_g_just_sum} is by $n^k$ (for $j = 1$) and since the $m$-ary number that corresponds to $m^{-i}$ is non-zero in at most the first $n^k$ digits, the sum of $m^{-i}$ and \Cref{eq:kwl_g_just_sum} have no intersecting non-zero digits. 
As a result,
\begin{equation} \label{eq:kwl_g_without_hash}
 m^{-i} + \sum_{j \in [k]} \beta_j \cdot \sum_{w \in V(G)} m^{-l^{(j)}(w)} 
\end{equation}
is unique for each  tuple
\begin{equation*}
    \Big(C_t^k(\vec{u}), M_1^{(t)}(\vec{u}), \dots, M_k^{(t)}(\vec{u})\Big).
\end{equation*}
Now, by the induction hypothesis, we can approximate each $m^{-i}$ with $g^{(t-1)}(\vec{u}_i)$ arbitrarily close. Further, we can approximate each $m^{-l^{(j)}(w)}$ with $g^{(t-1)}(\phi_j(\vec{u}, w))$ arbitrarily close. As a result, we can approximate \Cref{eq:kwl_g_without_hash} with
\begin{equation*}
  g^{(t-1)}(\vec{u}) + \sum_{j \in [k]} \beta_j \cdot \sum_{w \in V(G)} g^{(t-1)}(\phi_j(\vec{u}, w))  
\end{equation*}
arbitrarily close, i.e.,
\begin{equation*}
  \big|\big|\Big( g^{(t-1)}(\vec{u}) + \sum_{j \in [k]} \beta_j \cdot \sum_{w \in V(G)} g^{(t-1)}(\phi_j(\vec{u}, w))\Big)  - \Big( m^{-i} + \sum_{j \in [k]} \beta_j \cdot \sum_{w \in V(G)} m^{-l^{(j)}(w)  } \Big)\big|\big|_\text{F} < \epsilon,
\end{equation*}
for an arbitrarily small $\epsilon$. Hence, we by choosing $\epsilon$ small enough,
\begin{equation*}
    g^{(t-1)}(\vec{u}) + \sum_{j \in [k]} \beta_j \cdot \sum_{w \in V(G)} g^{(t-1)}(\phi_j(\vec{u}, w))
\end{equation*}
is also unique for each unique tuple
\begin{equation*}
    \Big(C_t^k(\vec{u}), M_1^{(t)}(\vec{u}), \dots, M_k^{(t)}(\vec{u})\Big).
\end{equation*}
Finally, since for finite graphs of order $n$ there exists only a finite number of such tuples, there exists a continuous function mapping the output of \Cref{eq:kwl_g_without_hash} to $m^{-i}$ where $i$ is the position of the color $C^{k}_{t}(\vec{u})$ in $[n^k]$, for all $\vec{u} \in V(G)^k$.
We approximate this function with $\mathsf{FFN}$ arbitrarily close and obtain
\begin{equation*}
||g^{(t)}(\vec{u}) - m^{-i}|| < \epsilon,
\end{equation*}
for an arbitrarily small $\epsilon > 0$.
This concludes the proof.
\end{proof}
We finish with a few Corollaries regarding variants of the \kwl{k}.
\begin{corollary}
Let  $G = (V(G), E(G), \ell)$ be a $n$-order (node-)labeled graph. Then for each $k$-tuple $\vec{u} \coloneqq (u_1, \dots u_k)$ and each $t > 0$, we can equivalently express the coloring of the \dkwl{k} as $C^{\delta, k}_t(\vec{u}) \coloneqq$
\begin{align*}
    \REL\Big( C^{\delta, k}_{t-1}(\vec{u}), &\{\!\! \{ C^{\delta, k}_{t-1}(\phi_1(\vec{u}, w)) \mid w \in {N}(u_1) \} \!\!\},\{\!\! \{ C^{\delta, k}_{t-1}(\phi_1(\vec{u}, w)) \mid w \in V(G) \setminus {N}(u_1) \} \!\!\}, \dots, \\
    &\{\!\! \{ C^{\delta, k}_{t-1}(\phi_k(\vec{u}, w)) \mid w \in {N}(u_k) \} \!\!\}, \{\!\! \{ C^{\delta, k}_{t-1}(\phi_k(\vec{u}, w)) \mid w \in V(G) \setminus {N}(u_k) \} \!\!\} \Big).
\end{align*}
\end{corollary}
With the above statement, we can directly derive a $\delta$-variant of \Cref{prop:kwl_gnn}.
To this end, let $\Delta_j(\vec{u})$ denote the set of vertices adjacent to the $j$-th node in $\vec{u}$.
\begin{corollary}\label{cor:unified_k_gnn}
Let  $G = (V(G), E(G), \ell)$ be a $n$-order (vertex-)labeled graph and assume a vertex feature matrix $\vec{F} \in \mathbb{R}^{n \times d}$ that is consistent with $\ell$. Then for all $t \geq 1$, there exists a function $g^{(t)}$ and scalars $\beta_1, \dots, \beta_{2k}$ with
\begin{equation*}
    g^{(t)}(\vec{u}) \coloneqq \mathsf{FFN} \Big( g^{(t-1)}(\vec{u})  + \sum_{j \in [k]} \Big( \alpha_j \cdot \sum_{w \in \Delta_j(\vec{u})} g^{(t-1)}(\phi_j(\vec{u}, w)) \, + \, \beta_j \cdot \sum_{w \in V(G) \setminus \Delta_j(\vec{u})} g^{(t-1)}(\phi_j(\vec{u}, w) \Big) \Big),
\end{equation*}
such that for all $\vec{u}, \vec{v} \in V(G)^k$
\begin{equation}%
   C^{k,\text{M}}_{t}(\vec{u}) = C^{k,\text{M}}_{t}(\vec{v}) \Longleftrightarrow g^{(t)}(\vec{u})= g^{(t)}(\vec{v}).
\end{equation}
\end{corollary}
The proof is the same as for \Cref{prop:kwl_gnn}. 
Further, we can also recover the \dklwl{k} variant of \Cref{prop:kwl_gnn} by setting $\beta_j = 0$. Lastly, we again recover \Cref{prop:kwl_gnn} by setting $\alpha_j = \beta_j$.

\subsection{Higher-order pure transformers}
Here, we use the results from the previous section to show the expressivity results for \Cref{theorem:k_wl}, \Cref{theorem:delta_k_wl} and \Cref{theorem:k_s_wl} in the main paper.
In fact, we prove a slightly stronger result than \Cref{theorem:k_wl} from which then also \Cref{theorem:delta_k_wl} and \Cref{theorem:k_s_wl} will follow directly. %
\begin{theorem}[Generalization of \Cref{theorem:k_wl} in main paper]
Let $G=(V(G),E(G),\ell)$ be a labeled graph with $n$ nodes and $k \geq 2$ and $\vec{F} \in \mathbb{R}^{n \times d}$ be a node feature matrix consistent with $\ell$. Let $C^k_t$ denote the coloring function of the \kwl{k}, \dkwl{k}, or \dklwl{k} at iteration $t$.
Then for all iterations $t \geq 0$, there exists a parametrization of the \kgt{k} such that
\begin{equation*}
    C^{k}_t(\vec{v}) = C^{k}_t(\vec{w}) \Longleftrightarrow \vec{X}^{(t, k)}(\vec{v}) = \vec{X}^{(t, k)}(\vec{w})
\end{equation*}
for all $k$-tuples $\vec{v}$ and $\vec{w} \in V(G)^k$. If $C^k_t$ is the coloring function of the \kwl{k}, it suffices that the structural embeddings of \kgt{k} are sufficiently node-identifying. Otherwise, we require the structural embeddings of \kgt{k} to be both sufficiently node and adjacency-identifying.
\end{theorem}
\begin{proof}
We begin by recalling \Cref{eq:kwl_token_embeddings} as
\begin{equation*}
    \vec{X}^{(0, k)}(\vec{v}) \coloneqq \big[
        \vec{F}(v_i)
    \big]_{i=1}^k \mathbf{W}^F + \big[
        \vec{P}(v_i)
    \big]_{i=1}^k \mathbf{W}^P + \embed{atp}(\vec{v}),
\end{equation*}
where $\vec{P}$ are structural embeddings and $\embed{atp}$ is a learnable embedding of the atomic type.
We further know from \Cref{lemma:concatenated_kwl_token_embeddings} that we can write
\begin{equation}\label{eq:concatenated_kwl_token_embeddings}
    \vec{X}^{(0, k)}(\vec{v}) = \begin{bmatrix}
     \vec{F}'(v_1) \,\,\,\, \dots \,\,\,\, \vec{F}'(v_k) & 
     \vec{0} &
     \embed{deg}'(v_1) & \dots & \embed{deg}'(v_k) & \vec{P}'(v_1) & \dots & \vec{P}'(v_k) & \embed{atp}'(\vec{v})
  \end{bmatrix},
\end{equation}
where it holds for every $v, w \in V(G)$,
\begin{equation*}
  \vec{F}(v) = \vec{F}(w) \Longleftrightarrow \vec{F}'(v) = \vec{F}'(w)
\end{equation*}
and
\begin{equation*}
  \embed{deg}(v) = \embed{deg}(w) \Longleftrightarrow \embed{deg}'(v) = \embed{deg}'(w)
\end{equation*}
and
\begin{equation*}
  \vec{P}(v) = \vec{P}(w) \Longleftrightarrow \vec{P}'(v) = \vec{P}'(w)
\end{equation*}
where $\vec{P}'(v)$ is sufficiently node and adjacency-identifying
and for every $\vec{v}, \vec{w} \in V(G)^k$
\begin{equation*}
  \embed{atp}(\vec{v}) = \embed{atp}(\vec{w}) \Longleftrightarrow \embed{atp}'(\vec{v}) = \embed{atp}'(\vec{w}).
\end{equation*}
Further, $\vec{F}'(v) \in \mathbb{R}^{p}$, $\vec{0} \in \mathbb{R}^{k^2 \cdot p}$, $\embed{deg}' \colon \mathbb{N} \rightarrow \mathbb{R}^{r}$, $\embed{atp}'\colon \mathbb{N} \rightarrow \mathbb{R}^{o}$ and $\vec{P}' \in \mathbb{R}^{s}$, for some choice of $p, r, s, o$ where $d = k^2 \cdot p + k \cdot p +  k\cdot r + k\cdot s + o$, as specified in \Cref{lemma:concatenated_kwl_token_embeddings}.
We will use this parameterization for $\vec{X}^{(0,k)}$ throughout the rest of the proof.

We prove our statement by induction over iteration $t$.
For the base case, notice that the initial color of a tuple $\vec{v}$ depends on the atomic type and the node labeling. In \Cref{eq:concatenated_kwl_token_embeddings}, we encode the atomic type with $\embed{atp}'(\vec{v})$ and the node labels by concatenating the features $\vec{F}'(v_1), \dots, \vec{F}'(v_k)$ of the $k$ nodes $v_1, \dots, v_k$ in $\vec{v}$. The concatenation of both node labels and atomic type is clearly injective, and so 
\begin{equation*}
    C^{k}_0(\vec{v}) = C^{k}_0(\vec{w}) \Longleftrightarrow \vec{X}^{(0, k)}(\vec{v}) = \vec{X}^{(0, k)}(\vec{w})
\end{equation*}
for any two $\vec{v}, \vec{w} \in V(G)^k$.

Before continuing with the induction, we define some additional notation. Throughout the induction, we will denote the color representation of a tuple $\vec{v}$ at iteration $t$ as $\vec{F}^{(t)}(\vec{v})$. Further, we initially define $\vec{F}^{(0)}(\vec{v}) = \big[ \vec{F}'(v_i) \big]_{i=1}^k$, $\embed{deg}'(\vec{v}) = \big[ \embed{deg}'(v_i) \big]_{i=1}^k$ and $\vec{P}'(\vec{v}) = \big[ \vec{P}'(v_i) \big]_{i=1}^k$. Hence, we can rewrite
\begin{equation*}
    \vec{X}^{(0, k)}(\vec{v}) = \begin{bmatrix}
        \vec{F}^{(0)}(\vec{v}) & 
        \vec{0} &
        \embed{deg}'(\vec{v}) & \vec{P}'(\vec{v}) & \embed{atp}'(\vec{v})
    \end{bmatrix}.
\end{equation*}

For the induction step,
we show
\begin{equation}\label{eq:kwl_proof_inductive_case}
    C^k_t(\vec{v}) = C^k_t(\vec{w}) \Longleftrightarrow \vec{X}^{(t, k)}(\vec{v}) = \vec{X}^{(t, k)}(\vec{w}),
\end{equation}
that is we compute the \kwl{k}-, \dklwl{k}, or \dkwl{k}-equivalent aggregation.
Clearly, if \Cref{eq:kwl_proof_inductive_case} holds for all $t$, then the proof statement follows.
Thereto, we show that the standard transformer updates the tuple representation of tuple $\vec{u}_i = (u_1, \dots, u_k)$, following \Cref{cor:unified_k_gnn}, as
\begin{equation}\label{eq:kwl_proof_X_update}
    \vec{X}^{(t, k)}(\vec{u}_i) = \begin{bmatrix}
    \mathsf{FFN}\Big(\vec{F}^{(t-1)}(\vec{u}_i) + H(\vec{X}^{(t-1, k)}(\vec{u}_i))\Big) & 
    \vec{0} &
    \embed{deg}'(\vec{u}_i) &
    \vec{P}'(\vec{u}_i) & \embed{atp}'(\vec{u}_i)  
    \end{bmatrix},
\end{equation}
where
\begin{equation*}
   H(\vec{X}^{(t-1, k)}(\vec{u}_i)) \coloneqq \sum_{j \in [k]} \Big( \alpha_j \cdot \sum_{w \in \Delta_j(\vec{u}_i)}  \vec{F}^{(t-1)}(\phi_j(\vec{u}_i, w)) + \beta_j \cdot \sum_{w \in V(G) \setminus\Delta_j(\vec{u}_i)}  \vec{F}^{(t-1)}(\phi_j(\vec{u}_i, w)) \Big),
\end{equation*}
for $\alpha_1, \dots, \alpha_k, \beta_1, \dots, \beta_k \in \mathbb{R}$, for $j\in[k]$ that we pick such that $H(\vec{X}^{(t-1, k)}(\vec{u}_i))$ is equivalent to the neural architecture in \Cref{cor:unified_k_gnn}.
Note that we obtain the \kwl{k} for $\alpha_j = \beta_j$ for all $j \in [k]$. Then,
\begin{equation*}
   H(\vec{X}^{(t-1, k)}(\vec{u}_i)) = \sum_{j \in [k]} \Big( \beta_j \cdot \sum_{w \in V(G)}  \vec{F}^{(t-1)}(\phi_j(\vec{u}_i, w)) \Big).
\end{equation*}
Further, we obtain \dklwl{k} for $\beta_j = 0$ for all $j \in [k]$.
Then, 
\begin{equation*}
   H(\vec{X}^{(t-1, k)}(\vec{u}_i)) = \sum_{j \in [k]} \Big( \alpha_j \cdot \sum_{w \in \Delta_j(\vec{u}_i)}  \vec{F}^{(t-1)}(\phi_j(\vec{u}_i, w)) \Big).
\end{equation*}
Hence, it remains to show that the standard transformer layer can update tuple representations according to \Cref{eq:kwl_proof_X_update}. 
To this end, we will require $2k$ transformer heads $h^1_1, \dots h^1_k, h^2_1, \dots h^2_k$ in each layer. 
Specifically, in the first $k$ heads, we want to compute
\begin{equation}
    h^1_j(\vec{X}^{(t-1, k)}(\vec{u}_i)) = 
        \alpha_j \cdot \sum_{w \in \Delta_j(\vec{u}_i)}  \vec{F}^{(t-1)}(\phi_j(\vec{u}_i, w)) \label{eq:kwl_proof_head_1_j}.
\end{equation}
In the remaining $k$ heads, we want to compute
\begin{equation}
    h^2_j(\vec{X}^{(t-1, k)}(\vec{u}_i)) = 
        \beta_j \cdot \sum_{w \in V(G) \setminus\Delta_j(\vec{u}_i)}  \vec{F}^{(t-1)}(\phi_j(\vec{u}_i, w))  \label{eq:kwl_proof_head_2_j}.
\end{equation}

In both of the above cases, for a head $h^\gamma_j$, $\gamma \in \{-1, 1\}$ denotes the type of head, and $j \in [k]$ denotes the $j$-neighbors the head aggregates over. For the head dimension, we set $d_v = d$.

For each $j$, recall the definition of the standard transformer head at tuple $\vec{u}_i$ at position $i$ as
\begin{equation*}
h^\gamma_j(\vec{X}^{(t-1, k)}(\vec{u}_i)) = \mathsf{softmax} \Big( \begin{bmatrix}
    \vec{Z}_{i1}^{(t-1, \gamma)} & \dots & \vec{Z}_{in^k}^{(t-1, \gamma)}
\end{bmatrix} \Big) 
    \begin{bmatrix}
       \vec{X}^{t-1}(\vec{u}_1) & \dots & \vec{X}^{t-1}(\vec{u}_{n^k}) 
    \end{bmatrix}^T \mathbf{W}^V,
\end{equation*}
where
\begin{equation}\label{eq:kwl_proof_Zstar}
    \vec{Z}_{il}^{(t-1, \gamma)} \coloneqq \frac{1}{\sqrt{d_k}} (\vec{X}^{(t-1, k)} (\vec{u}_i)\vec{W}^{Q,\gamma})(\vec{X}^{(t-1, k)} (\vec{u}_l)\vec{W}^{K,\gamma})^T \in \mathbb{R}
\end{equation}
is the unnormalized attention score between tuples $\vec{u}_i$ and $\vec{u}_l$, with
\begin{equation*}
    \vec{X}^{(t-1, k)}(\vec{u}_i) = \begin{bmatrix}
       \vec{F}^{(t-1)}(\vec{u}_i) &
       \vec{0} & 
       \embed{deg}'(\vec{u}_i) & 
       \vec{P}'(\vec{u}_i) & \embed{atp}'(\vec{u}_i)
    \end{bmatrix}
\end{equation*}
for all $i \in [n^k]$ and
\begin{align*}
   \vec{W}^{Q, \gamma} &= \begin{bmatrix}
      \vec{W}^Q_F \\
      \vec{W}^Q_Z \\
      \vec{W}^Q_D \\
      \vec{W}^{Q, \gamma}_P \\
      \vec{W}^Q_A
   \end{bmatrix} \\
   \vec{W}^{K, \gamma} &= \begin{bmatrix}
      \vec{W}^K_F \\
      \vec{W}^K_Z \\
      \vec{W}^K_D \\
      \vec{W}^{K, \gamma}_P \\
      \vec{W}^K_A \\
   \end{bmatrix} \\
   \vec{W}^V &= \begin{bmatrix}
      \vec{W}^V_F \\
      \vec{W}^V_Z \\
      \vec{W}^V_D \\
      \vec{W}^V_P \\
      \vec{W}^V_A \\
   \end{bmatrix},
\end{align*}
where 
\begin{align*}
   \vec{F}^{(t-1)}(\vec{u}_i) &\text{ is projected by }  \vec{W}^Q_F, \vec{W}^K_F, \vec{W}^V_F, \\
   \vec{0} &\text{ is projected by }  \vec{W}^Q_Z, \vec{W}^K_Z, \vec{W}^V_Z, \\
   \embed{deg}'(\vec{u}_i) &\text{ is projected by }  \vec{W}^Q_D, \vec{W}^K_D, \vec{W}^V_D, \\
   \vec{P}'(\vec{u}_i) &\text{ is projected by }  \vec{W}^{Q,\gamma}_P, \vec{W}^{K,\gamma}_P, \vec{W}^V_P, \\
   \embed{atp}'(\vec{u}_i) &\text{ is projected by }  \vec{W}^Q_A, \vec{W}^K_A, \vec{W}^V_A. \\
\end{align*}
Note that only sub-matrices $\vec{W}^{Q, \gamma}_P$ and $\vec{W}^{K, \gamma}_P$ are different for different $\gamma$.
We now specify projection matrices $\vec{W}^{Q,\gamma}$, $\vec{W}^{K,\gamma}$ and $\vec{W}^V$ in a way that allows the attention head $h^\gamma_j$ to dynamically recover the $j$-neighborhood adjacency as well as the adjacency between $j$-neighbors in the attention matrix. To this end, in heads $h^1_j$ and $h^2_j$ we set
\begin{align*}
\vec{W}^Q_F = \vec{W}^K_F &= \vec{0} \\
\vec{W}^Q_Z = \vec{W}^K_Z = \vec{W}^V_Z &= \vec{0} \\
\vec{W}^V_P &= \vec{0} \\
\vec{W}^Q_D = \vec{W}^K_D = \vec{W}^V_D &= \vec{0} \\
\vec{W}^Q_A = \vec{W}^K_A = \vec{W}^V_A &= \vec{0}.
\end{align*}
The remaining non-zero sub-matrices are $\vec{W}^V_F, \vec{W}^{Q,\gamma}_P, \vec{W}^{K,\gamma}_P$, which we will define next.

We begin by defining $\vec{W}^{Q,\gamma}_P$ and  $\vec{W}^{K,\gamma}_P$. Specifically, we want to choose $\vec{W}^{Q,\gamma}_P$ and $\vec{W}^{K,\gamma}_P$ such that the attention matrix in head $h^\gamma_j$ can approximate the generalized adjacency matrix $\vec{A}^{(k,j,\gamma)}$ in \Cref{eq:generalized_adjacency_matrix}.
To this end, we simply invoke \Cref{lemma:approx_generalized_adjacency}, guaranteeing that there exists $\vec{W}^{Q,\gamma}_P$ and $\vec{W}^{K,\gamma}_P$ such that
that for each $\epsilon > 0$ and each $\gamma \in \{-1, 1\}$,
\begin{equation*}
    \Big\lVert \mathsf{softmax} \Big( \begin{bmatrix}
    \vec{Z}_{i1}^{(t-1, \gamma)} & \dots & \vec{Z}_{in^k}^{(t-1, \gamma)}
\end{bmatrix} \Big) -  \Tilde{\vec{A}}^{(k,j,\gamma)}_i \Big\rVert_F < \varepsilon,
\end{equation*}
i.e., we can approximate the generalized adjacency matrix arbitrarily close for each $k > 1, j \in [k]$ and $\gamma \in \{-1, 1\}$. In the following, for clarity of presentation, we assume
\begin{equation*}
\mathsf{softmax} \Big( \begin{bmatrix}
    \vec{Z}_{i1}^{(t-1, \gamma)} & \dots & \vec{Z}_{in^k}^{(t-1,\gamma)}
\end{bmatrix} \Big) = \Tilde{\vec{A}}^{(k,j,\gamma)}_i
\end{equation*}
although we only approximate it arbitrarily close. However, by choosing $\varepsilon$ small enough, we can still approximate the matrix $\vec{X}^{(t, k)}$, see below, arbitrarily close.
We now set
\begin{equation*}
    \vec{W}^V = \begin{bmatrix}
        \vec{I}_{n \times kp} & \vec{0}
    \end{bmatrix},
\end{equation*} 
where $\vec{I}_{n \times kp}$ denotes the first $n$ rows of the $kp$-by-$kp$ identity matrix if $kp > n$ or else the first $kp$ columns of the $n$-by-$n$ identity matrix and  $\vec{0} \in \mathbb{R}^{n \times d - kp}$ is an all-zero matrix.
The above yields
\begin{equation*}
\begin{split}
    h^\gamma_j(\vec{X}^{(t-1, k)}(\vec{u}_i)) &=
    \frac{1}{d^\gamma_{ij}} \begin{bmatrix}
       \Tilde{\vec{A}}^{(k,j,\gamma)}_{i1} & \dots & \Tilde{\vec{A}}^{(k,j,\gamma)}_{in^k}
   \end{bmatrix} \cdot \begin{bmatrix}
       \vec{F}^{(t-1)}(\vec{u}_1) \\
       \vdots \\
       \vec{F}^{(t-1)}(\vec{u}_{n^k})
   \end{bmatrix} \\ \\
    &= \frac{1}{d^\gamma_{ij}} \cdot \begin{cases}
         \sum_{w \in \Delta_j(\vec{u}_i)} \vec{F}^{(t-1)}(\phi_j(\vec{u}_i, w)) & \gamma = 1 \\ \\
        \sum_{w \in V(G) \setminus \Delta_j(\vec{u}_i)} \vec{F}^{(t-1)}(\phi_j(\vec{u}_i, w)) & \gamma = -1 \\
    \end{cases},
\end{split}
\end{equation*}
where
\begin{equation*}
   d^\gamma_{ij} = \begin{cases}
        d_{ij} & \gamma = 1 \\
        n - d_{ij} & \gamma = -1
    \end{cases}
\end{equation*}
with $d_{ij}$ the degree of node $u_j$ in $k$-tuple $\vec{u}_i = (u_1, \dots, u_k)$.

We now conclude our proof as follows. Recall that the standard transformer layer computes the final representation $\vec{X}^{(t, k)}(\vec{u}_i)$ as
\begin{equation*}
\vec{X}^{(t, k)}(\vec{u}_i) = \mathsf{FFN} \Bigg( \vec{X}^{(t-1, k)}(\vec{u}_i) + \begin{bmatrix}
  h^1_1(\vec{X}^{(t-1, k)})(\vec{u}_i) \\ \vdots \\ h^1_k(\vec{X}^{(t-1, k)})(\vec{u}_i) \\ h^2_1(\vec{X}^{(t-1, k)})(\vec{u}_i)\\ \vdots \\ h^2_k(\vec{X}^{(t-1, k)})(\vec{u}_i) 
\end{bmatrix}^T \vec{W}^O \Bigg).
\end{equation*}
We can now satisfy the requirements in \Cref{eq:kwl_proof_head_1_j} and \Cref{eq:kwl_proof_head_2_j} as follows.

Recall that for the first $k$ heads we set $\gamma = 1$, in the remaining $k$ heads we set $\gamma = -1$. Further, since we have $2k$ heads, $\vec{W}^O \in \mathbb{R}^{2k \cdot d_v \times d}$.
We set
\begin{equation*}
    \vec{W}^O_{ij} = \begin{cases}
        \alpha_j & kp < i = j \leq k+kp \\
        \beta_j & k+kp < i = j \leq 2k+kp \\
        0 & \text{ else},
    \end{cases}
\end{equation*}
i.e., we leave the first $kp$ diagonal elements zero and then fill the next $2k$ diagonal elements of $\vec{W}^O$ with the $\alpha_j$ and $\beta_j$, respectively. All other elements are $0$.
We now obtain
\begin{align*}
   \begin{bmatrix}
  h^1_1(\vec{X}^{(t-1, k)})(\vec{u}_i) \\ \vdots \\ h^1_k(\vec{X}^{(t-1, k)})(\vec{u}_i) \\ h^2_1(\vec{X}^{(t-1, k)})(\vec{u}_i)\\ \vdots \\ h^2_k(\vec{X}^{(t-1, k)})(\vec{u}_i) 
\end{bmatrix}^T \vec{W}^O &=
   \begin{bmatrix}
  \frac{\alpha_1}{d_{i1}}\sum_{w \in \Delta_1(\vec{u}_i)} \vec{F}^{(t-1)}(\phi_1(\vec{u}_i, w))  \\ \vdots \\ \frac{\alpha_k}{d_{ik}}\sum_{w \in \Delta_k(\vec{u}_i)} \vec{F}^{(t-1)}(\phi_k(\vec{u}_i, w))  \\ \frac{\beta_1}{n - d_{i1}}\sum_{w \in V(G) \setminus \Delta_1(\vec{u}_i)} \vec{F}^{(t-1)}(\phi_1(\vec{u}_i, w))\\ \vdots \\ \frac{\beta_k}{n - d_{ik}}\sum_{w \in V(G) \setminus \Delta_k(\vec{u}_i)} \vec{F}^{(t-1)}(\phi_k(\vec{u}_i, w))
\end{bmatrix}^T \vec{W}^O \\ \\
&= \begin{bmatrix}
  \vec{0} \\
  \frac{\alpha_1}{d_{i1}}\sum_{w \in \Delta_1(\vec{u}_i)} \vec{F}^{(t-1)}(\phi_1(\vec{u}_i, w))  \\ \vdots \\ \frac{\alpha_k}{d_{ik}}\sum_{w \in \Delta_k(\vec{u}_i)} \vec{F}^{(t-1)}(\phi_k(\vec{u}_i, w))  \\ \frac{\beta_1}{n - d_{i1}}\sum_{w \in V(G) \setminus \Delta_1(\vec{u}_i)} \vec{F}^{(t-1)}(\phi_1(\vec{u}_i, w))\\ \vdots \\ \frac{\beta_k}{n - d_{ik}}\sum_{w \in V(G) \setminus \Delta_k(\vec{u}_i)} \vec{F}^{(t-1)}(\phi_k(\vec{u}_i, w)) \\
  \vec{0}
\end{bmatrix}^T,
\end{align*}
where the first zero vector is in $\mathbb{R}^{kp}$ and the second zero vector is in $\mathbb{R}^{d - k(r + s) + o}$.
We now define the vector
\begin{equation*}
\begin{split}
   \Tilde{\vec{X}}(\vec{u}_i) &=\begin{bmatrix}
        \vec{F}^{(t-1)}(\vec{u}_i) \\ 
        \vec{0} \\
        \embed{deg}'(\vec{u}_i) \\
        \vec{P}'(\vec{u}_i) \\
        \embed{atp}'(\vec{u}_i)
    \end{bmatrix}^T + \begin{bmatrix}
  \vec{0} \\
  \frac{\alpha_1}{d_{i1}}\sum_{w \in \Delta_1(\vec{u}_i)} \vec{F}^{(t-1)}(\phi_1(\vec{u}_i, w))  \\ \vdots \\ \frac{\alpha_k}{d_{ik}}\sum_{w \in \Delta_k(\vec{u}_i)} \vec{F}^{(t-1)}(\phi_k(\vec{u}_i, w))  \\ \frac{\beta_1}{n - d_{i1}}\sum_{w \in V(G) \setminus \Delta_1(\vec{u}_i)} \vec{F}^{(t-1)}(\phi_1(\vec{u}_i, w))\\ \vdots \\ \frac{\beta_k}{n - d_{ik}}\sum_{w \in V(G) \setminus \Delta_k(\vec{u}_i)} \vec{F}^{(t-1)}(\phi_k(\vec{u}_i, w)) \\
  \vec{0}
\end{bmatrix}^T \\
&= \begin{bmatrix}
  \vec{F}^{(t-1)}(\vec{u}_i) \\
  \frac{\alpha_1}{d_{i1}}\sum_{w \in \Delta_1(\vec{u}_i)} \vec{F}^{(t-1)}(\phi_1(\vec{u}_i, w))  \\ \vdots \\ \frac{\alpha_k}{d_{ik}}\sum_{w \in \Delta_k(\vec{u}_i)} \vec{F}^{(t-1)}(\phi_k(\vec{u}_i, w))  \\ \frac{\beta_1}{n - d_{i1}}\sum_{w \in V(G) \setminus \Delta_1(\vec{u}_i)} \vec{F}^{(t-1)}(\phi_1(\vec{u}_i, w))\\ \vdots \\ \frac{\beta_k}{n - d_{ik}}\sum_{w \in V(G) \setminus \Delta_k(\vec{u}_i)} \vec{F}^{(t-1)}(\phi_k(\vec{u}_i, w)) \\
  \embed{deg}'(\vec{u}_i) \\
        \vec{P}'(\vec{u}_i) \\
        \embed{atp}'(\vec{u}_i)
\end{bmatrix}^T \in \mathbb{R}^d.
\end{split}
\end{equation*}
 To simplify terms, we additionally define
\begin{equation*}
    \Tilde{\vec{F}}^\alpha(\vec{u}_i) \coloneqq \begin{bmatrix}
        \frac{\alpha_1}{d_{i1}}\sum_{w \in \Delta_1(\vec{u}_i)} \vec{F}^{(t-1)}(\phi_1(\vec{u}_i, w))  \\ \vdots \\ \frac{\alpha_k}{d_{ik}}\sum_{w \in \Delta_k(\vec{u}_i)} \vec{F}^{(t-1)}(\phi_k(\vec{u}_i, w))
    \end{bmatrix}^T
\end{equation*}
and
\begin{equation*}
    \Tilde{\vec{F}}^\beta(\vec{u}_i) \coloneqq \begin{bmatrix}
        \frac{\beta_1}{n - d_{i1}}\sum_{w \in V(G) \setminus \Delta_1(\vec{u}_i)} \vec{F}^{(t-1)}(\phi_1(\vec{u}_i, w))\\ \vdots \\ \frac{\beta_k}{n - d_{ik}}\sum_{w \in V(G) \setminus \Delta_k(\vec{u}_i)} \vec{F}^{(t-1)}(\phi_k(\vec{u}_i, w))
    \end{bmatrix}^T
\end{equation*}
and consequently, we can write
\begin{equation*}
    \Tilde{\vec{X}}(\vec{u}_i) = \begin{bmatrix}
  \vec{F}^{(t-1)}(\vec{u}_i) \\
   \Tilde{\vec{F}}^\alpha(\vec{u}_i) \\ 
   \Tilde{\vec{F}}^\beta(\vec{u}_i)\\
  \embed{deg}'(\vec{u}_i) \\
        \vec{P}'(\vec{u}_i) \\
        \embed{atp}'(\vec{u}_i)
\end{bmatrix}^T
\end{equation*}
We additionally define
\begin{equation*}
   \embed{deg}'(u_j) = \begin{bmatrix}
       d_{ij} & n - d_{ij}
   \end{bmatrix} \in \mathbb{R}^2,
\end{equation*}
where $d_{ij}$ is again the degree of node $u_j$ in $k$-tuple $\vec{u}_i = (u_1, \dots, u_k)$.
Note that $\Tilde{\vec{X}}(\vec{u}_i)$ represents all the information we feed into the final $\mathsf{FFN}$.
Specifically, we obtain the updated representation of tuple $\vec{u}_i$ as
\begin{equation*}
   \vec{X}^{(t, k)}(\vec{u}_i) = \mathsf{FFN}\Big(\Tilde{\vec{X}}(\vec{u}_i)\Big).
\end{equation*}
We now show that there exists an $\mathsf{FFN}$ such that
\begin{align*}
    \vec{X}^{(t, k)} &= \mathsf{FFN}\Big(\Tilde{\vec{X}}(\vec{u}_i)\Big) \\
&= \begin{bmatrix}
       \vec{F}^{(t)}(\vec{u}_i) & \vec{0} &
       \embed{deg}'(\vec{u}_i) &
        \vec{P}'(\vec{u}_i) &
        \embed{atp}'(\vec{u}_i)
    \end{bmatrix},
\end{align*}
which then implies the proof statement. It is worth to pause at this point and remind ourselves what each element in $\Tilde{\vec{X}}(\vec{u}_i)$ represents. To this end, we use \texttt{PyTorch}-like array slicing, i.e., for a vector $\vec{x}$, $\vec{x}[a:b]$ corresponds to the sub-vector of length $b - a$ for $b > a$ containing the $(a+1)$-st to $b$-th element of $\vec{x}$. E.g., for a vector $\vec{x} = (x_1, x_2, x_3, x_4, x_5)^T$, $\vec{x}[1:4] = (x_2, x_3, x_4)^T$. Now, we interpret $\Tilde{\vec{X}}(\vec{u}_i)$ by its sub-vectors. Concretely, 
\begin{enumerate}
    \item $\Tilde{\vec{X}}(\vec{u}_i)[0:kp] \in \mathbb{R}^{kp}$ corresponds to the previous color representation $\vec{F}^{(t-1)}(\vec{u}_i)$ of tuple $\vec{u}_i$.
    \item $\Tilde{\vec{F}}^{\alpha}(\vec{u}_i)[(j-1)kp:jkp] \in \mathbb{R}^{kp}$ corresponds to the degree-normalized sum over adjacent $j$-neighbors $\frac{\alpha_j}{d_j}\sum_{w \in \Delta_j(\vec{u}_i)} \vec{F}^{(t-1)}(\phi_j(\vec{u}_i, w))$.
    \item $\Tilde{\vec{F}}^{\beta}(\vec{u}_i)[(j-1)kp:jkp] \in \mathbb{R}^{kp}$ corresponds to the degree-normalized sum over non-adjacent $j$-neighbors $\frac{\beta_j}{n - d_j}\sum_{w \in \Delta_j(\vec{u}_i)} \vec{F}^{(t-1)}(\phi_j(\vec{u}_i, w))$.
    \item $\embed{deg}'(\vec{u}_i)[2(j-1)] = d_{ij}$, where $d_{ij}$ is the degree of node $u_j$ in the $k$-tuple $\vec{u}_i = (u_1, \dots, u_k)$.
    \item $\embed{deg}'(\vec{u}_i)[2(j-1) + 1] = n - d_{ij}$, where $d_{ij}$ is the degree of node $u_j$ in the $k$-tuple $\vec{u}_i = (u_1, \dots, u_k)$.
\end{enumerate}

To this end, we show that there exists a sequence of functions $f_\text{FFN} \circ f_\text{add} \circ f_\text{deg}$ such that
\begin{align*}
    \vec{X}^{(t, k)} &= f_\text{FFN} \circ f_\text{add} \circ f_\text{deg}\Big(\Tilde{\vec{X}}(\vec{u}_i) \Big) \\
&= \begin{bmatrix}
       \vec{F}^{(t)}(\vec{u}_i) & \vec{0} &
       \embed{deg}'(\vec{u}_i) &
        \vec{P}'(\vec{u}_i) &
        \embed{atp}'(\vec{u}_i)
    \end{bmatrix},
\end{align*}
where the functions are applied independently to each row.

Since our domain is compact, there exist choices of $f_\text{FFN}, f_\text{add}, f_\text{deg}$ such that $f_\text{FFN} \circ f_\text{add} \circ f_\text{deg}: \mathbb{R}^d \rightarrow \mathbb{R}^d$, in which case $f_\text{FFN} \circ f_\text{add} \circ f_\text{deg}$ is continuous. As a result,
$\mathsf{FFN}$ can approximate  $f_\text{FFN} \circ f_\text{add} \circ f_\text{deg}$ arbitrarily close.

Concretely, we define
\begin{equation*}
     f_\text{deg}\Big( \Tilde{\vec{X}}(\vec{u}_i) \Big) = \begin{bmatrix}
  \vec{F}^{(t-1)}(\vec{u}_i) \\
  d_{i1} \cdot \frac{\alpha_1}{d_{i1}}\sum_{w \in \Delta_1(\vec{u}_i)} \vec{F}^{(t-1)}(\phi_1(\vec{u}_i, w))  \\ \vdots \\ d_{ik} \cdot \frac{\alpha_k}{d_{ik}}\sum_{w \in \Delta_k(\vec{u}_i)} \vec{F}^{(t-1)}(\phi_k(\vec{u}_i, w))  \\ (n - d_{i1}) \cdot  \frac{\beta_1}{n - d_{i1}}\sum_{w \in V(G) \setminus \Delta_1(\vec{u}_i)} \vec{F}^{(t-1)}(\phi_1(\vec{u}_i, w))\\ \vdots \\ (n - d_{ik}) \cdot \frac{\beta_k}{n - d_{ik}}\sum_{w \in V(G) \setminus \Delta_k(\vec{u}_i)} \vec{F}^{(t-1)}(\phi_k(\vec{u}_i, w)) \\
  \embed{deg}'(\vec{u}_i) \\
        \vec{P}'(\vec{u}_i) \\
        \embed{atp}'(\vec{u}_i)
\end{bmatrix}^T,
\end{equation*}
where for each $j \in [k]$, $f_\text{deg}$ multiplies 
\begin{enumerate}
    \item $\Tilde{\vec{F}}^{\alpha}(\vec{u}_i)[(j-1)kp:jkp]$ with $\embed{deg}'(\vec{u}_i)[2(j-1)]$.
    \item $\Tilde{\vec{F}}^{\beta}(\vec{u}_i)[(j-1)kp:jkp]$ with $\embed{deg}'(\vec{u}_i)[2(j-1) + 1]$.
\end{enumerate}
We then define
\begin{equation*}
    \vec{F}^\alpha(\vec{u}_i) \coloneqq \begin{bmatrix}
        \alpha_1 \cdot \sum_{w \in \Delta_1(\vec{u}_i)} \vec{F}^{(t-1)}(\phi_1(\vec{u}_i, w))  \\ \vdots \\ \alpha_k \cdot \sum_{w \in \Delta_k(\vec{u}_i)} \vec{F}^{(t-1)}(\phi_k(\vec{u}_i, w))
    \end{bmatrix}^T
\end{equation*}
and
\begin{equation*}
    \vec{F}^\beta(\vec{u}_i) \coloneqq \begin{bmatrix}
        \beta_1 \cdot \sum_{w \in V(G) \setminus \Delta_1(\vec{u}_i)} \vec{F}^{(t-1)}(\phi_1(\vec{u}_i, w))\\ \vdots \\ \beta_k \cdot \sum_{w \in V(G) \setminus \Delta_k(\vec{u}_i)} \vec{F}^{(t-1)}(\phi_k(\vec{u}_i, w))
    \end{bmatrix}^T
\end{equation*}
and consequently, we can write
\begin{equation*}
     f_\text{deg}\Big( \Tilde{\vec{X}}(\vec{u}_i) \Big) = \begin{bmatrix}
  \vec{F}^{(t-1)}(\vec{u}_i) \\
  \vec{F}^\alpha(\vec{u}_i) \\
  \vec{F}^\beta(\vec{u}_i)  \\
  \embed{deg}'(\vec{u}_i) \\
        \vec{P}'(\vec{u}_i) \\
        \embed{atp}'(\vec{u}_i)
\end{bmatrix}^T.
\end{equation*}

Next, we define
\begin{align*}
     f_\text{add}\Big( f_\text{deg}\Big( \Tilde{\vec{X}}(\vec{u}_i) \Big) \Big)  &= \begin{bmatrix}
  \vec{F}^{(t-1)}(\vec{u}_i) + \sum_{j \in [k]} ( \vec{F}^{\alpha}(\vec{u}_i)[(j-1)kp:jkp] + \vec{F}^{\beta}(\vec{u}_i)[(j-1)kp:jkp]) \\
  \vec{0} \\
  \embed{deg}'(\vec{u}_i) \\
        \vec{P}'(\vec{u}_i) \\
        \embed{atp}'(\vec{u}_i)
\end{bmatrix}^T,
\end{align*}
where $f_\text{add}$ sums up $\vec{F}^{(t-1)}(\vec{u}_i)$ with $\vec{F}^{\alpha}(\vec{u}_i)[(j-1)kp:jkp]$ and $\vec{F}^{\beta}(\vec{u}_i)[(j-1)kp:jkp]$ for each $j$. Afterwards, $f_\text{add}$ sets 
\begin{equation*}
    f_\text{deg}\Big( \Tilde{\vec{X}}(\vec{u}_i) \Big)[kp:2k^2p + kp] = \vec{0} \in \mathbb{R}^{2k^2p}.
\end{equation*}
Now, finally, note from the definition of $\vec{F}^{\alpha}(\vec{u}_i)$ and $\vec{F}^{\beta}(\vec{u}_i)$ that
\begin{equation*}
    \sum_{j \in [k]} ( \vec{F}^{\alpha}(\vec{u}_i)[(j-1)kp:jkp] + \vec{F}^{\beta}(\vec{u}_i)[(j-1)kp:jkp]) = H(\vec{X}^{(t-1, k)}(\vec{u}_i)),
\end{equation*}
where we recall that we defined
\begin{equation*}
   H(\vec{X}^{(t-1, k)}(\vec{u}_i)) = \frac{1}{n} \cdot \sum_{j \in [k]} \Big( \alpha_j \cdot \sum_{w \in \Delta_j(\vec{u}_i)}  \vec{F}^{(t-1)}(\phi_j(\vec{u}_i, w)) + \beta_j \cdot \sum_{w \in V(G) \setminus\Delta_j(\vec{u}_i)}  \vec{F}^{(t-1)}(\phi_j(\vec{u}_i, w)) \Big).
\end{equation*}
Hence, we have that
\begin{equation*}
    f_\text{add}\Big( f_\text{deg}\Big( \Tilde{\vec{X}}(\vec{u}_i) \Big) = \begin{bmatrix}
  \vec{F}^{(t-1)}(\vec{u}_i) +  H(\vec{X}^{(t-1, k)}(\vec{u}_i)) \\
  \vec{0} \\
  \embed{deg}'(\vec{u}_i) \\
        \vec{P}'(\vec{u}_i) \\
        \embed{atp}'(\vec{u}_i)
\end{bmatrix}^T.
\end{equation*}
Finally, we define
\begin{align*}
     f_\text{FFN}\Big( f_\text{add}\Big( f_\text{deg}\Big( \Tilde{\vec{X}}(\vec{u}_i) \Big)  \Big) &=
    \begin{bmatrix}
        \mathsf{FFN} \Big(\vec{F}^{(t-1)}(\vec{u}_i) + H(\vec{X}^{(t-1, k)}(\vec{u}_i)) \Big) \\
    \vec{0} \\
    \embed{deg}'(\vec{u}_i) \\
        \vec{P}'(\vec{u}_i) \\
        \embed{atp}'(\vec{u}_i)
    \end{bmatrix}^T   \\ &=
    \begin{bmatrix}
        \vec{F}^{(t)}(v) &
         \vec{0} &
        \embed{deg}'(v) &
        \vec{P}'(v) 
    \end{bmatrix},
\end{align*}
where $\mathsf{FFN}$ denotes the FFN in \Cref{eq:kwl_proof_X_update}, from which the last equality follows. Thank you for sticking around until the end. This completes the proof.
\end{proof}

The above proof shows that the \kgt{k} can simulate the $\delta$-\kwl{k}. Since the $\delta$-\kwl{k} is strictly more expressive than the \kwl{k} \citep{Morris2020b}, \Cref{theorem:delta_k_wl} follows directly from \Cref{theorem:k_wl}.

Further, the above proof shows that the \kgt{k} can simulate the \localkwl{}. If we simulate the \localkwl{} with the \kgt{k} while only using token embeddings for tuples in $V(G)^k_s$, \Cref{theorem:k_s_wl} directly follows.

\end{document}